\documentclass[11pt]{article}

\usepackage[ruled,vlined]{algorithm2e}
\usepackage{amssymb,amsmath,amsfonts,amsthm,eurosym,geometry,ulem,graphicx,caption,color,setspace,comment,footmisc,caption,natbib,pdflscape,array,hyperref,bm,graphicxbox,xcolor,cleveref,siunitx,indentfirst,booktabs,subcaption,authblk}
\usepackage[title]{appendix}
\usepackage{tikz}
\usepackage{listofitems} 
\usepackage{enumitem}
\usepackage{adjustbox}

\usetikzlibrary{arrows.meta}

\tikzstyle{mynode}=[thick,draw=blue,fill=blue!20,circle,minimum size=22]
\usetikzlibrary{positioning}
\usepackage{pgffor}
\usetikzlibrary{decorations.pathreplacing}
\usetikzlibrary{fit}
\usepackage{makecell}

\usepackage{rotating}
\usepackage{tablefootnote}
\usepackage{todonotes}
\usepackage{float}
\usetikzlibrary{calc}

\usepackage{multirow}
\usepackage{bigdelim}

\usepackage{threeparttable}
\usepackage{bbm}

\def\sym#1{\ifmmode^{#1}\else\(^{#1}\)\fi}

\colorlet{mylightred}{red!95!black!30}
\colorlet{mylightblue}{blue!95!black!30}
\colorlet{mylightgreen}{green!95!black!30}

\newtheorem{theorem}{Theorem}
\newtheorem{corollary}[theorem]{Corollary}
\newtheorem{proposition}{Proposition}
\newtheorem{assumption}{Assumption}
\newtheorem{lemma}{Lemma}

\newcolumntype{L}[1]{>{\raggedright\let\newline\\arraybackslash\hspace{0pt}}m{#1}}
\newcolumntype{C}[1]{>{\centering\let\newline\\arraybackslash\hspace{0pt}}m{#1}}
\newcolumntype{R}[1]{>{\raggedleft\let\newline\\arraybackslash\hspace{0pt}}m{#1}}

\begin{document}

\begin{titlepage}
\title{\textbf{Bridging Structured Knowledge and Data:\\\vspace{0.1in}A Unified Framework with Finance Applications\thanks{We thank the seminar and conference participants at the NUS DAO-ISEM-IORA Seminar Series.}\\\vspace{0.2in}}}
\bigskip

\author{
        Yi Cao\thanks{Department of Financial and Actuarial Mathematics, Xi'an Jiaotong-Liverpool University,
Suzhou, Jiangsu 215123, PR China. Email: Yi.Cao@xjtlu.edu.cn}\quad\quad
        Zexun Chen\thanks{Management Science and Business Economics Group, Business School, University of
Edinburgh, 29 Buccleuch Place, Edinburgh EH8 9JS, UK. Email: Zexun.Chen@ed.ac.uk}\quad\quad
        Lin William Cong\thanks{Nanyang Technological University, ABFER, CEPR, \& NBER, 91 Nanyang Avenue, Wee Cho Yaw Plaza, Singapore 639956. Email: Will.Cong@ntu.edu.sg}\quad\quad
        Heqing Shi\thanks{Management Science and Business Economics Group, Business School, University of Edinburgh, 29 Buccleuch Place, Edinburgh EH8 9JS, UK. Email: Heqing.Shi@ed.ac.uk}\quad\quad
        }

\date{%
    First draft: Dec 2025; current draft: Mar 2026.
}
\maketitle

\begin{abstract}
\noindent

We develop Structured-Knowledge-Informed Neural Networks (SKINNs), a unified estimation framework that embeds theoretical, simulated, previously learned, or cross-domain insights as differentiable constraints within flexible neural function approximation. SKINNs jointly estimate neural network parameters and economically meaningful structural parameters in a single optimization problem, enforcing theoretical consistency not only on observed data but over a broader input domain through collocation, and therefore nesting approaches such as functional GMM, Bayesian updating, transfer learning, PINNs, and surrogate modeling. SKINNs define a class of M-estimators that are consistent and asymptotically normal with $\sqrt{N}$ convergence, sandwich covariance, and recovery of pseudo-true parameters under misspecification. We establish identification of structural parameters under joint flexibility, derive generalization and target-risk bounds under distributional shift in a convex proxy, and provide a restricted-optimal characterization of the weighting parameter that governs the bias–variance tradeoff. In an illustrative financial application to option pricing, SKINNs improve out-of-sample valuation and hedging performance, particularly at longer horizons and during high-volatility regimes, while recovering economically interpretable structural parameters with improved stability relative to conventional calibration. More broadly, SKINNs provide a general econometric framework for combining model-based reasoning with high-dimensional, data-driven estimation.\\


\end{abstract}

\bigskip
\setcounter{page}{0}
\thispagestyle{empty}
\end{titlepage}

\setstretch{1.4}

\section{Introduction}

The last decade has witnessed a paradigm shift towards data-driven modelling across scientific and industrial domains, thanks to the rapid advancements in artificial intelligence (AI). From Large Language Models (LLMs) exhibiting human-like fluency to deep learning algorithms achieving expert-level medical diagnoses, the success of these empirical methods has fostered a compelling narrative: sufficient data and computation allow learning complex patterns directly from observations, bypassing or even rendering obsolete theories, traditional econometric/reduced-form models, and knowledge-based conceptual frameworks, as is said by \cite{anderson2008end}:
\begin{center}
    ---\textit{``This is a world where massive amounts of data and applied mathematics replace every other tool that might be brought to bear. [...] With enough data, the numbers speak for themselves''.}
    
\end{center}
Yet many researchers, especially economists, would resist such a notion: data-driven models are vulnerable to noise and spurious temporal correlations \citep{harvey2016and}, a fragility particularly acute in settings characterized by low signal-to-noise ratios, such as finance \citep[e.g.,][]{gu2020empirical}. The lack of transparent mechanisms, model interpretability, and domain intuition further frustrates social scientists, not to mention the high computational costs associated with large models and AI algorithms.

Traditional theories and econometric approaches indeed help researchers understand a wide range of underlying principles, causal relationships, and economic mechanisms. But their over-simplified assumptions (e.g., market efficiency, perfect rationality, free of arbitrage, etc.), rigid model structures, and over-reliance on human expertise fail to capture the full complexity of real-world phenomena, or the high-dimensionality, non-stationarity, and nonlinearity of modern datasets \citep{cochrane2011presidential, cong2019alphaportfolio}, leading to significant performance degradation and under-adoption by practitioners. This dichotomy, between principled but idealized theory and flexible but fragile empiricism, underscores a critical gap and pressing need in current methodologies by both science and social science inquiries.

To this end, we propose Structured-Knowledge-Informed Neural Networks (SKINNs), a framework that integrates flexible neural function approximation with structured domain knowledge (generically referred to as ``theory''). SKINNs jointly estimate the parameters of a deep neural network (NN) and a structured theoretical module within a single optimization problem, allowing each component to mitigate the other's misspecification rather than fixing theory parameters ex ante or estimating them sequentially. Estimation proceeds via a composite loss that balances data fidelity and structural consistency, with theoretical restrictions enforced not only at observed data points but over a broader input domain through collocation, thereby preserving economic structure under data sparsity and distribution shift. The framework accommodates high- and potentially infinite-dimensional latent objects (e.g., probability distributions or machine-learned state spaces) for which conventional GMM, Bayesian, or transfer-learning approaches become computationally infeasible. We show that SKINNs define a class of M-estimators: under standard regularity conditions, the joint estimator of neural and structural parameters is consistent and asymptotically normal with $\sqrt{N}$ convergence and sandwich-form covariance; under misspecification, the structural parameters converge to pseudo-true values. Consequently, the learned structural parameters retain economic interpretation and support formal statistical inference, bridging modern machine learning and classical econometrics.

Beyond classical consistency and asymptotic normality, we establish identification of the structural parameters under joint flexibility of the neural component, clarifying the conditions under which economically meaningful latent structural parameters are uniquely determined even when the function class is highly expressive. We further derive generalization and target-risk bounds under distributional shift in a convex proxy, providing formal support for the role of structured regularization and collocation in improving robustness. Within the GMM interpretation, we characterize the weighting parameter as a restricted-optimal choice that balances bias and variance, and show that under orthogonal moment conditions, the estimator-side of SKINNs for structural parameters attains the optimal asymptotic variance within the given moment model.


To illustrate the empirical relevance of the framework, we apply SKINNs to option pricing, a setting in which theoretical structures are well-developed yet frequently misspecified, and purely data-driven methods often struggle under distributional shifts. Using over two decades of S\&P 500 index option data, we show that SKINNs improve out-of-sample pricing accuracy and Delta-hedging performance relative to both flexible NNs and classical structural models, particularly at longer prediction horizons and during high-volatility regimes. The gains are modest when market conditions are stable, but become economically meaningful when volatility rises. Moreover, the latent parameters embedded in the structured-knowledge component of SKINNs evolve in economically coherent ways and exhibit improved numerical stability relative to conventional calibration. These findings suggest that jointly integrating structured knowledge with flexible function approximation can enhance generalization while preserving economically interpretable latent structure.

\paragraph{Methodological innovations and contributions.} SKINNs introduce three methodological innovations that distinguish them from existing hybrid modeling approaches. First, they jointly estimate the parameters of both the data-driven model and the structured theoretical model within a single optimization problem. Unlike approaches that require separate calibration of structural parameters \citep[e.g., the forward-problem PINNs in][]{raissi2019physics} or sequential pre-training and fine-tuning as in conventional transfer learning \citep[][]{chen2023teaching}, SKINNs optimize the NN parameters and the structured-knowledge parameters simultaneously through a unified objective. The composite loss balances empirical fit with theoretical consistency, and gradients are computed with respect to both parameter sets and updated concurrently via backpropagation. This co-adaptation allows the neural component to capture rich empirical patterns while remaining disciplined by theory, and permits the structural parameters to adjust in light of the data, avoiding both static calibration and purely empirical black-box estimation.

Second, SKINNs accommodate a broad spectrum of structured-knowledge representations. The theoretical module may consist of closed-form analytical solutions (e.g., Black–Scholes-Merton), surrogate neural approximations to computationally intensive structural models, abstract constraints defining admissible functions, or even machine-learned knowledge by generative models. The only requirement is differentiability with respect to inputs and learnable parameters, ensuring compatibility with gradient-based optimization. This flexibility enables SKINNs to operate across domains with varying degrees of theoretical structure, from settings with well-established parametric models to those where theory is expressed through equilibrium restrictions, simulation procedures, or distributional constraints. In the most general case, ``theory'' can be learned by a machine, e.g., via a (variational) autoencoder.

Third, SKINNs are designed for structural parameter recovery as well as prediction. While most machine learning models focus exclusively on predictive accuracy, SKINNs simultaneously estimate economically interpretable parameters—such as implied volatilities, risk-aversion coefficients, or risk-neutral distributions—that characterize underlying mechanisms. These parameters are learned jointly with the flexible function approximators and inherit formal inferential properties from the M-estimation framework. As a result, SKINNs transform flexible approximation into structurally interpretable estimation, enabling statistical inference on the interested structural parameters rather than merely improving the predictive power.

SKINNs are closely related to the generalized method of moments \citep[GMM;][]{hansen1982large}, Bayesian posterior estimation, transfer learning \citep[][]{pratt1992discriminability, chen2023teaching, bai2024teach}, and physics-informed NNs \citep[PINNs;][]{raissi2017physics,raissi_physics_2017}, but differ in important respects. Like GMM, SKINNs combine empirical fit with theory-implied restrictions; however, SKINNs jointly learn both the structural parameters and a flexible function, and enforce theoretical consistency not only at observed data points but over a broader input domain through collocation. From a Bayesian perspective, SKINNs resemble maximum a posteriori estimation with theory-based regularization, but unlike standard Bayesian procedures that impose a fixed ex ante prior distribution over structural parameters, SKINNs determine both the structural parameters and the predictive function jointly through a single optimization problem. The regularization induced by theory is therefore endogenous to the data-fitting objective rather than anchored to a fixed prior specification. Unlike PINNs, which impose differential equations through high-order derivatives and often suffer from gradient pathologies, SKINNs enforce consistency with model solutions through various formats of representations—parametric, semi-parametric, or non-parametric—allowing scalable estimation even in the presence of unobservable state variables and high-dimensional latent structures.

\paragraph{Applications and empirical findings.} We illustrate the framework in option pricing, a canonical setting in which (i) strong structural restrictions are available but routinely misspecified, (ii) purely data-driven methods can fit well in-sample yet degrade under distribution shifts, and (iii) ``theory'' naturally comes in multiple differentiable forms—from closed-form solutions to simulation-based models and non-parametric no-arbitrage restrictions. This makes option pricing a demanding laboratory for evaluating whether SKINNs can improve generalization while preserving the interpretability of latent economic objects.

Empirically, we instantiate SKINNs with three classes of structured-knowledge modules. In the parametric class, we consider the Black–Scholes-Merton (BSM, one latent parameter), the ad-hoc Black–Scholes-Merton (ABSM) specification with strike–maturity dependent volatility (six parameters), Heston's stochastic volatility (HSV, five parameters) model, its extension with jumps (HSVJ, nine parameters), and a deliberately high-dimensional SABR-type specification (hundreds of parameters). In the semi-parametric class, we approximate computationally intensive models by differentiable deep surrogates as in \citet{chen2021deep}, including HSV and a non-affine stochastic volatility (NASV, six parameters) model. In the non-parametric class, we implement a martingale option pricing approach (MOPA) with a high-dimensional risk-neutral distribution (2{,}000 parameters) and an autoencoder-based representation learned by machine (arbitrarily dimensional, but we consider two- and fifty-parameter cases). We benchmark against plain-vanilla NNs (with and without boundary conditions), shape-constrained NNs enforcing monotonicity/convexity, forward- and inverse-formulation PINNs \citep{raissi2017physics, raissi_physics_2017}, and transfer learning NNs \citep{pratt1992discriminability, chen2023teaching, bai2024teach}. These alternatives span purely data-driven methods and existing theory-hybrid designs, but either impose only static model-free constraints or become fragile when the theoretical structure is high-dimensional or involves unobservable state variables.

Our empirical analysis uses a comprehensive dataset of S\&P 500 index options spanning over two decades, on which SKINNs exhibit superior performance across two critical dimensions: out-of-sample option pricing accuracy and Delta-hedging effectiveness. We use daily transaction quotes of S\&P 500 index options from OptionMetrics, covering the period from January 4, 1996, to December 31, 2022. This extensive 27-year timespan encompasses several major economic crises, including the dot-com bubble, the 2007-2009 global financial crisis, and the COVID-19 pandemic, providing a rigorous stress test for all models under varying market conditions. We utilize raw option data that reflects actual market conditions, including missing values and erroneous prices that challenge NN models. We focus on call options with maturities between 7 and 365 calendar days to ensure liquidity and information content. Using a forward rolling-window approach, we train models on three-month panels of option data and evaluate their out-of-sample performance over two consecutive months: a shorter prediction horizon (one month ahead) and a longer prediction horizon (two months ahead). This rolling procedure generates 317 training and testing periods, with the longer horizon presenting significantly greater challenges due to potential shifts in the data-generating process (DGP) relative to the training sample.

First, we evaluate all models across the 317 rolling periods using the \cite{diebold2002comparing} test to assess the statistical significance of performance differences. For shorter horizons, NN demonstrates competitive performance, statistically outperforming models with gradient pathology issues as well as structural models. This suggests that data-driven learning effectively captures latent option price patterns when the testing data closely resembles the training sample. However, only SKINNs incorporating sophisticated structured knowledge achieve statistically significant improvements over NN, indicating that marginal gains require more advanced knowledge representations when patterns are familiar. For longer prediction horizons, where the DGP may diverge substantially from the training sample, the limitations of pure data-driven approaches become apparent. The NN model underperforms most structural models, as previously learned patterns become less effective. In contrast, almost all SKINN variants significantly outperform both NN and NN with boundary conditions at the 1\% significance level, demonstrating substantially enhanced generalizability through the incorporation of structured knowledge. This improvement is particularly noteworthy given that longer-horizon predictions pose considerably greater challenges due to potential shifts in market dynamics across different economic regimes encountered in our dataset. Overall, SKINNs reduce out-of-sample pricing errors (measured in terms of the root mean square error, RMSE) by roughly 10–15\% relative to leading NN benchmarks and by substantially more in high-volatility regimes.

Furthermore, we assess the Delta-hedging capability of the models by constructing hedged portfolios and measuring next-day hedging errors across all prediction periods. The NN model becomes less accurate in this exercise. In both shorter and longer horizons, NN and NN with boundary conditions are significantly outperformed by all SKINN variants. They only manage to outperform those gradient-challenged NNs, and, interestingly, some sophisticated structural models (e.g., HSV, HSVJ). This suggests that the static model-free constraints also fail to enhance the estimation of the Delta ratios. Additionally, sophisticated structural models tend to underperform simpler ones, such as the BSM model, for the Delta-hedging purpose. This is also reflected in SKINNs, as while all SKINN variants consistently outperform the benchmark NNs, simpler structured-knowledge representations generally deliver superior hedging performance. These consistent improvements across both pricing and hedging tasks, across different prediction horizons, and through multiple economic crisis periods, underscore the practical value and robustness of incorporating structured knowledge into NN architectures for financial derivative applications.

We next investigate the economic mechanisms underlying SKINNs' outperformance, focusing on two related properties: (i) the countercyclical performance during volatile market conditions, and (ii) the economic interpretability of the SKINN-learned latent structural parameters embedded in structured-knowledge representations.

Our predictive regressions indicate that the relative pricing accuracy of SKINNs improves when market volatility increases. Across the 317 longer prediction periods, a one-unit increase in the average VIX is associated with reductions of approximately 0.09–0.10 in RMSE for SKINNs, compared to roughly 0.03 for boundary-constrained NNs and statistically insignificant effects for classical structural models. The differential becomes more pronounced in high-volatility regimes: during periods above the 80th percentile of average daily VIX, SKINNs exhibit statistically significant RMSE reductions of 0.14–0.15 at the 5\% level, whereas benchmark models do not display comparable improvements. In contrast, during low-volatility periods, when option price surfaces are relatively stable and patterns are easier to learn, performance differences across models are modest. Taken together, these results suggest that the benefits of incorporating structured knowledge are most apparent when market conditions are unstable and distributional shifts are more likely. In the highest-volatility quintile, SKINNs achieve pricing-error reductions that are approximately three to four times larger than those of boundary-constrained NNs, while also delivering significantly improved Delta-hedging effectiveness.

We further examine the latent structural parameters learned by SKINNs. Unlike transfer-learning approaches that require separate calibrations, SKINNs estimate network parameters and structured-knowledge parameters jointly within a unified optimization. In the one-dimensional BSM case, the SKINN-learned volatility closely tracks conventional implied volatility estimates. In higher-dimensional settings, the learned state variables display smoother time-series evolution and align with identifiable economic regimes, often exhibiting greater numerical stability than parameters obtained via stand-alone calibration. Even in the non-parametric martingale option pricing implementation, where the structured-knowledge component contains 2{,}000 probability parameters, the implied risk-neutral density evolves in economically interpretable ways across market episodes. These patterns indicate that the latent parameters are not merely auxiliary regularization devices, but capture economically meaningful information consistent with market conditions.

Beyond option pricing, the SKINNs framework applies naturally to a broad class of economic estimation problems. In production function estimation, structured knowledge can encode technological restrictions—such as monotonicity, returns to scale, or equilibrium implications of firm optimization—while the NN flexibly captures heterogeneous productivity dynamics. In demand estimation, SKINNs can combine flexible demand systems with utility-based or revealed-preference restrictions, allowing simultaneous recovery of demand functions and economically interpretable parameters. In dynamic discrete choice models, Bellman or Euler conditions can be embedded, enabling joint estimation of value functions and structural parameters without repeatedly solving the dynamic program. In macroeconomic applications, equilibrium conditions from DSGE models—such as Euler equations, resource constraints, or policy rules—can serve as structured knowledge while neural components approximate high-dimensional policy functions or shock processes. For data-driven asset pricing models, SKINNs can guide them with mean-variance utilities, which are traditionally solved by costly bi-level portfolio optimizations. In each case, SKINNs provide a unified econometric framework for jointly estimating flexible functional relationships and economically meaningful latent structure under theory-imposed discipline.

The remainder of this paper is organized as follows: Section \ref{sec_methodology} introduces the SKINNs framework, proves its statistical properties, and discusses connections to well-established econometric models and existing approaches that integrate data-driven learning with domain knowledge. Section \ref{sec_models} provides various specifications of option pricing structured-knowledge representation for SKINNs. Section \ref{sec_empirical} and \ref{sec_interpret} present the empirical findings, comparing the predictive performance of SKINNs against alternative methods and extracting economic insights from the learned structural parameters. Section \ref{sec_skinns_ap} demonstrates SKINN asset pricing models, which integrate the utility-based portfolio optimization to improve the learned asset pricing relations. Section \ref{sec_conclude} concludes.

\section{The SKINNs Framework}\label{sec_methodology}
The SKINNs framework is designed to synergize the predictive power of data-driven models with the explanatory rigor of theories, where the term ``theories'' should be broadly interpreted as reduced-form representations of domain knowledge. 
Purely data-driven models struggle with noisy, non-stationary environments where they are prone to overfitting and distributional shifts, while traditional theory-based approaches often rely on rigid, oversimplified assumptions and ad-hoc calibration processes that fall short of capturing complex market dynamics.

SKINNs aim to bridge this gap using a principled mechanism for embedding a theoretical model's structure directly into the learning objective of an NN. This is achieved through a composite loss function that allows the joint estimation of the NN's parameters and the latent structural parameters of the embedded theoretical model. By doing so, SKINNs regularize the learning process, guiding the NN towards solutions that are not only empirically accurate but also theoretically plausible. This approach generalizes and addresses key limitations of existing hybrid methods, such as PINNs \citep{raissi2017physics, raissi_physics_2017}---which suffer from spectral biases, gradient pathologies, and unobservable differentiation variables \citep[see, e.g.,][]{wang_understanding_2020, wang_when_2022}---, and transfer learning \citep[e.g.,][]{pratt1992discriminability, chen2023teaching, bai2024teach}, which lacks a mechanism for dynamic latent parameter discovery in its two-stage design, and can be vulnerable to catastrophic forgetting.

Specifically, the architecture of SKINNs is composed of two core components: a data-driven function approximator and a structured-knowledge representation borrowed from theories. These components are trained in concert to reconcile empirical observations with theoretical principles.

\subsection{The Data-Driven Component}
Let the primary learning model be a deep NN, $f(\mathbf{X};\mathbf{\theta}):\mathbb{R}^{d}\mapsto\mathbb{R}$, which maps a set of input features $\mathbf{X} \in \mathbb{R}^d$ to some target outputs $\mathbf{y} \in \mathbb{R}$. 
The vector $\mathbf{\theta}$ represents the full set of trainable NN parameters, i.e., the biases and weights. For instance, a standard multi-layer fully-connected feed-forward NN is defined by a recursive formulation:
\begin{equation}
    f^{(l)} = h\left(\mathbf{b}^{(l-1)} + \mathbf{W}^{(l-1)}f^{(l-1)}\right), \quad l=1, \dots, L, 
\end{equation}
where $f^{(0)}$ is the input layer; $f^{(l)}$ denotes each of the $L$ hidden layers; $f^{(L+1)} = \mathbf{b}^{(L)} + \mathbf{W}^{(L)}f^{(L)}$ is the final output layer; $h(\cdot)$ is a non-linear activation function, e.g., ReLU where $h(x)=max(x, 0)$; and $\mathbf{\theta} = \{\mathbf{b}^{(l)}, \mathbf{W}^{(l)}\}_{l=1}^L$ contains all the trainable NN parameters.

The NN parameters, $\mathbf{\theta}$, are traditionally optimized by minimizing solely a data-centric loss function, $\mathcal{L}_{\text{Data}}$, with respect to a set of observations $\mathcal{D}_{\text{Obs}} = \{(\mathbf{X}^{[i]}_{\text{Obs}}, \mathbf{y}^{[i]}_{\text{Obs}})\}_{i=1}^{N_{\text{Obs}}}$. A common choice is the mean squared error, given the training samples indexed by $1,\cdots,N_{\text{Obs}}$ from the observations:
\begin{equation}\label{eq_data_centric_loss}
    \mathcal{L}_{\text{Data}}(\mathbf{\theta}; \mathcal{D}_{\text{Obs}}) = \frac{1}{N_{\text{Obs}}} \sum_{i=1}^{N_{\text{Obs}}} \left( f_{\mathbf{\theta}}(\mathbf{X}^{[i]}_{\text{Obs}}) - \mathbf{y}^{[i]}_{\text{Obs}} \right)^2.
\end{equation}
This loss function, which serves as the data-driven component, anchors the model to the empirical evidence provided by the data. However, with the data-driven component alone, the model will reduce to a plain-vanilla NN, with all the drawbacks of being a black box, prone to overfitting, etc.

Yet the limitations are inherent, plain-vanilla NNs are still a powerful model to relax the rigid assumptions and fixed specifications in many disciplines, particularly economics. NNs can also incorporate higher-dimensional conditioning variables than traditional econometric models, and hence alleviate the omitted variable problem \citep[see, e.g.,][]{kelly2019characteristics,gu2019autoencoder,cong2019alphaportfolio,chen2023teaching,feng2024deep}.


\subsection{The Structured-Knowledge Component}\label{structured-knowledge component}

To better tame the over-parameterized NNs, as the size of $\theta$ is usually much larger than the number of observations $N_{\text{Obs}}$, the second component in SKINNs is a formal representation of structured knowledge, denoted by a function $g(\mathbf{X}^{\text{SK}}; {\mathbf{\phi}}): \mathbb{R}^{d_{\text{SK}}+d_{\mathbf{\phi}}} \to \mathbb{R}$, where $\mathbf{X}^{\text{SK}}$ is the input features that are observable; $\mathbf{\phi}$ is the latent structured-knowledge parameters; and $d_{\text{SK}}, d_{\phi}$ are their respective dimensions. Typically, $\mathbf{X}^{\text{SK}}$ is a subset of $\mathbf{X}$ (i.e., $d_{\text{SK}}\leq d$). $\mathbf{X}^{\text{SK}}$ therefore preserves a model-driven parsimonious structure, while $\mathbf{X}\in\mathbb{R}^{d}$, the inputs of the NN component, can incorporate higher-dimensional empirical features that are possibly omitted. The structured-knowledge function $g_{\mathbf{\phi}}$ encapsulates the structure of theories in a certain domain, mapping the theoretically-relevant inputs $\mathbf{X}^{\text{SK}} \subseteq \mathbf{X}$ to a model-driven estimate. SKINNs embed the structured knowledge by treating the vector $\mathbf{\phi}$ from the theoretical model $g_{\mathbf{\phi}}$ as a set of learnable parameters additional to the NN parameters $\mathbf{\theta}$.

Such a joint learning mechanism by SKINNs facilitates a co-evolutionary process of both components, which offers improved model flexibility and efficiency over many two-stage hybrid approaches \citep[e.g.,][]{chen2023teaching, zhang2023two}, including the ``Smart Predict-then-Optimize'' (SPO) framework and the differentiable optimization layers for bi-level problems involving upstream prediction model estimations and downstream optimizations \citep[see,][]{amos2017optnet, agrawal2019differentiable, elmachtoub2022smart, wang2026machine}. The two sets of learnable parameters—one includes the nuisance NN parameters and the other includes the interpretable latent structural parameters—make SKINNs also semi-parametric-model alike. However, distinct from the classic semi-parametric econometric models, where both $g_{\phi}$ and $f_{\theta}$ contribute to the estimates as seen in the equation (1.1) in \cite{Chernozhukov2018}, $g_{\phi}$ is rather employed in SKINNs to regularize the training of $f_{\theta}$ who ultimately produces estimates. More importantly, the two parameter sets are simultaneously identified by SKINNs instead of sequentially, as classic semi-parametric models do. Hence, the data-driven component compensates for the structured-knowledge component in the same learning process. 
This naturally mitigates the overfitting impact that the over-parameterized $f_{\theta}$ can have on the estimation of the structured-knowledge parameters $\phi$, an issue that is resolved by double machine learning \citep[see,][]{Chernozhukov2018}.

The structured-knowledge function \( g_\mathbf{\phi} \) admits diverse formats, as long as it remains first-order differentiable with respect to both its inputs \( \mathbf{X}^{\text{SK}} \) and the learnable latent parameters \( \mathbf{\phi} \). 
One advantage of this is that the latent parameters $\phi$ can be scaled to high- and potentially infinite-dimension, thanks to the first-order optimization methods that are suitable for large-scale problems associated with NNs. These flexibilities allow SKINNs to embed a wide range of theoretical, empirical, and machine-learned structures seamlessly into the learning process, as well as to learn the high-dimensional parameters in theoretical models that were previously hard to estimate. We categorize the formats of structured knowledge into three main types: parametric representations, semi-parametric representations, and non-parametric representations.

\subsubsection{Parametric Structured-Knowledge Representations}
Parametric structured-knowledge representations (PSKRs) are the most direct methods of embedding established theories, as the prevalence of many established theories stems from their well-posed parametric equations. For PSKRs, the number of latent parameters and the functional form of the governing equation are fully specified. Therefore, for this format, we are allowed to conveniently define a structured-knowledge function $g_{\mathbf{\phi}}$ by a closed-form or semi-closed-form mathematical expression derived directly from theories. 
A canonical example from finance is the BSM option pricing model, where the theoretically-relevant inputs $\mathbf{X}^{\text{SK}}$ would consist of the observable characteristics such as the underlying asset price ($S$), strike price ($K$), risk-free rate ($r$), and maturity time ($T$); the latent parameter vector $\mathbf{\phi}$ would contain the unobservable parameter, namely the asset return volatility, $\sigma$. The function $g_\mathbf{\phi}$ is thus explicitly defined as: $g_{\mathbf{\phi}}(\mathbf{X}^{\text{SK}}) \equiv \text{BSM}(\mathbf{X}^{\text{SK}}; \mathbf{\phi}=\{\sigma\})$.

In the cases where the theory-driven models admit only semi-closed-form solutions, for example the HSV model, the function $g_{\mathbf{\phi}}(\mathbf{X}^{\text{SK}})\equiv\text{HSV}(\mathbf{X}^{\text{SK}}, \mathbf{u};\mathbf{\phi})$ requires numerical methods such as the fast Fourier transform (FFT) and the Fourier-Cosine series expansions (COS), with $\mathbf{u}$ denotes a vector of algorithmic coefficients, e.g., the grid points in the frequency domain to evaluate the FFT, and the upper or lower limit of an integral, that are determined by users. Figure (\ref{fig:param_format_repre}) visualizes such PSKRs.


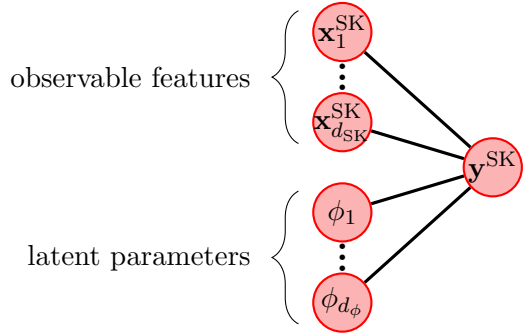
\begin{figure}[h]
\centering

\begin{tikzpicture}[x=2.0cm,y=1.2cm]
  \readlist\Nnod{4,1} 


  \foreachitem \N \in \Nnod{ 
    \foreach \i [evaluate={\x=\Ncnt; \y=\N/2-\i+0.5; \prev=int(\Ncnt-1);}] in {1,...,\N}{ 
      \node[mynode, color=mylightred, draw=red] (N\Ncnt-\i) at (\x,\y) {};
      \ifnum\Ncnt>1 
        \foreach \j in {1,...,\Nnod[\prev]}{ 
          \draw[very thick] (N\prev-\j) -- (N\Ncnt-\i); 
        }
      \fi 
    }
  }


  \node at ([yshift=-14pt] N1-1) {$\Large\bm{\vdots}$};
  \node at ([yshift=-14pt] N1-3) {$\Large\bm{\vdots}$};

  \node at (N1-1) {$\mathbf{x}^{\text{SK}}_{1}$};
  \node at (N1-2) {$\mathbf{x}^{\text{SK}}_{d_{\text{SK}}}$};
  \node at (N1-3) {$\mathbf{\phi}_{1}$};
  \node at (N1-4) {$\mathbf{\phi}_{d_{\mathbf{\phi}}}$};

  \node at (N2-1) {$\mathbf{y}^{\text{SK}}$};

  \draw[decorate,decoration={brace,amplitude=10pt,mirror}]
    ([xshift=-0.3cm] N1-1.north west) -- ([xshift=-0.3cm] N1-2.south west) 
    node[midway, left=0.5cm, anchor=east, align=left] {observable features};

  \draw[decorate,decoration={brace,amplitude=10pt,mirror}]
    ([xshift=-0.3cm] N1-3.north west) -- ([xshift=-0.3cm] N1-4.south west) 
    node[midway, left=0.5cm, anchor=east, align=left] {latent parameters};

\end{tikzpicture}

\caption{PSKR by a (semi-)parametric equation. $\mathbf{x}^{\text{SK}}_{1}\cdots\mathbf{x}^{\text{SK}}_{d_{\text{SK}}}$ denote each of the observable features, and $\mathbf{\phi}_{1}\cdots\mathbf{\phi}_{d_{\mathbf{\phi}}}$ denote each of the learnable latent parameters, in the PSKR.}
\label{fig:param_format_repre}
\end{figure}

\subsubsection{Semi-Parametric Structured-Knowledge Representations}

In more complex systems, closed- or semi-closed-form solutions are intractable. Theories may be expressed through mechanistic simulators (e.g., agent-based models), complex stochastic differential equations (SDEs), or high-dimensional partial differential equations (PDEs), which can only be solved via iterated computationally expensive numerical procedures (e.g., Monte-Carlo simulation, finite-difference).

SKINNs solve this through building a pre-trained deep surrogate neural network (DSNN) that learns the mapping $\{\mathbf{X}^{\text{SK}}, \mathbf{\phi}\}\mapsto\mathbf{y}^{\text{SK}}$ that is unavailable from theory. In this scenario, theory-driven models still take the observable features $\mathbf{X}^{\text{SK}}$ and the fixed number of latent parameters $\phi$ as the inputs, but the output $\mathbf{y}^{\text{SK}}$ is solved entirely by numerical procedures, and learned later by an NN. In other words, a DSNN is employed to approximate the unknown function, $g_{\mathbf{\phi}}(\mathbf{X}^{\text{SK}}):=\{\mathbf{X}^{\text{SK}}, \mathbf{\phi}\}\mapsto\mathbf{y}^{\text{SK}}$, thanks to the universal approximation capability \citep[see,][]{cybenko1989approximation, hornik1989multilayer,hanin2019universal}.

In engineering, physics, biomedicine, and finance, complex systems are often prescribed as multivariate SDEs. When the SDE system is in the affine class, there is a standard treatment allowing the system to be solved with a (semi-)closed-form solution, and therefore, we are able to construct a PSKR as for the HSV case \citep[see,][]{duffie2000transform, freire2023autoencoder}. However, in many unfavoured settings, the SDE system is non-affine, and we are only able to approximate the solutions with, e.g., Monte-Carlo simulation. For example, in finance, asset prices can be prescribed by three-factor jump diffusion SDEs, which can be non-affine. See details in Section (2), Equation (1)-(3) from \cite{kaeck2012volatility}.

To train a DSNN, one is required to randomly sample a large number of input instances for a target complex system, $\left(\mathbf{x}_{\text{SK}}^{[i,1]},\cdots,\mathbf{x}_{\text{SK}}^{[i,d_{\text{SK}}]},\mathbf{\phi}^{[i,1]},\cdots,\mathbf{\phi}^{[i,d_{\mathbf{\phi}}]}\right), 1\leq i\leq N_{\text{SR}}$, from a multivariate uniform distribution, and then query the corresponding numerical solutions, $\mathbf{y}_{\text{SK}}^{[i]}$. 
The simulated data space $\mathcal{D}_{\text{SR}}=\{[\mathbf{X}_{\text{SK}}^{[i]},\mathbf{\phi}^{[i]}], \mathbf{y}_{\text{SK}}^{[i]}\}_{i=1}^{N_{\text{SR}}}$ is used for training a DSNN offline, with the objective of minimizing an approximation loss, as much as possible, without the concern of overfitting:
\begin{equation}
    \mathcal{L}_{\text{SR}}(\theta;\mathcal{D}_{\text{SR}}) = \sum_{i=1}^{N_{\text{SR}}} \left(f^{\text{SR}}_{\mathbf{\theta}}\left(\mathbf{X}_{\text{SK}}^{[i]}, \mathbf{\phi}^{[i]}\right)-\mathbf{y}_{\text{SK}}^{[i]}\right)^{2} \Big/ N_{\text{SR}}.
\end{equation}
Once the desired approximation accuracy is achieved during the offline pre-training, we successfully create a NN surrogate $f^{\text{SR}}_{\mathbf{\theta}}(\mathbf{X}^{\text{SK}},\phi)$ of the target system that can be used repeatedly, with little computational burden. 
 
We then freeze some of the DSNN weights and biases, but keep the latent parameters $\phi$ in the DSNN inputs as the learnable parameters when integrating it into a SKINN. In this case, the structured-knowledge representation function $g_{\mathbf{\phi}}$ becomes an another deep NN, i.e., $g_{\phi}(\mathbf{X}^{\text{SK}})\equiv f^{\text{SR}}_{\theta}(\mathbf{X}^{\text{SK}}, \phi)$. We visualize such semi-parametric structured-knowledge representations (SPSKRs) in Figure (\ref{fig:semi_param_format_repre}). 

\begin{figure}[htbp]
\centering
\begin{tikzpicture}[x=2.0cm,y=1.2cm]
  \readlist\Nnod{4,5,5,1} 

  \foreachitem \N \in \Nnod{
    \foreach \i [evaluate={\x=\Ncnt; \y=\N/2-\i+0.5; \prev=int(\Ncnt-1);}] in {1,...,\N}{
      \node[mynode, color=mylightred, draw=red] (N\Ncnt-\i) at (\x,\y) {};
      
      \ifnum\Ncnt=2 
        \foreach \j in {1,...,\Nnod[\prev]}{ \draw[very thick] (N\prev-\j) -- (N\Ncnt-\i); }
      \fi
      \ifnum\Ncnt=4 
        \foreach \j in {1,...,\Nnod[\prev]}{ \draw[very thick] (N\prev-\j) -- (N\Ncnt-\i); }
      \fi

      \ifnum\Ncnt=3
        \foreach \j in {1,...,\Nnod[\prev]}{
           \draw[very thick] (N\prev-\j) -- ($(N\prev-\j)!0.4!(N\Ncnt-\i)$); 
           \draw[very thick] ($(N\prev-\j)!0.6!(N\Ncnt-\i)$) -- (N\Ncnt-\i); 
        }
      \fi
    }
  }

  \node at (2.53, 0.5) {$\Large\bm{\cdots}$};
  \node at (2.53, 0) {$\Large\bm{\cdots}$};
  \node at (2.53, -0.5) {$\Large\bm{\cdots}$};

  \node at ([yshift=-14pt] N1-1) {$\Large\bm{\vdots}$};
  \node at ([yshift=-14pt] N1-3) {$\Large\bm{\vdots}$};
  \node at (N1-1) {$\mathbf{x}^{\text{SK}}_{1}$};
  \node at (N1-2) {$\mathbf{x}^{\text{SK}}_{d_{\text{SK}}}$};
  \node at (N1-3) {$\mathbf{\phi}_{1}$};
  \node at (N1-4) {$\mathbf{\phi}_{d_{\mathbf{\phi}}}$};
  \node at (N4-1) {$\mathbf{y}^{\text{SK}}$};

  \draw[decorate,decoration={brace,amplitude=10pt,mirror}]
    ([xshift=-0.3cm] N1-1.north west) -- ([xshift=-0.3cm] N1-2.south west) 
    node[midway, left=0.5cm, anchor=east, align=center] {observable features\\ (simulated)};

  \draw[decorate,decoration={brace,amplitude=10pt,mirror}]
    ([xshift=-0.3cm] N1-3.north west) -- ([xshift=-0.3cm] N1-4.south west) 
    node[midway, left=0.5cm, anchor=east, align=center] {latent parameters\\ (simulated)};

  \draw[decorate,decoration={brace,amplitude=5pt,raise=15pt}] ([xshift=-5pt] N2-1 |- N2-1) -- ([xshift=5pt] N3-1 |- N3-1) node[midway,yshift=18pt,above] {hidden layers (deep)};

\end{tikzpicture}

\caption{SPSKR by an expensive-to-solve target system. The inputs, i.e., the observable features $\mathbf{X}^{\text{SK}}$ and the latent parameters $\mathbf{\phi}$, are still determined by theories, as in PSKRs. But the structured-knowledge function $g_{\mathbf{\phi}}(\mathbf{X}^{\text{SK}})$, in this case, is approximated by a pre-trained DSNN.}
\label{fig:semi_param_format_repre}

\end{figure}
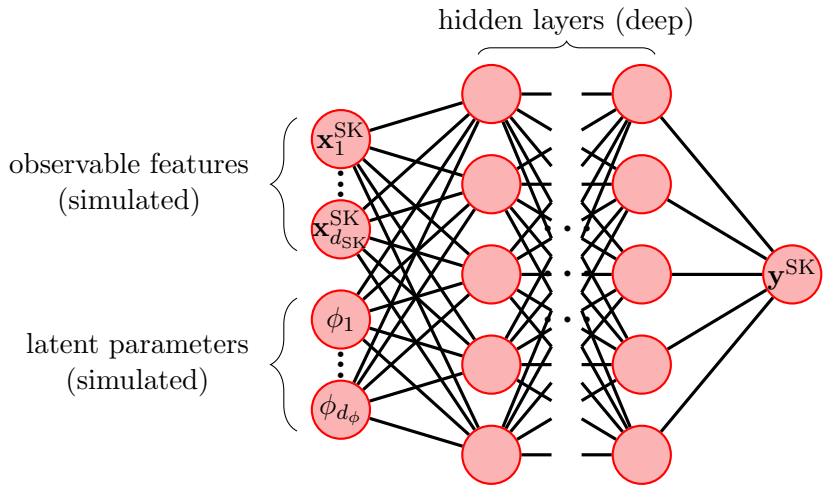

\subsubsection{Non-Parametric Structured-Knowledge Representations}

PSKRs and SPSKRs are characterized by the theory-driven inputs $\mathbf{X}^{\text{SK}}$ and a fixed-sized, low-dimensional vector of $\mathbf{\phi}$. This sparsity assumption on the latent structural parameters by the established theories is mainly for the sake of tractability, as the low-dimensional parameters are easier to estimate. However, it can be the case that theories are constructed in a way that they take an unspecified number of latent parameters in $\mathbf{\phi}$, for greater generality and fewer assumptions. In such high-dimensional settings, the deep surrogate approach becomes infeasible, as the required simulated data space $\mathcal{D}_{\text{SR}}$ expands exponentially with the dimension of $d_{\phi}$. We provide two examples below to illustrate how these high-dimensional problems can be tackled by SKINNs, by using non-parametric structured-knowledge representations (NPSKRs) that are not prescribed by a fixed number of latent structural parameters, and can therefore easily accommodate high-dimensional (or even ultra-high-dimensional) specifications.

\paragraph{Probabilistic distributional structured knowledge.} Theories may provide functional forms such that they require a probabilistic distribution governing the underlying variables within the functions. Such probability distributions are usually unknown, but the probabilities themself can be treated as high-dimensional latent parameters. These unknown non-parametric distributions frequently appear in theories as the expectations for the underlying variables, e.g., $\mathbb{E}[z]$, or some known functions of them, e.g., $\mathbb{E}[u(z,\mathbf{X}^{\text{SK}})]$. 
The underlying variable $z$ determines the output of a theoretical model through the function of its probability distribution.

A typical example in finance is the principle of no-arbitrage, which dictates that an asset's price must equal its discounted expected payoff under a risk-neutral probability measure \(\mathbb{Q}\). This principle does not prescribe a specific parametric form for the underlying asset dynamics, but instead requires an expectation from an unknown distribution:
\begin{equation}\label{eq:mopa}
    \mathbf{y}^{\text{SK}} = e^{-r(T-t)}\mathbb{E}^{\mathbb{Q}}[u(z,\mathbf{X}^{\text{SK}})\vert\mathcal{F}_{t}],
\end{equation}
where $u(z,\mathbf{X}^{\text{SK}})$ is the terminal payoff of the asset at time $T$ (e.g., $u(z,\mathbf{X}^{\text{SK}})=(S_{T}-K)^{+}$ for a European call option, where $S_{T}$ is the underlying variable), and $\mathbf{y}^{\text{SK}}$ denotes the the theory-driven output, i.e., the no-arbitrage price of the option at time $t$.

The expectation of an unknown non-parametric distribution in Equation~\eqref{eq:mopa} needs the probabilities $\mathbb{Q}(z)$ of potentially infinite states of nature. The structured-knowledge function in this case can be formulated as:
\begin{align}
    g_{\phi}(\mathbf{X}^{\text{SK}})&\equiv e^{-r(T-t)}\sum_{i=1}^{\infty}u(z^{[i]},\mathbf{X}^{[i],\text{SK}})\mathbb{Q}(z^{[i]}), \label{eq:mopa_inf_sum}
\end{align}
where $z=\{S_{T}^{[1]}, S_{T}^{[2]}, \cdots, S_{T}^{[\infty]}\}$ refer to the infinite states of the underlying variable $S_{T}$, and the learnable latent parameters $\mathbf{\phi}$ are defined as the vector of unknown risk-neutral probabilities associated with each state:
\begin{equation}
    \mathbf{\phi} = \{\mathbb{Q}(z^{[1]}), \mathbb{Q}(z^{[2]}), \cdots, \mathbb{Q}(z^{[\infty]})\}.
\end{equation}
To facilitate the integration of such NPSKRs in SKINNs, we use the empirical counterpart of Equation~\eqref{eq:mopa_inf_sum}:
\begin{equation}\label{eq:mopa_empirical}
    g_{\mathbf{\phi}}(\mathbf{X}^{\text{SK}}) := e^{-r(T-t)}\sum_{i=1}^{N_{z}} u(z^{[i]},\mathbf{X}^{[i],\text{SK}})\mathbb{Q}(z^{[i]}),
\end{equation}
where $N_{z}$ is a sufficiently large integer to approximate the infinite sum. The learnable latent parameters $\mathbf{\phi}=\{\mathbb{Q}(z^{[1]}), \mathbb{Q}(z^{[2]}), \cdots, \mathbb{Q}(z^{[N_{z}]})\}$ are constrained such that each of them are non-negative and $\sum_{i=1}^{N_{z}} \mathbb{Q}(z^{[i]})=1$, which can be enforced during training using a softmax function.

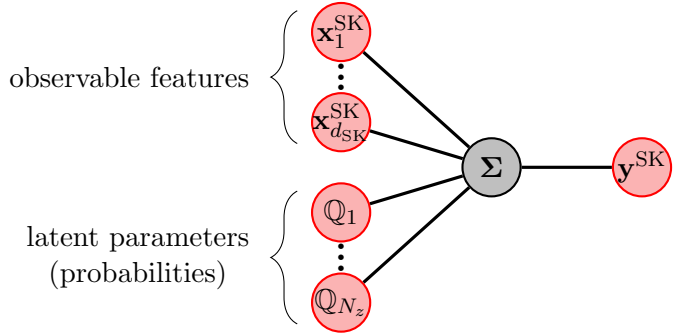
\begin{figure}[htbp]
\centering
\begin{tikzpicture}[x=2.0cm,y=1.2cm]
  \readlist\Nnod{4,1,1} 

  \foreachitem \N \in \Nnod{ 
    \foreach \i [evaluate={\x=\Ncnt; \y=\N/2-\i+0.5; \prev=int(\Ncnt-1);}] in {1,...,\N}{
      
      \ifnum\Ncnt=2
        \node[mynode, fill=lightgray, draw=black] (N\Ncnt-\i) at (\x,\y) {};
      \else
        \node[mynode, color=mylightred, draw=red] (N\Ncnt-\i) at (\x,\y) {};
      \fi

      \ifnum\Ncnt>1 
        \foreach \j in {1,...,\Nnod[\prev]}{ 
          \draw[very thick] (N\prev-\j) -- (N\Ncnt-\i); 
        }
      \fi 
    }
  }

  \node at ([yshift=-14pt] N1-1) {$\Large\bm{\vdots}$};
  \node at ([yshift=-14pt] N1-3) {$\Large\bm{\vdots}$};

  \node at (N1-1) {$\mathbf{x}^{\text{SK}}_{1}$};
  \node at (N1-2) {$\mathbf{x}^{\text{SK}}_{d_{\text{SK}}}$};
  \node at (N1-3) {$\mathbb{Q}_{1}$};
  \node at (N1-4) {$\mathbb{Q}_{N_{z}}$};

  \node at (N2-1) {$\bm{\Sigma}$};
  \node at (N3-1) {$\mathbf{y}^{\text{SK}}$};

  \draw[decorate,decoration={brace,amplitude=10pt,mirror}]
    ([xshift=-0.3cm] N1-1.north west) -- ([xshift=-0.3cm] N1-2.south west) 
    node[midway, left=0.5cm, anchor=east, align=left] {observable features};

  \draw[decorate,decoration={brace,amplitude=10pt,mirror}]
    ([xshift=-0.3cm] N1-3.north west) -- ([xshift=-0.3cm] N1-4.south west) 
    node[midway, left=0.5cm, anchor=east, align=center] {latent parameters \\ (probabilities)};

\end{tikzpicture}
\caption{NPSKR by an unknown distribution. The inputs include a high-dimensional vector of probabilities, which serves as the learnable latent parameters in SKINNs. The probabilities are passed to a softmax function to ensure non-negativity and the total probability of one unit.}
\label{fig:nonparam_format_repre_prob}
\end{figure}

The SKINNs framework in this case learns a high-dimensional vector of risk-neutral probabilities directly from data, guided solely by the foundational economic principle of no-arbitrage. We visualize such NPSKRs in Figure (\ref{fig:nonparam_format_repre_prob}).

\paragraph{Machine-learned structured knowledge.} SKINNs also provide solutions when there is no guiding knowledge from theories. At times, we observe the outputs from a DGP, rather than knowing the true specification of it. For example, independent of any assumed asset pricing models, investors still observe stock returns, the cross-section of option prices, and the term-structure of interest rates generated by the market. In these cases, we are unknown about the DGP, but we observe its realizations. We employ a (variational) auto-encoder (AE) to learn the structured knowledge, that is, the latent parameters are selected by the machine instead of by theory.

Given a series of observation vectors $\mathcal{D}_{\text{AE}}=\{\hat{\mathbf{y}}_{i}\}_{i=1}^{N_{\text{AE}}}$, probably with noise, which are realizations of an unknown DGP, we reconstruct them through training an AE. An AE can be viewed as a non-linear, NN counterpart of PCA, which is an unsupervised learning method \citep[see, for example, in][]{gu2019autoencoder,freire2023autoencoder}. We assume that the unknown DGP can be characterized by a set of arbitrary-dimensional latent parameters $\mathbf{\phi}^{\text{AE}}$. 
We use the encoder part of an AE to compress the observation vector $\hat{\mathbf{y}}$ into the latent parameter space, and use the decoder part to project the latent parameters back to the input observation vector. 
Therefore, the objective of training an AE is to minimize a reconstruction loss between the input samples and the AE-generated samples:
\begin{equation}
    \mathcal{L}_{\text{AE}}(\theta;\mathcal{D}_{\text{AE}}) = \sum_{i=1}^{N_{\text{AE}}} \left(f^{\text{AE}}_{\mathbf{\theta}}(\hat{\mathbf{y}}^{[i]})-\hat{\mathbf{y}}^{[i]}\right)^{2} \Big/ N_{\text{AE}}.
\end{equation}

The AE is also pre-trained offline, based on the dataset $\mathcal{D}_{AE}$ where it is believed to contain structured knowledge that resembles the DGP. The dataset $\mathcal{D}_{\text{AE}}$ can be obtained from expert simulations or collected directly from the real-world observations. Once the AE is trained, we use the decoder NN, $f^{\text{Dec}}_{\theta}(\mathbf{X}^{\text{SK}}, \mathbf{\phi}_{AE})$, to represent the NPSKR, that is, $g_{\mathbf{\phi}_{\text{AE}}}(\mathbf{X}^{\text{SK}})\equiv^{\text{Dec}}_{\mathbf{\theta}}(\mathbf{X}^{\text{SK}}, \mathbf{\phi}_{AE})$ where $\mathbf{\phi}_{\text{AE}}$ contains the learnable parameters but determined by a machine. 
We illustrate the AE-based NPSKRs in Figure (\ref{fig:non_param_format_repre_ae}).

\begin{figure}[htbp]
  \centering
  \begin{tikzpicture}[x=1.5cm, y=1.2cm] 
    \readlist\Nnod{3,4,3,2,3,4,3} 

    \foreachitem \N \in \Nnod{
      \foreach \i [evaluate={\x=\Ncnt; \y=\N/2-\i+0.5; \prev=int(\Ncnt-1);}] in {1,...,\N}{

        \ifnum\Ncnt=1 \def\isouter{1} \else \ifnum\Ncnt=7 \def\isouter{1} \else \def\isouter{0} \fi \fi
        
        \ifnum\isouter=1
          \ifnum\i=2
            \node (N\Ncnt-\i) at (\x,\y) {$\Large\bm{\vdots}$};
          \else
            \node[mynode, color=mylightred, draw=red] (N\Ncnt-\i) at (\x,\y) {};
          \fi
        \else
          \node[mynode, color=mylightred, draw=red] (N\Ncnt-\i) at (\x,\y) {};
        \fi

        \ifnum\Ncnt>1
          \foreach \j in {1,...,\Nnod[\prev]}{
            \draw[very thick] (N\prev-\j) -- (N\Ncnt-\i);
          }
        \fi
      }
    }

    \node[inner sep=0pt, yshift=3pt] at (4, 0) {$\Large\bm{\vdots}$};

    \node at (N1-1) {$\hat{\mathbf{y}}_{1}$};
    \node at (N1-3) {$\hat{\mathbf{y}}_{n}$};

    \node at (N7-1) {$\hat{\mathbf{y}}^{\text{SK}}_{1}$};
    \node at (N7-3) {$\hat{\mathbf{y}}^{\text{SK}}_{n}$};

    \node at (N4-1) {$\mathbf{\phi}_{\text{ae}}^{[1]}$};
    \node at (N4-2) {$\mathbf{\phi}_{\text{ae}}^{[d]}$};

    \draw[decorate, decoration={brace, amplitude=5pt, raise=37pt}] 
    (N1-1.center) -- ([xshift=-5pt] N4-1 |- N1-1) 
    node[midway, yshift=40.5pt, above] {encoder};

    \draw[decorate, decoration={brace, amplitude=5pt, raise=37pt}] 
    ([xshift=5pt] N4-1 |- N1-1) -- (N7-1 |- N1-1) 
    node[midway, yshift=40.5pt, above] {decoder};

    \draw[<-, thick, >=stealth] (4, -1.2) -- (4, -2.2) 
    node[below, align=center] {auto-encoded latent parameters};

  \end{tikzpicture}
  \caption{NPSKR by an AE. The inputs include a vector of realizations from an unknown DGP, with size $n$. The encoder compresses the inputs into a vector of latent parameters of size $d$ ($d\ll n$). The $d$-dimensional AE latent parameters are then passed to a decoder, which returns an $n$-dimensional vector that reconstructs the inputs.}
  \label{fig:non_param_format_repre_ae}
\end{figure}
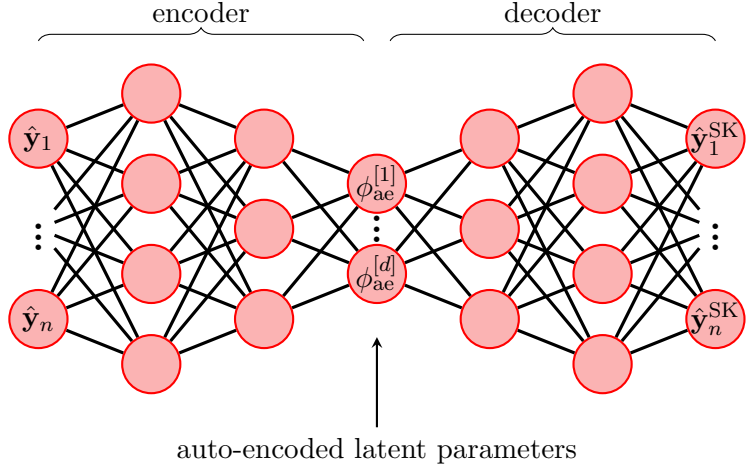

\subsection{The Composite Objective}\label{sec:formal_setup}

The SKINNs framework unifies the data-driven and the structured-knowledge components through a composite loss function. The data-driven component measures the discrepancy between the NN's predictions and the observed outcomes through the data-centric loss as in Equation (\ref{eq_data_centric_loss}).

The structured-knowledge component penalizes deviations between the NN's predictions and the structured-knowledge representation's output $g_\mathbf{\phi}$ via $\ell(\cdot,\cdot)$, e.g., a squared error:
\begin{align}
    \mathcal{L}_{\text{SK}}(\theta, \phi;\mathbf{X}^{\text{SK}}_{\text{Colloc}}) :&= \mathbb{E}\left[\ell(f_\theta(\mathbf{X}_{\text{Colloc}}), g_\phi(\mathbf{X}_{\text{Colloc}}^{\text{SK}}))\right] \\
    &= \frac{1}{N_{\text{Colloc}}}\sum_{i=1}^{N_{\text{Colloc}}}\left(f_{\theta}\left(\mathbf{X}^{[i]}_{\text{Colloc}}\right)-g_{\phi}(\mathbf{X}^{[i],\text{SK}}_{\text{Colloc}})\right)^{2}.
\end{align}
Crucially, this loss is evaluated on a set of collocation points, $\{\mathbf{X}_{\text{Colloc}}^{[i]}\}_{i=1}^{N_{\text{Colloc}}}$, which can be randomly sampled from the entire valid input domain of the NN. These points do not need to correspond to the observed data $\{\mathbf{X}_{\text{Obs}}^{[i]}\}_{i=1}^{N_{\text{Obs}}}$. This allows the framework to enforce theoretical consistency in regions of the feature space where observations may be sparse or absent, thereby promoting robust generalization.
Appendix~\ref{app:gen-stability}--\ref{app:gen-shift} provides formal generalization and target-risk bounds in a convex proxy, clarifying how structured regularization and collocation can improve robustness under distributional shifts.

The complete learning objective for a SKINN is the weighted sum of the data-centric and the structured-knowledge loss:
\begin{equation}\label{eq:total_loss}
    \mathcal{L}_{\text{Total}}(\mathbf{\theta}, \mathbf{\phi};\mathbf{X}_{\text{Obs}}\cup\mathbf{X}_{\text{Colloc}},\mathbf{y}_{\text{Obs}}) = \lambda_{\text{Data}}\underbrace{\mathcal{L}_{\text{Data}}(\mathbf{\theta}; \mathbf{X}_{\text{Obs}})}_{\text{Data fidelity}} + \lambda_{\text{SK}}\underbrace{\mathcal{L}_{\text{SK}}(\mathbf{\theta}, \mathbf{\phi}; \mathbf{X}_{\text{Colloc}})}_{\text{Theory consistency}}.
\end{equation}
The hyper-parameters $\lambda_{\text{data}}$ and $\lambda_{\text{SK}}$ control the trade-off between fitting the observed data and adhering to the embedded theoretical structure. Without loss of generality, we can normalize $\lambda_{\text{data}} = 1$ and denote $\lambda = \lambda_{\text{SK}}$, yielding:
\begin{align}
     \mathcal{L}_{\text{Total}}(\mathbf{\theta}, \mathbf{\phi};\cdots) &= \mathcal{L}_{\text{Data}}(\mathbf{\theta}; \mathbf{X}_{\text{Obs}}) + \lambda \mathcal{L}_{\text{SK}}(\mathbf{\theta}, \mathbf{\phi}; \mathbf{X}_{\text{Colloc}}) \\
                              &= \frac{1}{N_{\text{Obs}}}\sum_{i=1}^{N_{\text{Obs}}}\left(f_{\theta}(\mathbf{X}^{[i]}_{\text{Obs}}) -\hat{\mathbf{y}}_{\text{Obs}}^{[i]}\right)^{2} + \frac{\lambda}{N_{\text{Colloc}}}\sum_{i=1}^{N_{\text{Colloc}}}\left(f_{\theta}\left(\mathbf{X}^{[i]}_{\text{Colloc}}\right)-g_{\phi}(\mathbf{X}^{[i],\text{SK}}_{\text{Colloc}})\right)^{2}.
    \label{eq:normalized_objective}
\end{align}



\subsection{The Joint Learning of Both Parameter Sets}\label{sec:skinn_optim}

Our SKINNs are trained by jointly minimizing this composite loss with respect to both sets of parameters, e.g., the NN parameter and the structured-knowledge parameters:
\begin{equation}\label{eq:skinn_min}
    (\mathbf{\theta}^*, \mathbf{\phi}^*) =: \underset{\mathbf{\theta} \in \Theta, \mathbf{\phi} \in \Phi}{\arg\min}\,\mathcal{L}_{\text{Total}}(\mathbf{\theta}, \mathbf{\phi};\mathbf{X}_{\text{Obs}}\cup\mathbf{X}_{\text{Colloc}},\mathbf{y}_{\text{Obs}}).
\end{equation}

The optimization of Equation~\eqref{eq:skinn_min} is carried out using (stochastic) gradient-based methods at each iteration. While $\mathbf{\theta}$ can be astronomically-dimensional, $\mathbf{\phi}$ is usually relatively lower-dimensional and domain-interpretable (e.g.\ material constants in physics, preference or volatility parameters in economics, or control parameters in engineering). The key innovation of SKINNs is the simultaneous, gradient-based discovery of both parameter sets, where gradients of the composite loss with respect to $\mathbf{\theta}$ and $\mathbf{\phi}$ are computed via automatic differentiation, enabling the theoretical parameters to adapt dynamically to reconcile theory with empirical evidence so that the resulting function approximator respects domain-specific structure while maintaining flexibility. This joint optimization distinguishes SKINNs from approaches that pre-calibrate $\mathbf{\phi}$ separately or ignore the theoretical structure entirely. Figure~\ref{fig:skinn_architecture} illustrates the architecture of SKINNs.
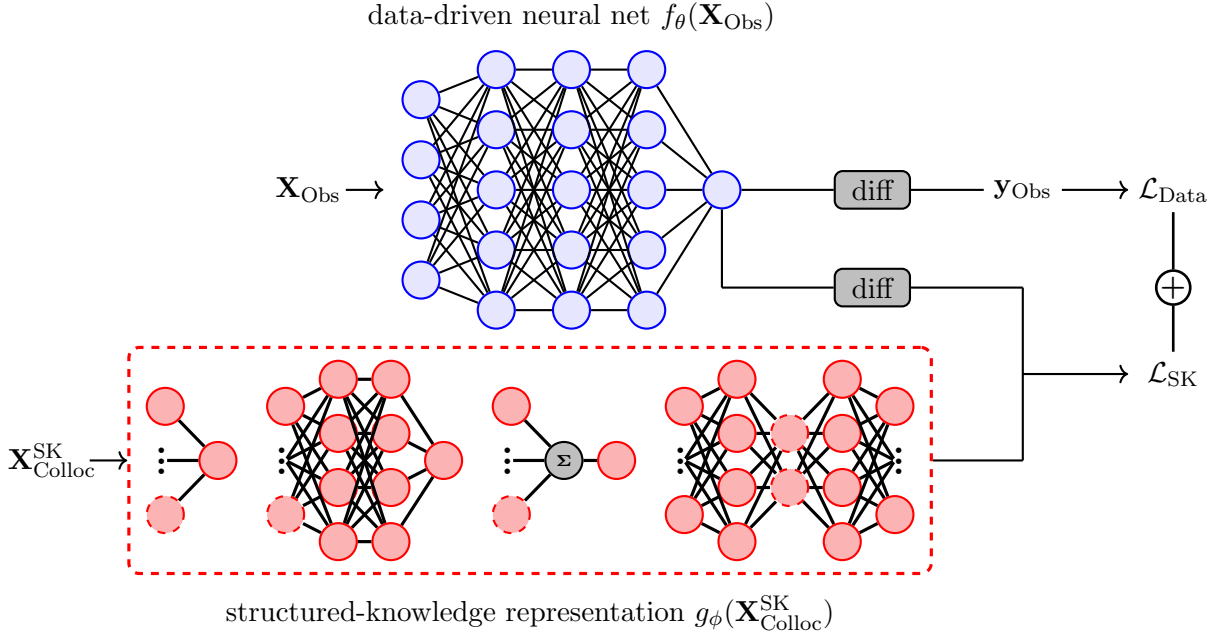
\begin{figure}[htbp]
\centering
\begin{tikzpicture}[
    x=1.0cm, y=0.8cm, 
    mynode/.style={circle, thick, minimum size=14pt, inner sep=0pt, font=\normalsize, draw=blue, fill=blue!10},
  ]

  \readlist\Nnod{4,5,5,5,1} 
  \foreachitem \N \in \Nnod{
    \foreach \i [evaluate={\x=\Ncnt; \y=\N/2-\i+0.5; \prev=int(\Ncnt-1);}] in {1,...,\N}{
      \node[mynode] (N\Ncnt-\i) at (\x,\y) {};
      \ifnum\Ncnt>1
        \foreach \j in {1,...,\Nnod[\prev]}{\draw[thick] (N\prev-\j) -- (N\Ncnt-\i);}
      \fi
    }
  }
  \node[
    draw=black,
    fill=lightgray,
    thick,
    rounded corners=3pt,
    minimum width=1.0cm,
    minimum height=0.5cm,
    inner sep=2pt,
    font=\normalsize
] at ([xshift=2cm]N5-1) (diff)
    {diff};

  \draw[thick,-] (N5-1) -- (diff);

  \node at ([xshift=2cm]diff) (yobs) {$\mathbf{y}_{\text{Obs}}$};

  \draw[thick,-] (diff) -- (yobs);

  \node at ([xshift=2cm]yobs) (data_loss) {$\mathcal{L}_{\text{Data}}$};

  \draw[thick,->] (yobs) -- (data_loss);

  \node[
    draw=black,
    fill=lightgray,
    thick,
    rounded corners=3pt,
    minimum width=1.0cm,
    minimum height=0.5cm,
    inner sep=2pt,
    font=\normalsize
] at ([yshift=-1.3cm]diff) (fin_diff)
    {diff};

  \draw[thick] (N5-1) -- ($ (N5-1) + (0,-1.3cm) $);

  \draw[thick] ($ (N5-1) + (0,-1.3cm) $) -- (fin_diff);

  \draw[thick] (fin_diff) -- ($ (yobs) + (0,-1.3cm) $);


  
  \begin{scope}[shift={(-1.3*2-0.5, -4.5)}]  

    \readlist\Fonenod{3,1}

    \foreachitem \N \in \Fonenod{
      \foreach \i [evaluate={\lx=\Ncnt*0.7; \ly=(\N/2 - \i + 0.5)*0.9;}] in {1,...,\N}{

        \ifnum\Ncnt=1
          \ifnum\i=2
            \node[inner sep=0pt, draw=none, fill=none] (B1-1-2) at (\lx,\ly) {};
            \node at ([xshift=-0.05cm,yshift=0.1cm]B1-1-2) {$\Large\bm{\vdots}$};
          \else\ifnum\i=3
            \node[mynode, color=mylightred, draw=red, dashed] (B1-1-3) at (\lx,\ly) {};
          \else
            \node[mynode, color=mylightred, draw=red] (B1-1-1) at (\lx,\ly) {};
          \fi\fi
        \else
          \node[mynode, color=mylightred, draw=red] (B1-\Ncnt-\i) at (\lx,\ly) {};
        \fi

      }
    }

    \foreach \i in {1,...,3}{
      \draw[very thick] (B1-1-\i) -- (B1-2-1);
    }


  \end{scope}

\begin{scope}[shift={(-0.5*2-0.5, -4.5)}]

  \readlist\FAnod{3,4,4,1}

  \foreachitem \N \in \FAnod{
    \foreach \i [evaluate={\lx=\Ncnt*0.7; \ly=(\N/2 - \i + 0.5)*0.9;}] in {1,...,\N}{

      \ifnum\Ncnt=1
        \ifnum\i=2
          \node[inner sep=0pt, draw=none, fill=none] (B2-1-2) at (\lx,\ly) {};
          \node at ([xshift=-0.05cm,yshift=0.1cm]B2-1-2) {$\Large\bm{\vdots}$};
        \else\ifnum\i=3
          \node[mynode, color=mylightred, draw=red, dashed] (B2-1-3) at (\lx,\ly) {};
        \else
          \node[mynode, color=mylightred, draw=red] (B2-1-1) at (\lx,\ly) {};
        \fi\fi
      \else\ifnum\Ncnt=4
        \ifnum\i=2
          \node[inner sep=0pt, draw=none, fill=none] (B2-4-2) at (\lx,\ly) {};
          \node at ([xshift=-0.05cm,yshift=0.1cm]B2-4-2) {$\Large\bm{\vdots}$};
        \else
          \node[mynode, color=mylightred, draw=red] (B2-4-\i) at (\lx,\ly) {};
        \fi
      \else
        \node[mynode, color=mylightred, draw=red] (B2-\Ncnt-\i) at (\lx,\ly) {};
      \fi\fi
    }
  }

  \foreach \i in {1,...,3}{ \foreach \j in {1,...,4}{ \draw[very thick] (B2-1-\i) -- (B2-2-\j); } }
  \foreach \i in {1,...,4}{ \foreach \j in {1,...,4}{ \draw[very thick] (B2-2-\i) -- (B2-3-\j); } }
  \foreach \i in {1,...,4}{ \draw[very thick] (B2-3-\i) -- (B2-4-1); }

\end{scope}

\begin{scope}[shift={(1*2-0.5, -4.5)}]

  \readlist\Fthreenod{3,1,1}

  \foreachitem \N \in \Fthreenod{
    \foreach \i [evaluate={\lx=\Ncnt*0.7; \ly=(\N/2 - \i + 0.5)*0.9;}] in {1,...,\N}{

      \ifnum\Ncnt=1
        \ifnum\i=2
          \node[inner sep=0pt, draw=none, fill=none] (B3-1-2) at (\lx,\ly) {};
          \node at ([xshift=-0.05cm,yshift=0.1cm]B3-1-2) {$\Large\bm{\vdots}$};
        \else\ifnum\i=3
          \node[mynode, color=mylightred, draw=red, dashed] (B3-1-3) at (\lx,\ly) {};
        \else
          \node[mynode, color=mylightred, draw=red] (B3-1-1) at (\lx,\ly) {};
        \fi\fi
      \else\ifnum\Ncnt=2
        \node[mynode, fill=lightgray, draw=black] (B3-\Ncnt-\i) at (\lx,\ly) {};
      \else
        \node[mynode, color=mylightred, draw=red] (B3-\Ncnt-\i) at (\lx,\ly) {};
      \fi\fi
    }
  }

  \foreach \i in {1,...,3}{ \draw[very thick] (B3-1-\i) -- (B3-2-1); }
  \draw[very thick] (B3-2-1) -- (B3-3-1);

  \node[font=\tiny] at (B3-2-1) {$\bm{\Sigma}$};

\end{scope}

\begin{scope}[shift={(2.15*2-0.5, -4.5)}]

  \readlist\FAnod{3,4,2,4,3}

  \foreachitem \N \in \FAnod{
    \foreach \i [evaluate={\lx=\Ncnt*0.7; \ly=(\N/2 - \i + 0.5)*0.9;}] in {1,...,\N}{

      \ifnum\Ncnt=1
        \ifnum\i=2
          \node[inner sep=0pt, draw=none, fill=none] (B4-1-2) at (\lx,\ly) {};
          \node at ([xshift=-0.05cm,yshift=0.1cm]B4-1-2) {$\Large\bm{\vdots}$};
        \else
          \node[mynode, color=mylightred, draw=red] (B4-1-\i) at (\lx,\ly) {};
        \fi
      \else\ifnum\Ncnt=3
        \node[mynode, color=mylightred, draw=red, dashed] (B4-3-\i) at (\lx,\ly) {};
      \else\ifnum\Ncnt=5
        \ifnum\i=2
          \node[inner sep=0pt, draw=none, fill=none] (B4-5-2) at (\lx,\ly) {};
          \node at ([xshift=+0.05cm,yshift=0.1cm]B4-5-2) {$\Large\bm{\vdots}$};
        \else
          \node[mynode, color=mylightred, draw=red] (B4-5-\i) at (\lx,\ly) {};
        \fi
      \else
        \node[mynode, color=mylightred, draw=red] (B4-\Ncnt-\i) at (\lx,\ly) {};
      \fi\fi\fi

    }
  }

  \foreach \i in {1,...,3}{ \foreach \j in {1,...,4}{ \draw[very thick] (B4-1-\i) -- (B4-2-\j); } }
  \foreach \i in {1,...,4}{ \foreach \j in {1,...,2}{ \draw[very thick] (B4-2-\i) -- (B4-3-\j); } }
  \foreach \i in {1,...,2}{ \foreach \j in {1,...,4}{ \draw[very thick] (B4-3-\i) -- (B4-4-\j); } }
  \foreach \i in {1,...,4}{ \foreach \j in {1,...,3}{ \draw[very thick] (B4-4-\i) -- (B4-5-\j); } }


\end{scope}

\node[
  draw, dashed, red, rounded corners, very thick,
  fit={
    ([yshift=0.3cm]B1-1-1.north west) 
    ([yshift=-0.3cm]B4-5-3.south east)
  },
  inner sep=0.3cm
] (allformats) {};

\draw[thick] (allformats.east) -- (yobs |- allformats.east);

\draw[thick] ($ (yobs) + (0,-1.3cm) $) -- (yobs |- allformats.east);

\draw[thick, ->] ($ (allformats.west) + (-0.5cm,0) $) -- (allformats.west);

\node at ($ (allformats.west) + (-1cm,0) $) (SK_inputs) {$\mathbf{X}^{\text{SK}}_{\text{Colloc}}$};

\draw[thick, ->] ($ (N2-3) + (-2.0cm,0) $) -- ($ (N2-3) + (-1.5cm,0) $);

\node at ($ (N2-3) + (-2.0cm,0) + (-0.5cm,0)$) (data_inputs) {$\mathbf{X}_{\text{Obs}}$};

\coordinate (top) at ($ (yobs) + (0,-1.3cm) $);
\coordinate (bottom) at (yobs |- allformats.east);

\coordinate (mid) at ($(top)!0.5!(bottom)$);

\draw[thick,->] (mid) -- (data_loss.west |- mid);

\node at (data_loss |- mid) (SK_loss) {$\mathcal{L}_{\text{SK}}$};

\node[draw, circle, minimum size=10pt, inner sep=0pt, line width=1pt] (composite) at (data_loss |- fin_diff) {$\bm{+}$};

\draw[line width=1pt] (data_loss) -- (composite);

\draw[line width=1pt] (SK_loss) -- (composite);

\node at ([yshift=0.7cm]N3-1) {data-driven neural net $f_{\mathbf{\theta}}(\mathbf{X}_{\text{Obs}})$};
\node at ([yshift=-0.5cm]allformats.south) {structured-knowledge representation $g_{\mathbf{\phi}}(\mathbf{X}^{\text{SK}}_{\text{Colloc}})$};




\end{tikzpicture}
\caption{The architecture of SKINNs. The NN on the top (blue circles) is the base function approximator. The mappings at the bottom (red circles) are structured-knowledge representations of different formats. All of them take observable features $\mathbf{X}^{\text{SK}}_{\text{Colloc}}$; the current epoch value of the learnable latent parameters as inputs, and output $g_{\phi}$.}
\label{fig:skinn_architecture}
\end{figure}




A natural concern is whether the structured-knowledge parameters $\phi$ can be correctly identified, jointly with the NN parameters. To this end, we provide two complementary results below. The first is a general profiling lemma showing that the two sets of parameters can be jointly identified, and the second provides a sufficient condition under which $\phi$ can be uniquely identified.

\begin{lemma}[Identification through profiling]\label{lem:profile-id}
Given the composite objective $\mathcal{L}(\mathbf{\theta},\mathbf{\phi})$ in Equation~\eqref{eq:normalized_objective}, define the profiled criterion
\begin{equation}\label{eq:profiled_criterion}
Q(\mathbf{\phi}) \equiv \inf_{\mathbf{\theta}\in\Theta}\, \mathcal{L}(\mathbf{\theta},\mathbf{\phi}).
\end{equation}
If $Q(\mathbf{\phi})$ has a unique minimizer $\mathbf{\phi}^{\ast}$ over $\Phi$, then $\mathbf{\phi}^{\ast}$ is identified as the unique structured-parameter value selected by the SKINNs objective. If, in addition, the minimizer in $\mathbf{\theta}$ at $\mathbf{\phi}^{\ast}$ is unique, then the joint minimizer $(\mathbf{\theta}^{\ast},\mathbf{\phi}^{\ast})$ is uniquely identified.
\end{lemma}

\begin{proof}
By definition, $(\theta^{\ast},\phi^{\ast})$ minimizes $L(\theta,\phi)$ over $\Theta\times\Phi$ if and only if $\phi^{\ast}$ minimizes the profiled criterion $Q(\phi)$ and $\theta^{\ast}\in\arg\min_{\theta}L(\theta,\phi^{\ast})$. Uniqueness of the minimizer of $Q(\phi)$ yields uniqueness of $\phi^{\ast}$; uniqueness of the minimizer in $\theta$ at $\phi^{\ast}$ yields uniqueness of the joint minimizer. \end{proof}

\begin{proposition}[A sufficient condition]\label{prop:closedform-id}
Consider the squared-error SKINNs objective in Equation~\eqref{eq:normalized_objective}. Assume:
(i) the function class $\{f_\theta\}$ is sufficiently rich so that minimizing over $\theta\in\Theta$ is equivalent to minimizing over all measurable functions $f:\mathcal{X}\to\mathbb{R}$ with finite second moments;
(ii) $\mathbf{X}_{\text{Obs}}$ and $\mathbf{X}_{\text{Colloc}}$ are random inputs with the same distribution, written $\mathbf{X}_{\text{Colloc}}\stackrel{d}{=}\mathbf{X}_{\text{Obs}}$, and $\mathbf{X}_{\text{Colloc}}^{\text{SK}}$ is a measurable function of $\mathbf{X}_{\text{Colloc}}$ (equivalently, $\mathbf{X}_{\text{Obs}}^{\text{SK}}$ is a measurable function of $\mathbf{X}_{\text{Obs}}$);
and (iii) $\mathbb{E}[\mathbf{y}^2]<\infty$ and $\mathbb{E}\!\left[g_{\phi}(\mathbf{X}_{\text{Colloc}}^{\text{SK}})^2\right]<\infty$ for all $\phi\in\Phi$.
Then, for each fixed $\phi$, the minimizer over $f_{\theta}$ exists and can be written pointwise in terms of $\mathbf{X}_{\text{Obs}}$ as
\begin{equation}\label{eq:f_star_phi}
f^{\ast}_{\theta}(\mathbf{X}_{\text{Obs}})
=
\frac{\mathbb{E}[\mathbf{y}\mid \mathbf{X}_{\text{Obs}}] + \lambda\, g_{\phi}(\mathbf{X}_{\text{Obs}}^{\text{SK}})}{1+\lambda}.
\end{equation}
Moreover, the profiled criterion $Q(\phi)\equiv \inf_{\theta\in\Theta}\mathcal{L}(\theta,\phi)$ satisfies, up to $\phi$-independent constants,
\begin{equation*}
Q(\phi)
=
\frac{\lambda}{1+\lambda}\,
\mathbb{E}\!\left[
\big(\mathbb{E}[\mathbf{y}\mid \mathbf{X}_{\text{Obs}}] - g_{\phi}(\mathbf{X}_{\text{Obs}}^{\text{SK}})\big)^2
\right]
\;+\; \text{const.}
\end{equation*}
Hence, if the mapping,
\begin{equation*}
\phi
\mapsto
\mathbb{E}\!\left[
\big(\mathbb{E}[\mathbf{y}\mid \mathbf{X}_{\text{Obs}}] - g_{\phi}(\mathbf{X}_{\text{Obs}}^{\text{SK}})\big)^2
\right],
\end{equation*}
has a unique minimizer over $\Phi$, then $\phi$ is identified.\footnote{In particular, the substitution of $\mathbf{X}_{\text{Obs}}$ for $\mathbf{X}_{\text{Colloc}}$ in all above formulas is purely notational: it uses $\mathbf{X}_{\text{Colloc}}\stackrel{d}{=}\mathbf{X}_{\text{Obs}}$ to express both expectations with a common conditioning variable.}
\end{proposition}

\begin{proof}
Fix $\phi$. Under Assumptions (i)--(iii), the population objective in \eqref{eq:normalized_objective} can be written as
\begin{equation}\label{eq:obj_fixed_phi}
\mathbb{E}\!\left[\big(f_{\theta}(\mathbf{X}_{\text{Obs}})-\mathbf{y}\big)^2\right]
+
\lambda\,\mathbb{E}\!\left[\big(f_{\theta}(\mathbf{X}_{\text{Colloc}})-g_{\phi}(\mathbf{X}_{\text{Colloc}}^{\text{SK}})\big)^2\right],
\end{equation}
where $f_{\theta}$ ranges over measurable functions with finite second moments. Since $\mathbf{X}_{\text{Colloc}}\stackrel{d}{=}\mathbf{X}_{\text{Obs}}$, the second expectation equals
\begin{equation*}
\mathbb{E}\!\left[\big(f_{\theta}(\mathbf{X}_{\text{Colloc}})-g_{\phi}(\mathbf{X}_{\text{Colloc}}^{\text{SK}})\big)^2\right]
=
\mathbb{E}\!\left[\big(f_{\theta}(\mathbf{X}_{\text{Obs}})-g_{\phi}(\mathbf{X}_{\text{Obs}}^{\text{SK}})\big)^2\right],
\end{equation*}
so both terms can be expressed using the common conditioning variable $\mathbf{X}_{\text{Obs}}$. Conditioning on $\mathbf{X}_{\text{Obs}}$ and minimizing pointwise gives, for almost every $\mathbf{X}_{\text{Obs}}$,
\begin{equation*}
f^{\ast}_{\theta}(\mathbf{X}_{\text{Obs}})
=
\arg\min_{u\in\mathbb{R}}
\ \mathbb{E}\!\left[(u-\mathbf{y})^2\mid \mathbf{X}_{\text{Obs}}\right]
+
\lambda\big(u-g_{\phi}(\mathbf{X}_{\text{Obs}}^{\text{SK}})\big)^2,
\end{equation*}
which is a strictly convex quadratic in $u$. The first-order condition is
\begin{equation}\label{eq:foc_u}
2\big(u-\mathbb{E}[\mathbf{y}\mid \mathbf{X}_{\text{Obs}}]\big)
+
2\lambda\big(u-g_{\phi}(\mathbf{X}_{\text{Obs}}^{\text{SK}})\big)
=0.
\end{equation}
Solving \eqref{eq:foc_u} yields \eqref{eq:f_star_phi}. Substituting \eqref{eq:f_star_phi} back into \eqref{eq:obj_fixed_phi} and completing the square yields
\begin{equation*}
Q(\phi)
=
\frac{\lambda}{1+\lambda}\,
\mathbb{E}\!\left[
\big(\mathbb{E}[\mathbf{y}\mid \mathbf{X}_{\text{Obs}}] - g_{\phi}(\mathbf{X}_{\text{Obs}}^{\text{SK}})\big)^2
\right]
+\text{const.},
\end{equation*}
where the constant does not depend on $\phi$. Uniqueness of the minimizer of the displayed term implies identification of $\phi$.
\end{proof}

\subsection{Desirable Statistical Properties of SKINNs}\label{sec:inference}

By correctly identifying the astronomically-dimensional NN parameters and the structured-knowledge parameters, the proposed SKINNs framework goes beyond a black-box predictive model. Integrating the discovery process of the latent structural parameters from structured-knowledge representations, the SKINNs framework instead functions as a white box that is more interpretable. The SKINNs framework also presents desirable statistical properties, generalizing many classic econometric tools, e.g., GMM and semi-parametric models.


\paragraph{Consistency and asymptotic normality.}
The SKINNs estimator belongs to the broad class of M-estimators \citep{huber1967behavior, van1996weak, newey1994large}. Under standard regularity conditions (identification, compactness, continuity, and uniform convergence of the empirical loss; see Appendix~\ref{sec:assumption}), the jointly learned parameters converge in probability, $(\hat{\theta}_{N}, \hat{\phi}_{N}) \xrightarrow{p} (\theta^{*}, \phi^{*})$, as $N \to \infty$ (Theorem~\ref{thm:consistency}). Under additional smoothness and moment conditions, the estimator is asymptotically normal at the parametric rate (Theorem~\ref{thm:asymptotic_normality}):
\begin{equation}
\sqrt{N} \begin{pmatrix} \hat{\theta}_N - \theta^* \\ \hat{\phi}_N - \phi^* \end{pmatrix} \xrightarrow{d} \mathcal{N}(\mathbf{0}, V), \qquad V = H^{-1}\Xi\, H^{-1},
\end{equation}
where $H$ denotes the Hessian of the population loss and $\Xi$ the covariance of the score. The sandwich-form covariance \citep{huber1967behavior, white1980heteroskedasticity} is robust to potential misspecification of the structured-knowledge component $g_{\phi}$, so that standard errors and test statistics remain asymptotically valid even if the embedded theory is an imperfect description of reality. These results operate in a double-asymptotic regime where the number of collocation points $M_N$ grows proportionally with the sample size ($M_N/N \to c$ for finite $c > 0$), ensuring that discretization error from $\mathcal{L}_{\text{SK}}$ does not degrade the $O_p(N^{-1/2})$ convergence. In practice, both the Hessian and score covariance can be estimated via automatic differentiation and outer products of gradients, and block-matrix inversion allows extraction of the $q \times q$ covariance submatrix for $\phi$ without inverting the full $(p+q) \times (p+q)$ Hessian, making confidence intervals and Wald-type hypothesis tests for $\phi$ computationally feasible even when the NN contains millions of parameters. The full development, including practical covariance estimation and bootstrap alternatives, appears in Appendix~\ref{sec:general_theory}.

\paragraph{A GMM interpretation and the role of $\lambda$.}
For researchers grounded in the econometric tradition, the composite SKINNs objective admits a revealing reinterpretation as a regularized, overidentified GMM problem \citep{hansen1982large}. The detailed development in Appendix~\ref{sec:gmm_interpretation} recasts the two loss components as moment conditions and identifies $\lambda$ as an econometric weighting parameter: a large $\lambda$ reflects strong confidence in the embedded theory, while a small $\lambda$ prioritizes empirical fit. This tradeoff is formalized through a variance-minimizing choice within a restricted diagonal weighting class (Proposition~\ref{prop:lambda-restricted}) and a closed-form benchmark yielding $\lambda^{*} = \sigma_{\varepsilon}^{2}/\sigma_{\eta}^{2}$ in a stylized two-signal setting (Proposition~\ref{prop:lambda-closedform}).

\paragraph{Semi-parametric efficiency.}
The GMM formulation connects naturally to semi-parametric estimation theory. Under orthogonal moment conditions, where $f_{\theta}$ exerts only a second-order influence on the estimation of $\phi$, the SKINNs estimator achieves the semi-parametric efficiency bound for the given moment model (Theorem~\ref{thm:sp-eff}; Appendix~\ref{app:sp-eff}). This places SKINNs within the family of sieve GMM estimators \citep{ai2003efficient} and deep GMM methods \citep{bennett2019deep, farrell2021deep}, while distinguishing it through the explicit incorporation of a learnable structured-knowledge component.

\paragraph{Generalization and robustness under distributional shift.}
Out-of-sample behavior is of particular importance in finance, where distributional shifts between training and deployment periods are the norm. The stability-based analysis in Appendix~\ref{app:gen-stability} shows that the structured-knowledge penalty tightens the uniform stability bound on the expected generalization gap (Theorem~\ref{thm:stability}). A complementary target-risk decomposition separates the risk under a shifted target distribution into alignment between $f_{\theta}$ and $g_{\phi}$ (controlled by $\mathcal{L}_{\text{SK}}$) and portability of $g_{\phi}$ to the new environment (Theorem~\ref{thm:target-decomp}). When the collocation distribution matches the target marginal, the alignment term is directly minimized during training (Corollary~\ref{cor:collocation-upper}). These results formalize a compelling intuition: if the embedded theoretical model captures structural regularities, such as no-arbitrage conditions or equilibrium relations, that remain stable across regimes, the SKINNs estimator inherits this stability.

\paragraph{Surrogate differentiability and the curse of dimensionality.}
Joint gradient-based training requires that $g_{\phi}$ be first-order differentiable with respect to both $\mathbf{X}^{\text{SK}}$ and $\phi$. For PSKRs, this is immediate, but for SPSKRs and machine-learned NPSKRs, the question is more nuanced. Appendix~\ref{sec:differntial_curse_dimensionality} establishes that the NN architecture of the surrogate guarantees end-to-end differentiability through the composite SKINNs objective, and shows that deep surrogates circumvent the curse of dimensionality that afflicts direct PDEs/SDEs embedding by scaling polynomially rather than exponentially with the input dimension.

\paragraph{SKINNs as a unified framework.}
A distinctive strength of the framework is that its composite objective is not tied to a single methodological tradition. As detailed in Appendix~\ref{sec:alternative_perspectives}, several well-established paradigms emerge as special cases or limiting configurations of the same objective: functional GMM, Bayesian MAP estimation, transfer learning, physics-informed learning, and domain adaptation, all of which reside within the SKINNs framework, with the specific instantiation determined by the choice of $g_{\phi}$, the loss structure, and the regularization strength $\lambda$. This multiplicity of valid interpretations reflects a deeper architectural point: SKINNs provide a single composite objective from which diverse methodological traditions can be recovered, while the joint optimization of $(\theta, \phi)$ enables capabilities, such as dynamic latent parameter discovery and bidirectional theory-data reconciliation, that none of the individual paradigms provide in isolation.

\section{SKINNs Option Pricing Models}\label{sec_models}
The early data-driven option pricing models of \cite{hutchinson_nonparametric_1994,ait-sahalia_nonparametric_1998} failed to incorporate domain knowledge until the works by \cite{garcia2000pricing, dugas2000incorporating, dugas2009incorporating,zheng_incorporating_2021}. Recently, \cite{chen2023teaching} incorporates the BSM model via transfer learning; \cite{aboussalah_ai_2024} integrates the dynamic hedging principles similar to PINNs. These hybrid models rely on simplified or model-free constraints due to their deficient architectures. In this section, we review a series of structured-knowledge specifications that can be seamlessly embedded into SKINNs, leading to more robust and interpretable model reasoning.


\subsection{Option Pricing Structured Knowledge}\label{sec:option_pricing_structured_knowledge_representations}
\paragraph{The BSM representation.} The seminal BSM model requires only one-dimensional latent parameter $\sigma$. It assumes the asset price dynamics, $dS_{t} = r S_{t} dt + \sigma S_{t} dW_{t}$, where $W_{t}$ is a standard Brownian motion. The BSM model admits a closed-form solution, and hence a PSKR function $g_{\phi}$:
\begin{align}
    g^{\text{BS}}(\mathbf{X}^{\text{SK}};\mathbf{\phi}\in\{\sigma^{\star}\}) &= S_{t}\Phi\left(d_{1}(\;\sigma^{\star}\;)\right) - Ke^{-r(T-t)}\Phi\left(d_{2}(\;\sigma^{\star}\;)\right), \\
    d_{1}(\mathbf{\phi}\in\{\sigma^{\star}\}) &= \frac{1}{\;\sigma^{\star}\;\sqrt{T-t}}\left(\log(S_{t}/K) + (r+\frac{\;(\sigma^{\star})^{2}\;}{2})(T-t)\right), \\
    d_{2}(\mathbf{\phi}\in\{\sigma^{\star}\}) &= d_{1}(\;\sigma^{\star}\;) - \;\sigma^{\star}\;\sqrt{(T-t)},
\end{align}
where $\Phi(\cdot)$ is the cumulative distribution function of a standard Gaussian distribution, the asset price $S_{t}$, the strike price $K$, the risk-free rate $r$, and the maturity $T$ are given in the input features $\mathbf{X}^{\text{SK}}$. The risk-neutral volatility $\sigma^{\star}$ in this case serves as the one-dimensional learnable latent structural parameter in $g^{\text{BS}}_{\mathbf{\phi}}(\mathbf{X}^{\text{SK}})$.\footnote{Throughout this paper, We use $\star$ to indicate the learnable latent parameters in SKINNs, whose optimal solutions are indicated with *.}

\paragraph{The ABSM representation.} One can relax the constant volatility assumption in the BSM model by using the ABSM model \citep[][]{dumas1998implied}. The ABSM model assumes an ad-hoc linear form specification of the implied volatilities, allowing them to vary across strikes and maturities:
\begin{align}
    \mathbf{\sigma} = \mathbf{X}\mathbf{\alpha} + \mathbf{\varepsilon},\;\;\;\;\;\;\;&\\
    \text{where},\;\mathbf{X} = \left[\mathbf{m}, \mathbf{m}^{2}, \mathbf{T}, \mathbf{T}^{2}, \mathbf{mT}\right]&, \mathbf{\alpha} = \left[\alpha_{0}, \alpha_{1}, \alpha_{2}, \alpha_{3}, \alpha_{4}, \alpha_{5}\right]^{\top},
\end{align}
and $\mathbf{m}=\frac{\mathbf{K}}{\mathbf{S}}$ denotes a vector of moneyness; $\mathbf{T}$ denotes the vector of maturities; $\mathbf{X}$ is a matrix of covariates and $\alpha$ is the vector of coefficients.

By viewing the unknown parameters $\{\alpha_{0}^{\star},\alpha_{1}^{\star},\alpha_{2}^{\star},\alpha_{3}^{\star},\alpha_{4}^{\star},\alpha_{5}^{\star}\}$ as the learnable latent structural parameters, this gives rise to a PSKR function $g_{\phi}$, with the use of the BSM formula:
\small\begin{align}
     g^{\text{AHBS}}(&\mathbf{X}^{\text{SK}};\phi\in\{\alpha_{0}^{\star},\cdots,\alpha_{5}^{\star}\}) = S_{t}\Phi\left(d_{1}(\alpha_{0}^{\star},\cdots,\alpha_{5}^{\star})\right) - Ke^{-r(T-t)}\Phi\left(d_{2}(\alpha_{0}^{\star},\cdots,\alpha_{5}^{\star})\right), \\
                        d_{1}(\phi\in\{\alpha_{0}^{\star},\cdots,\alpha_{5}^{\star}\}) &= \frac{1}{(\alpha_{0}^{\star}+\alpha_{1}^{\star}m + \alpha_{2}^{\star}m^{2} + \alpha_{3}^{\star}(T-t) + \alpha_{4}^{\star}(T-t)^{2} +\alpha_{5}^{\star}m(T-t))\sqrt{(T-t)}} \nonumber\\
                        \times&\left(\log(S_{t}/K) + (r + \frac{(\alpha_{0}^{\star}+\alpha_{1}^{\star}m + \alpha_{2}^{\star}m^{2} + \alpha_{3}^{\star}(T-t) + \alpha_{4}^{\star}(T-t)^{2} +\alpha_{5}^{\star}m(T-t))^{2}}{2})(T-t)\right), \nonumber\\
                        d_{2}(\phi\in\{\alpha_{0}^{\star},\cdots,\alpha_{5}^{\star}\}) &= d_{1}(\alpha_{0}^{\star},\cdots,\alpha_{5}^{\star}) -(\alpha_{0}^{\star}+\alpha_{1}^{\star}m + \alpha_{2}^{\star}m^{2} + \alpha_{3}^{\star}(T-t) + \alpha_{4}^{\star}(T-t)^{2} +\alpha_{5}^{\star}m(T-t))\sqrt{(T-t)}. \nonumber
\end{align}\normalsize

\paragraph{The SABR representation.}
Another way of extending the BSM model is to use the stochastic alpha-beta-rho (SABR) volatility model. As introduced in \cite{hagan2002managing}, the SABR model assumes the following SDE system:
\begin{align}
    dF_{t} &= \alpha_{t}F_{t}^{\beta} dW_{t}^{1}, \label{eq_dynamic_sabr_F} \\
    d\alpha_{t} &= v_{t}\alpha_{t} dW_{t}^{2}, \label{eq_dynamic_sabr_alpha}\\
    \text{where}\;\;\langle dW^{1}, dW^{2}\rangle_{t}&=\rho_{t} dt,\;\;F_{0}=\hat{f},\;\;\alpha_{0}=\alpha\label{eq_dynamic_sabr_covar},
\end{align}
and $F_{t}=S_{t}e^{r(T-t)}$ denotes the forward price of the underlying asset; $\alpha_{t}$ denotes the volatility process for the asset return; $v_{t}$ is a time-dependent function of the volatility-of-volatility; $dW^{1}_{t}, dW^{2}_{t}$ are two correlated Brownian motions, where the correlation coefficient $\rho_{t}$ is another time-dependent function. Since the time-variant nature of the diffusion coefficient $v_{t}$ (the volatility-of-volatility) and the correlation $\rho_{t}$, the SABR model described by equation (\ref{eq_dynamic_sabr_F})-(\ref{eq_dynamic_sabr_covar}) is dynamic. \cite{osajima2007asymptotic} provides an approximation to its implied volatility surface:
\begin{align}
    \sigma_{\text{SABR}}(K,T,\hat{f}) &\approx \frac{1}{w}\left(1 + A_{1}(T)\log\left(\frac{K}{\hat{f}}\right) + A_{2}(T)\log^{2}\left(\frac{K}{\hat{f}}\right) + B(T)T\right), 
\end{align}
where $w:=\frac{\hat{f}^{1-\beta}}{\alpha}$, $A_{1}(T)$, $A_{2}(T)$, $B(T)$ are functions of the dynamic SABR latent structural parameters.\footnote{These functions are fully-determined by, and differentiable w.r.t. the latent parameters: \begin{align}A_{1}(T) &= \frac{\beta-1}{2} + \frac{\eta_{1}(T)w}{2}, \\ A_{2}(T) &= \frac{(1-\beta)^{2}}{12} + \frac{1-\beta-\eta_{1}(T)w}{4} + \frac{4v_{1}^{2}(T)+3\left(\eta_{2}^{2}(T) + 3\eta_{1}^{2}(T)\right)}{24}w^{2}, \\ B(T) &= \frac{1}{w^{2}}\left(\frac{(1-\beta)^{2}}{24} + \frac{w\beta\eta_{1}(T)}{4} + \frac{2v_{2}^{2}(T)-3\eta_{2}^{2}(T)w^{2}}{24}\right), \text{where,}\\ v_{1}^{2}(T) &= \frac{3}{T^{3}}\int_{0}^{T}(T-t)^{2}v^{2}_{t}dt, \label{eq_dynamic_sabr_v1sq}\\ v_{2}^{2}(T) &= \frac{6}{T^{3}}\int_{0}^{T}(T-t)tv^{2}_{t}dt, \label{eq_dynamic_sabr_v2sq}\\ \eta_{1}(T) &= \frac{2}{T^{2}}\int_{0}^{T}(T-t)v_{t}\rho_{t}dt, \text{and,}\label{eq_dynamic_sabr_eta1}\\ \eta_{2}(T) &= \frac{12}{T^{4}}\int_{0}^{T}\int_{0}^{t}\left(\int_{0}^{s}v_{u}\rho_{u}du\right)^{2}dsdt.\label{eq_dynamic_sabr_eta2}\end{align}} There are some special parametric specifications for the time-dependent functions $v_{t}$, $\rho_{t}$, such that $v_{t}>0, -1\leq\rho_{t}\leq1, \forall t\in[0, T]$, requiring different numbers of structural parameters.\footnote{For example, $v_{t}, \rho_{t}$ can be specified as two constants, $v_{t}=v_{0},\rho_{t}=\rho_{0}$ (two parameters, $v_{0},\rho_{0}$); as classical formulations, $v_{t}=v_{0}e^{-bt},\rho_{t}=\rho_{0}e^{-at}$ (four parameters, $v_{0},\rho_{0},a,b$); as piece-wise functions, $v_{t}=v_{0}, \rho_{t}=\rho_{0},\forall t\leq T^{*}$, and $v_{t}=v_{1}, \rho_{t}=\rho_{1},\forall t>T^{*}$ (five parameters, $v_{0},\rho_{0},v_{1},\rho_{1},T^{*}$); or as more general formulations, $v_{t}=(v_{0}+q_{v}t)e^{-bt}+d_{v}$, and $\rho_{t}=(\rho_{0}+q_{\rho}t)e^{-at}+d_{\rho}$ (eight parameters, $v_{0},\rho_{0},q_{v},q_{\rho},a,b,d_{v},d_{\rho}$). See the details of the dynamic SABR volatility models in \cite{hagan2002managing, gatheral2011volatility}.} Moreover, thanks to the scalability of SKINNs to high-dimensional estimation problems, one can even treat $v_{t}$ and $\rho_{t}$ themselves as the learnable parameters by SKINNs. 
In our implementation, we also consider $v_{t}=\{v_{t_{i}}\}_{i=1}^{360}$ and $\rho_{t}=\{\rho_{t_{i}}\}_{i=1}^{360}$ as 720 latent parameters in total, where $t_{i}=\frac{i}{360}, 1\leq i\leq360$.

By viewing $\{v^{\star}_{t_{i}}\}_{i=1}^{360}$, $\{\rho^{\star}_{t_{i}}\}_{i=1}^{360}$, $\alpha^{\star}$ and $\beta^{\star}$ as the learnable latent economic parameters, this gives rise to a PSKR function $g_{\phi}$, with the use of the BSM formula:
\begin{align}
    g^{\text{SABR}}(\mathbf{X}^{\text{SK}};\phi\in\{v^{\star}_{t_{i}}\}_{i=1}^{360}\cup&\{\rho^{\star}_{t_{i}}\}_{i=1}^{360}\cup\{\alpha^{\star}, \beta^{\star}\}) = S_{t}\Phi\left(d_{1}\left(\sigma_{\text{SABR}}(\;\{v^{\star}_{t_{i}}\}_{i=1}^{360}\cup\{\rho^{\star}_{t_{i}}\}_{i=1}^{360}\cup\{\alpha^{\star}, \beta^{\star}\}\;)\right)\right) \nonumber\\
    &\;\;\;\;\;- Ke^{-r(T-t)}\Phi\left(d_{2}\left(\sigma_{\text{SABR}}(\;\{v^{\star}_{t_{i}}\}_{i=1}^{360}\cup\{\rho^{\star}_{t_{i}}\}_{i=1}^{360}\cup\{\alpha^{\star}, \beta^{\star}\}\;)\right)\right),
\end{align}
in which specification, the learnable latent economic parameter vector $\phi$ becomes 722-dimensional.

\paragraph{The HSV representation.} The capability of SKINNs in dealing with highly non-linear structured-knowledge representations and high-dimensional latent parameters enables us to incorporate more sophisticated $g_{\phi}$, for example, the HSV model below:
\begin{align}
            dS_{t} &= rS_{t}\,dt + \sqrt{v_{t}}S_{t}\,dW^{1}_{t}, \label{eq:heston_sde_S}\\
            dv_{t} &= \kappa\left(v_{\theta}-v_{t}\right)\,dt + \sigma_{v}\sqrt{v_{t}}\,dW^{2}_{t}, \label{eq:heston_sde_v}\\
                \langle dW^{1},dW^{2}\rangle_{t} &= \rho dt, \label{eq:heston_sde_rho}
\end{align}
where $v_{t}$ is the instantaneous variance that is driven by a mean-reverting process. $\kappa$ is the mean-reverting speed of the variance process, $v_{\theta}$ is the long-term average of the variance, and $\sigma_{v}$ is the volatility of the variance. $\rho$ is the correlation of the Brownian motions from the two processes.

For such sophisticated theory-driven models, convenient closed-form solutions are unavailable, and numerical methods are needed. The HSV model belongs to the affine class of SDEs that are solvable with characteristic functions \citep{carr_option_1999, duffie2000transform}.\footnote{The log-return characteristic function of the Heston model is given as:
\begin{align}
        \psi_{T\vert t}(u;\phi) &= \exp\left[C_{T\vert t}\left(u;\phi\right)v_{\theta} + D_{T\vert t}\left(u;\phi\right)v_{0} + iu\log\left(Se^{r(T-t)}\right)\right], \label{eq_heston_ccf}
\end{align}
where $v_{0}$ is the initial instantaneous variance, and $C_{T\vert t}, D_{T\vert t}$ are functions of the HSV model latent structural parameters, known as:
\begin{align}
        C_{T\vert} &= \kappa\sigma_{v}^{-2}\left((\kappa-\rho\sigma_{v}ui-d)(T-t)-2\log\left(\frac{1-ge^{-d(T-t)}}{1-g}\right)\right), \\
        D_{T\vert} &= \sigma_{v}^{-2}\left(\kappa-\rho\sigma_{v}iu-d\right)\left(1-e^{-d(T-t)}\right)\Big/\left(1-ge^{-d(T-t)}\right), 
\end{align}where the function $d$ and $g$ are defined as:
\begin{align}
        d &= ((\rho\sigma_{v}ui-\kappa)^{2}-\sigma_{v}^{2}(-iu-u^{2}))^{1/2}, \\
        g &= (\kappa-\rho\sigma_{v}ui-d) \big/ (\kappa-\rho\sigma_{v}ui+d).
\end{align}}
Knowing the characteristic function $\psi_{T\vert t}(u;\phi)$, European option prices can be evaluated as a semi-closed-form expression using FFT techniques\citep{carr_option_1999}. 

We employ the COS method developed by \cite{fang_novel_2009} to make the Fourier inversion with the log-return characteristic function $\psi_{T\vert t}(u;\phi)$ more efficient, which gives rise to a PSKR function $g_{\phi}$:
\begin{align}\label{COS_Heston}
    g^{\text{HSV}}(\mathbf{X}^{\text{SK}};\phi\in\{v_{\theta}^{\star},v_{0}^{\star},\sigma_{v}^{\star},\rho^{\star},\kappa^{\star}\}) &= \frac{1}{2}e^{-r(T-t)}\mathfrak{Re}\left\{\psi_{T\vert t}\left(0;S,r,\;v_{\theta}^{\star},v_{0}^{\star},\sigma_{v}^{\star},\rho^{\star},\kappa^{\star}\;\right)\right\}V_{0}(K) \\
                            +e^{-r(T-t)}\sum_{w=1}^{N-1}&\mathfrak{Re}\left\{\psi_{T\vert t}\left(\frac{w\pi}{b-a};S,r,\;v_{\theta}^{\star},v_{0}^{\star},\sigma_{v}^{\star},\rho^{\star},\kappa^{\star}\;\right)e^{-iw\pi\frac{a}{b-a}}\right\}V_{w}\left(K\right),\nonumber
\end{align}
where $a, b, N$ are the algorithmic coefficients of the COS method that control the numerical precision, and $V_{w}(K)$ denotes the option payoff series coefficients, which can be known analytically for European options \citep[see details in][]{fang_novel_2009}.

\paragraph{The HSVJ representation.}
The HSV model can be extended to incorporate a jump process in the asset price dynamics. We consider an extension of the HSV model with a double-exponential jump process \citep{kou2002jump}:
\begin{align}
            dS_{t} &= rS_{t}\,dt + \sqrt{v_{t}}S_{t}\,dW^{1}_{t} + S_{t}(1-e^{J_{t}})dN_{t}, \\
            dv_{t} &= \kappa\left(v_{\theta}-v_{t}\right)\,dt + \sigma_{v}\sqrt{v_{t}}\,dW^{2}_{t},\\
                \langle dW^{1},dW^{2}\rangle_{t} &= \rho dt,
\end{align}
where the process $J_{t}$ represents the size of relative price jumps, which has a double-exponential probability density defined by:
\begin{equation}
    f_{J}(y)=p\eta_{1}e^{-\eta_{1}y}\mathbbm{1}_{\left\{y\ge0\right\}} + (1-p)\eta_{2}e^{-2\eta_{2}y}\mathbbm{1}_{\left\{y<0\right\}},
\end{equation}
where $p\in[0,1]$ is the probability of a positive jump in the asset return ($1-p$ is the probability of a negative jump), $\eta_{1}$ is the size of a positive jump, and $\eta_{2}$ is the size of a negative jump. $N_{t}$ is a Poisson process of rate $\lambda$ that controls the arrival time of jumps. Such an extension further adds four extra parameters to the HSV model.

The HSVJ model still lies within the affine class, and the log-return characteristic function is the product of the HSV model characteristic function, the equation (\ref{eq_heston_ccf}), and the characteristic function for the jump component:
\begin{equation}
    \psi^{J}_{T\vert t}(u;\phi\in\{p^{\star},\eta^{\star}_{1},\eta^{\star}_{2},\lambda^{\star}\}) = \exp\left[\;\lambda^{\star}\;(T-t)\left(\frac{\;p^{\star}\;\eta_{1}^{\star}}{\;\eta_{1}^{\star}\;-iu} + \frac{(1-\;p^{\star}\;)\;\eta_{2}^{\star}\;}{\;\eta_{2}^{\star}\;+iu}-1\right)\right].
\end{equation}
The compounted log-return characteristic function for this model becomes $\psi_{T\vert t}(u;\phi\in\{v_{\theta}^{\star}, v_{0}^{\star}, \sigma_{v}^{\star}, \rho^{\star}, \kappa^{\star}\})\times\psi_{T\vert t}^{J}(u;\phi\in\{p^{\star},\eta_{1}^{\star},\eta_{2}^{\star},\lambda^{\star}\})$, and we still rely on the COS method to evaluate the option prices given the characteristic function. This gives rise to a PSKR function $g_{\phi}$:
\begin{align}\label{COS_HSVKDEJ}
    g^{\text{HSVJ}}(\mathbf{X}^{\text{SK}};\phi&\in\{v_{\theta}^{\star},v_{0}^{\star},\sigma_{v}^{\star},\rho^{\star},\kappa^{\star},p^{\star},\eta_{1}^{\star},\eta_{2}^{\star},\lambda^{\star}\}) = \frac{1}{2}e^{-r(T-t)}\\
    &\times\mathfrak{Re}\left\{\psi_{T\vert t}\left(0;S,r,\;v_{\theta}^{\star},v_{0}^{\star},\sigma_{v}^{\star},\rho^{\star},\kappa^{\star}\;\right)\times\psi_{T\vert t}^{J}(0;\;p^{\star},\eta_{1}^{\star},\eta_{2}^{\star},\lambda^{\star}\;)\right\}V_{0}(K) \nonumber\\
                           +e^{-r(T-t)}\sum_{w=1}^{N-1}&\mathfrak{Re}\left\{\psi_{T\vert t}\left(\frac{w\pi}{b-a};S,r,\;v_{\theta}^{\star},v_{0}^{\star},\sigma_{v}^{\star},\rho^{\star},\kappa^{\star}\right)\psi_{T\vert t}^{J}\left(\frac{w\pi}{b-a};\;p^{\star},\eta_{1}^{\star},\eta_{2}^{\star},\lambda^{\star}\;\right)e^{-iw\pi\frac{a}{b-a}}\right\}V_{w}\left(K\right).\nonumber
\end{align}


\paragraph{The DSNN-HSV representation.} We draw 5 million instances of the HSV model inputs independently from a multivariate uniform distribution. 
For each of the drawn instances, we query the European call option price using the HSV SDEs, as described in Equation (\ref{eq:heston_sde_S})-(\ref{eq:heston_sde_rho}) with a Monte-Carlo simulation of 1,000 runs. A DSNN, $f^{\text{SR}}_{\mathbf{\theta}}$, is pre-trained based on the drawn inputs and the queried option prices. This gives rise to a SPSKR function $g_{\phi}$:
\begin{equation}
    g^{\text{DSNN-HSV}}(\mathbf{X}^{\text{SK}}, \theta^{\text{SR}};\phi\in\{v_{\theta}^{\star},v_{0}^{\star},\sigma_{v}^{\star},\rho^{\star},\kappa^{\star}\}) = f^{\text{SR}}_{\mathbf{\theta}}(\mathbf{X}^{\text{SK}}, v_{\theta}^{\star},v_{0}^{\star},\sigma_{v}^{\star},\rho^{\star},\kappa^{\star}),
\end{equation}
where the inputs $v_{\theta}^{\star},v_{0}^{\star},\sigma_{v}^{\star},\rho^{\star},\kappa^{\star}$ of the DSNN become the learnable latent structural parameters in this SKINN specification.

\paragraph{The DSNN-NASV representation.} We also consider an NASV SDE system, where there is no known characteristic function for log returns. This system is nested in the general three-factor jump diffusion setting in Section (2), Equation (1)-(3) from \cite{kaeck2012volatility}. In particular, we consider the following non-affine asset price dynamics:
\begin{align}
    dS_{t} &= rS_{t}dt + \sqrt{v_{t}}S_{t}dW^{1}_{t}, \label{eq:nasv_sde_S}\\
    dv_{t} &= \kappa(v_{\theta} - v_{t})dt + \sigma_{v}v_{t}^{\gamma_{v}}dW^{2}_{t}, \label{eq:nasv_sde_v}\\
    \langle dW^{1},dW^{2}\rangle_{t} &= \rho dt,
\end{align}
where when $\gamma_{v}=1/2$, the SDE system reduces to a FFT-solvable HSV model.

We draw 50 million instances of $\{\mathbf{x}^{\text{SK}}_{1}, \mathbf{x}^{\text{SK}}_{2}, \cdots, \mathbf{x}^{\text{SK}}_{d_{\text{SK}}}, v_{\theta}, v_{0}, \sigma_{v}, \rho, \kappa, \gamma_{v}\}$ in this case from a multivariate uniform distribution. 
For each of the drawn instances, we again query the European call option prices using the SDEs, Equations (\ref{eq:nasv_sde_S})-(\ref{eq:nasv_sde_v}), with a Monte-Carlo simulation of 1,000 runs. Since there are more inputs for the NASV model than the HSV model, it requires a larger number of instances to span the input space. Once the DSNN for this model is pre-trained, it gives rise to an SPSKR function $g_{\phi}$:
\begin{equation}
    g^{\text{DSNN-NASV}}(\mathbf{X}^{\text{SK}},\theta^{\text{SR}};\phi\in\{v_{\theta}^{\star},v_{0}^{\star},\sigma_{v}^{\star},\rho^{\star},\kappa^{\star},\gamma_{v}^{\star}\}) = f^{\text{SR}}_{\mathbf{\theta}}(\mathbf{X}^{\text{SK}}, \;v_{\theta}^{\star},v_{0}^{\star},\sigma_{v}^{\star},\rho^{\star},\kappa^{\star},\gamma_{v}^{\star}\;),
\end{equation}
where the inputs $v_{\theta}^{\star},v_{0}^{\star},\sigma_{v}^{\star},\rho^{\star},\kappa^{\star},\gamma_{v}^{\star}$ of the DSNN become the learnable latent structural parameters in this SKINN specification.


\paragraph{The MOPA representation.} The non-parametric MOPA to construct structured knowledge is to utilize Equation (\ref{eq:mopa_empirical}). The probabilities $\mathbb{Q}(z_{1}), \mathbb{Q}(z_{2}), \cdots, \mathbb{Q}(z_{N_{z}})$ govern the conditional expectation of the underlying asset terminal price at one specific terminal time $T$. In practice, an option pricing model is estimated with multiple maturities $T_{h}, 1\leq h\leq s, \tau_{h}=T_{h}-t$. Therefore, we design a risk-neutral probability matrix:
\begin{equation}
        \mathbb{Q}(S_{T})_{s\times q} = \begin{bmatrix}
                                            \mathbb{Q}(S_{T_{1}}^{[1]}) & \mathbb{Q}(S_{T_{1}}^{[2]}) & \cdots & \mathbb{Q}(S_{T_{1}}^{[q]}) \\
                                            \mathbb{Q}(S_{T_{2}}^{[1]}) & \mathbb{Q}(S_{T_{2}}^{[2]}) & \cdots & \mathbb{Q}(S_{T_{2}}^{[q]}) \\
                                            \vdots & \vdots  & \ddots & \vdots \\
                                            \mathbb{Q}(S_{T_{s}}^{[1]}) & \mathbb{Q}(S_{T_{s}}^{[2]}) & \cdots & \mathbb{Q}(S_{T_{s}}^{[q]})
                                \end{bmatrix},
\end{equation}
where $S^{[i]}_{T_{h}}, 1\leq i\leq q, 1\leq h\leq s$ is the $i$-th state of nature for the asset price at time $T_{h}$. In our implementation, we assume that different maturities share a common possible states of terminal asset price $\{0.5\times S_{t}, \cdots, 1.5\times S_{t}\}$ which is an array with 200 equally-spaced elements, ranging from the half, to the one and a half, of the current asset price $S_{t}$. We assume 10 time-to-maturities $\{0.1, 0.2, \cdots, 1.0\}$. 
The risk-neutral probability matrix $\mathbb{Q}(S_{T})_{s\times q}$ thus contains 2,000 learnable latent parameters.

For maturity $T_{h}$, there are $N_{h}$ options with different strike prices $K_{1}, K_{2}, \cdots, K_{N_{h}}$, in ascending order. The terminal payoffs, considering all the states of terminal asset prices, are known analytically as:
\begin{equation}
    \mathbf{V}_{T_{h}} = \begin{bmatrix}
                            \left(S_{T_{h}}^{[1]}-K_{1}\right)^{+} & \left(S_{T_{h}}^{[2]}-K_{1}\right)^{+} & \cdots & \left(S_{T_{h}}^{[q]}-K_{1}\right)^{+} \\
                            \left(S_{T_{h}}^{[1]}-K_{2}\right)^{+} & \left(S_{T_{h}}^{[2]}-K_{2}\right)^{+} & \cdots & \left(S_{T_{h}}^{[q]}-K_{2}\right)^{+} \\
                            \vdots & \vdots & \ddots & \vdots \\
                            \left(S_{T_{h}}^{[1]}-K_{N_{h}}\right)^{+} & \left(S_{T_{h}}^{[2]}-K_{N_{h}}\right)^{+} & \cdots &  \left(S_{T_{h}}^{[q]}-K_{N_{h}}\right)^{+} \\
                    \end{bmatrix},
\end{equation}
for example, for European call options. We stack the terminal payoff matrix for each maturity as a block-diagonal matrix:
\begin{equation}
    (\mathbf{V}_{T})_{N\times sq} = \begin{bmatrix}
                    \mathbf{V}_{T_{1}} & \mathbf{0} & \mathbf{0} & \mathbf{0}\\
                    \mathbf{0} & \mathbf{V}_{T_{2}} & \mathbf{0} & \mathbf{0}\\
                    \mathbf{0} & \mathbf{0} & \ddots & \mathbf{0} \\
                    \mathbf{0} & \mathbf{0} & \mathbf{0} & \mathbf{V}_{T_{s}}
            \end{bmatrix}.
\end{equation}
This leads to a NPSKR function $g_{\phi}$ with unknown underlying distributions:
\begin{equation}
    g^{\text{MOPA}}(\mathbf{X}^{\text{SK}}; \mathbf{\phi}\in\{\mathbb{Q}_{T_{h}}^{\star,[j]}\}_{1\leq h\leq s, 1\leq j\leq q}) = D_{t}(T)\times\mathbf{V}_{T}\times\text{flatten}\left(\;\mathbb{Q}^{\star}(S_{T})\;\right),
\end{equation}
where the probabilities in the vector $\phi$ are learnable and determine the distributions non-parametrically, $D_{t}(T)$ is a vector of discount factors. We flatten the matrix $\mathbb{Q}^{\star}(S_{T})_{s\times q}$ to match the shape for the multiplication. The matrix $\mathbb{Q}^{\star}(S_{T})_{s\times q}$ is passed to a row-wise softmax operation at each training iteration of a SKINN, to ensure the regularity conditions of probability distributions.

\paragraph{The AE-derived representation.} A distinctive feature of SKINNs is that it allows to embeds the structured knowledge that is entirely derived from data by a machine. In this case, economic models with known structures, e.g., the BSM and the HSV model, are not overly relied upon. Instead, the structures, including the latent parameters $\phi_{\text{AE}}$ and the structured-knowledge function $g_{\phi_{\text{AE}}}(\mathbf{X}^{\text{SK}})$, are learned from the target data, $\mathcal{D}_{\text{AE}}=\hat{\mathbf{y}}^{\text{SK}}$, with AEs.

The target data $\mathcal{D}_{\text{AE}}$ can be directly from the real-world observations, or from simulated data based on expert knowledge. We choose to simulate observations from a theoretical model (e.g., the BSM model), or from a basket of theoretical models (e.g., the mixture of BSM, HSV, and HSVJ models), plus some deliberate random noise. The simulated $\mathcal{D}_{\text{AE}}$ then cannot be represented by any existing model, as the DGP is unknown. One can customize one's way to form $\mathcal{D}_{\text{AE}}$, or simply use the market observations. In our implementation, there are 200 options defined on $\{(m_{i}, T_{i})\}_{i=1}^{200}$, from each cross-section we simulate.\footnote{We simulate option prices on a 2D mesh-grid with 20 moneyness from an equally-spaced array $[0, \cdots, 2.0]$, and 10 time-to-maturities from another equally-spaced array $[0,\cdots,1]$.} The prices then become the 200-dimensional inputs of an AE. We simulate hundreds thousand such inputs. The middle layer of the AE, a.k.a. the bottleneck, can be of an arbitrary dimension, but should be much lower-dimensional than the input, with the assumption that the latent parameters should have a sparse structure.

The observable features $\mathbf{X}^{\text{SK}}$ are not required by an AE. The decoder reconstructs $\hat{\mathbf{y}}^{\text{SK}}$ in the same order as the input, given some values for the unknown latent parameters $\phi_{\text{AE}}$.\footnote{One can also choose to incorporate $\mathbf{X}^{\text{SK}}$ into the bottleneck, and switch the grid-based evaluation of the decoder to the point-wise evaluation \citep[see, for example, in][]{bergeron2021variational}.} One can use the entire AE as the structured-knowledge function $g_{\phi}$, in which case the encoder receives the option price predictions by the NN component of a SKINN on the mesh grids, $\{m_{i},T_{i}\}_{i=1}^{200}$, and compresses them to $\phi_{\text{AE}}$. The decoder then returns the structured-knowledge outputs. This unrolls the learning of $\phi_{\text{AE}}$ to the forward- and backward-pass of the AE.\footnote{By unrolling the learning of $\phi_{\text{AE}}$, there is no need to initialize and jointly train any extra learnable latent structural parameters than the nuisance NN parameters in the NN component of the SKINN.} Alternatively, one can use only the decoder as the structured-knowledge function, by treating $\phi_{\text{AE}}$ as the learnable latent structural parameters. This therefore provides a NPSKR function $g_{\phi}$ learned by a machine:
\begin{align}
    g^{\text{AE}}(\mathbf{X}^{\text{SK}}) &= (f^{\text{Enc}}_{\theta}\circ f^{\text{Dec}}_{\theta})(\hat{\mathbf{y}}^{\text{SK}}),\;\;\text{or, alternatively,}\\ 
    g^{\text{AE}}(\mathbf{X}^{\text{SK}};\phi\in\{\phi_{\text{AE}}^{\star,[1]}, \phi_{\text{AE}}^{\star,[2]},\cdots,\phi_{\text{AE}}^{\star,[\phi_{d}]}\}) &= f^{\text{Dec}}_{\theta}(\mathbf{X}^{\text{SK}},\;\phi_{\text{AE}}^{\star,[1]}, \phi_{\text{AE}}^{\star,[2]},\cdots,\phi_{\text{AE}}^{\star,[\phi_{d}]}\;).
\end{align}

\subsection{Benchmark Neural Network Models}
While SKINNs allow to embed structured knowledge through a spectrum of representations, the existing hybrid models in the literature are less flexible. They only manage to incorporate static, model-free constraints, and are prohibitive to sophisticated, higher-dimensional structures. In this section, we revise some popular hybrid models. These models serve as the benchmark for the option pricing SKINNs in our later empirical findings section.

\subsubsection{Neural Networks with Model-free Constraints}
The model-free shape prior, such as the monotonicity and convexity of option prices surfaces with respect to strike prices (or, equivalently, moneyness) and maturities, can be enforced to a NN using derivative regularizations \citep[see, for example,][]{dugas2000incorporating,dugas2009incorporating,ackerer2020deep,chataigner2020deep}. Specifically, for European call options, the following penalizations are used: (i) strike price monotonicity, i.e., ${\partial_{K} f_{\theta}}(\mathbf{X})\leq 0 \rightarrow \lVert max({\partial_{K} f_{\theta}(\mathbf{X})}, 0)\rVert_{2}$; (ii) strike price convexity, i.e., ${\partial^{2}_{K} f_{\theta}}(\mathbf{X})\leq 0 \rightarrow \lVert max({\partial^{2}_{K} f_{\theta}}(\mathbf{X}), 0)\rVert_{2}$; (iii) maturity monotonicity, i.e., ${\partial_{\tau} f_{\theta}}(\mathbf{X})\leq 0 \rightarrow \lVert max({\partial_{\tau} f_{\theta}}(\mathbf{X}), 0)\rVert_{2}$,
which are summed together as a shape-derivative regularization term:
\begin{equation}\label{der_loss}
\mathcal{L}_{\text{ShapeDer}}\left(\theta;\mathbf{X}\right):= \lVert max({\partial_{K} f_{\theta}}(\mathbf{X}), 0)\rVert_{2} + \lVert min({\partial^{2}_{K}f_{\theta}}(\mathbf{X}), 0)\rVert_{2} + \lVert min({\partial_{\tau} f_{\theta}}(\mathbf{X}),0)\rVert_{2}.
\end{equation}

We refer to the NNs trained with the regularization (\ref{der_loss}) as the NN+Shape. The partial differential terms are effectively computed using automatic differentiation, similar to PINNs. Hence, NN+Shape inherently suffers from pathological gradient dynamics. Such gradient pathologies can be addressed by casting the partial differentiations to a quadratic programming operation \citep[see,][]{cohen2020detecting}, but according to our unreported tests, we find no significant evidence that this can remedy the deficiencies of model-free constraints.

\subsubsection{Neural Networks with PDE Constraints}
PINNs incorporate domain knowledge into NNs through a differentiable PDE regularization term, which is a non-linear partial differential operation of the NNs:
\begin{align}
    r(\theta; t, \mathbf{X}, \phi_{\text{PDE}}) &= \partial_{t}f_{\theta}(t, \mathbf{X}) + \mathfrak{N}_{\mathbf{X}}[f_{\theta}(t, \mathbf{X}); \phi_{\text{PDE}}],\\
    \mathcal{L}_{\text{PDE}}(\theta; t, \mathbf{X}, \phi_{\text{PDE}}) &= \lVert r(\theta; t, \mathbf{X}, \phi_{\text{PDE}}) \rVert_{2},
\end{align}
where the first term in $r(\theta; t, \mathbf{X}, \phi_{\text{PDE}})$ is the NN derivative against the time coordinate $t$, and the second term is the NN derivatives against the space coordinate $\mathbf{X}$. The non-linear partial differential operation is specified and parameterized by a PDE provided by theories, with $\phi_{\text{PDE}}$ being the PDE parameters. A PINN typically aims to find the data-driven solution to the PDE by minimizing the PDE loss $\mathcal{L}_{\text{PDE}}$, given that the PDE parameters $\phi_{\text{PDE}}$ are known \citep[see,][]{raissi2017physics}. This approach is also known as the forward-problem PINN, as $\phi_{\text{PDE}}$ is determined. In the financial context, it can be hard to obtain an accurate estimation of $\phi_{\text{PDE}}$ ex ante, and one has to learn it from the market data by using the inverse-problem PINN \citep[see,][]{raissi_physics_2017}, which poses a greater challenge due to the pathological PINN architecture.

Option pricing models can be formulated as PDEs. Consider the BSM that admits a PDE formulation:
\begin{equation}
    -\partial_{T} V + rS\partial_{S} V + \frac{1}{2}\sigma^{2}S^{2}\partial_{SS} V - rV = 0,
\end{equation}
A PINN $f^{\text{PINN}}_{\theta}(t,\mathbf{X})$ is trained to approximate the PDE solution $V(t, \mathbf{X})$. The first and foremost challenge is the vanishing gradient pathologies due to the deep chain-rule computation when back-propagating through $\mathcal{L}_{\text{PDE}}$, which leads to problematic gradient dynamics as documented in \cite{wang2021understanding,wang_when_2022}. We illustrate these pathologies in Figure (\ref{PINN_grads_dist}), with some simulated data. On the other hand, the gradient dynamics of a SKINN with everything else equivalent do not suffer from this problem, as shown in Figure (\ref{FINN_grads_dist}).
\begin{figure}[htbp]
    \begin{subfigure}[htbp]{\textwidth}
        \centering
        \includegraphics[width=\textwidth]{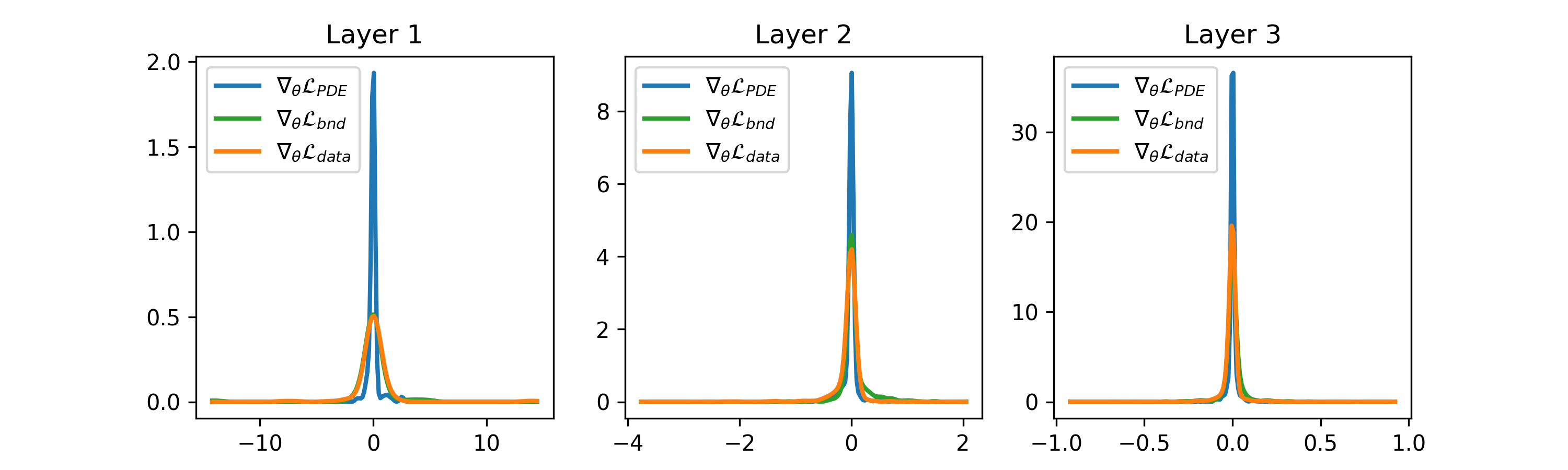}
        \caption{Loss gradients with respect to the NN weights.}
    \end{subfigure}
    
    \begin{subfigure}[htbp]{\textwidth}
        \centering
        \includegraphics[width=\textwidth]{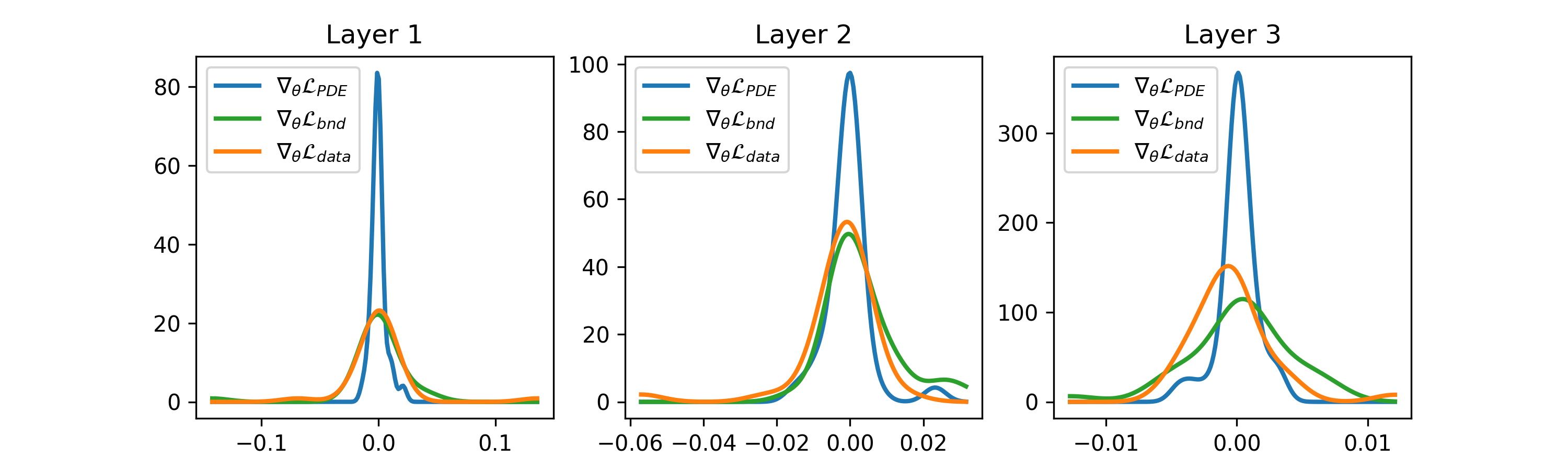}
        \caption{Loss gradients with respect to the NN biases.}
    \end{subfigure}
    \caption{The probability density of the loss gradients with respect to the NN weights and biases, for each loss component in a PINN (with inverse problem approach), including a boundary loss. The PINN is trained, for 5,000 iterations, using a cross-section of simulated European call option prices with noise, according to the BSM model.}
    \label{PINN_grads_dist}
\end{figure}

\begin{figure}[htbp]
    \begin{subfigure}[htbp]{\textwidth}
        \centering
        \includegraphics[width=\textwidth]{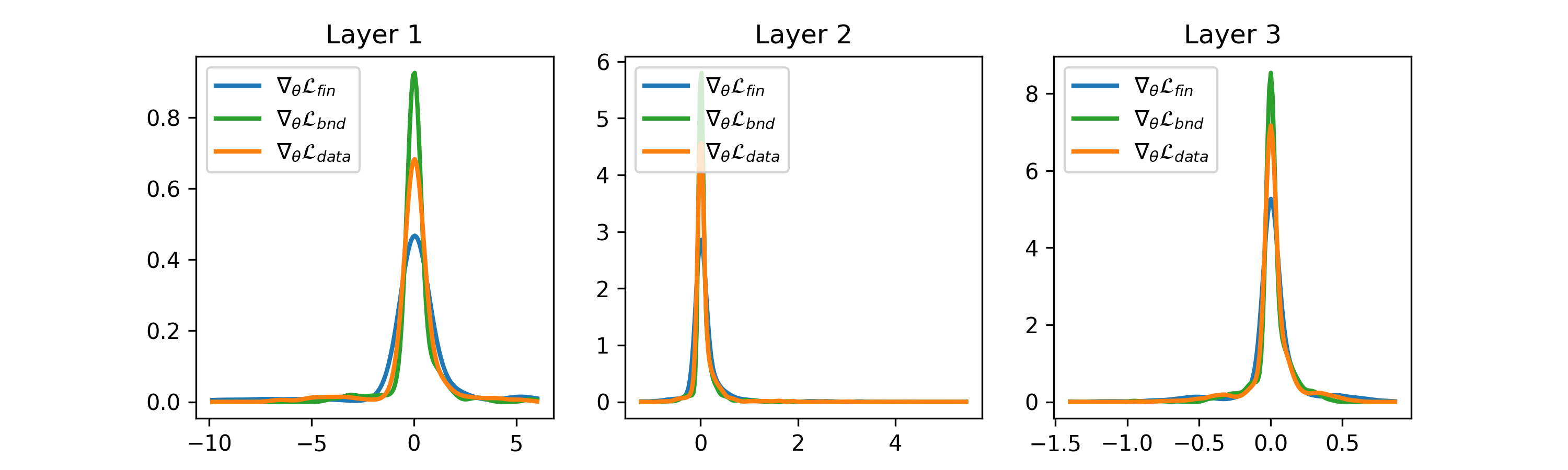}
        \caption{Gradients with respect to the NN weights.} 
    \end{subfigure}
    
    \begin{subfigure}[htbp]{\textwidth}
        \centering
    \includegraphics[width=\textwidth]{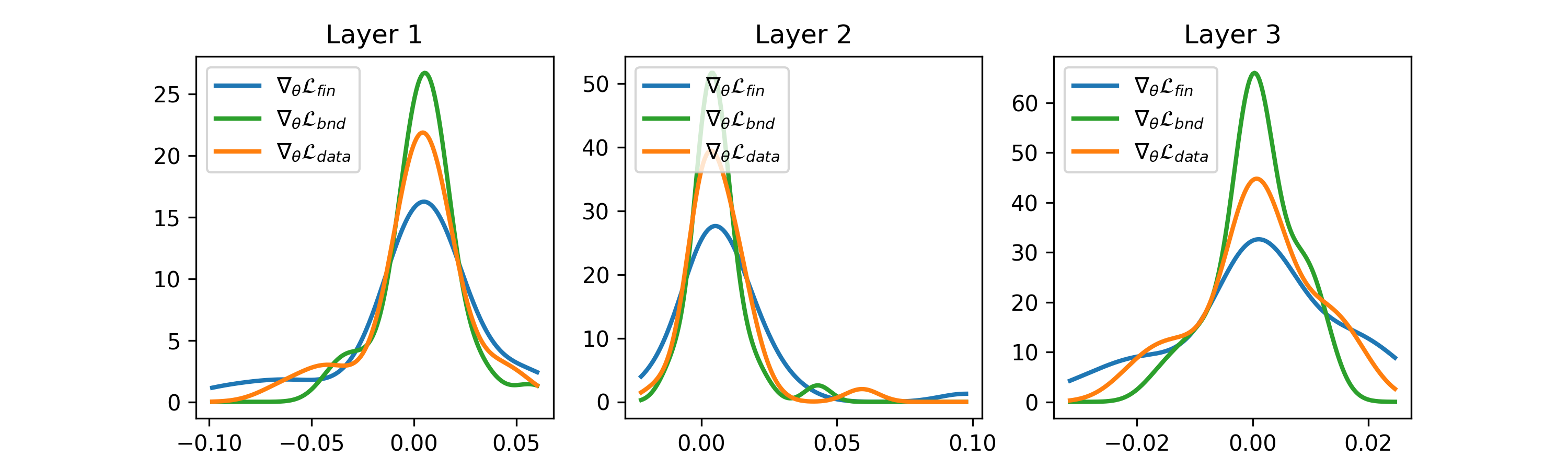}
        \caption{Gradients with respect to the NN biases.}  
    \end{subfigure}
    \caption{The probability density of the loss gradients with respect to the NN weights and biases, for each loss component in a SKINN, including a boundary loss. The SKINN is trained, for 5,000 iterations, using a cross-section of simulated European call option prices with noise, with the PSKR of the BSM model.}
    \label{FINN_grads_dist}
\end{figure}

One other challenge to PINNs for financial tasks is that PINNs can not handle the PDEs where there are many unobservable differentiation variables. A practical example is the HSV PDE, which is formulated as:
\begin{align}
    -&\partial_{T}V+\frac{1}{2}vS^{2}\partial_{SS}V+\rho\sigma vS\partial_{Sv}V \nonumber\\
    &+\frac{1}{2}\sigma^{2}v\partial_{vv}V+rS\partial_{S}V + \left[\kappa\left(\theta-v\left(t\right)\right)-\lambda\left(S,v,t\right)\right]\partial_{v}V -rV = 0.
\end{align}
The instantaneous variance $v$ is not observable, and it is impossible to evaluate the partial differential terms $\partial_{Sv}f^{\text{PINN}}_{\theta}(t, \mathbf{X}), \partial_{vv}f^{\text{PINN}}_{\theta}(t, \mathbf{X}), \partial_{v}f^{\text{PINN}}_{\theta}(t, \mathbf{X})$ from the HSV PDE with market data. There are many more such PDEs in finance that are prohibitive for PINNs to learn, especially when the system is multivariate with high-dimensional latent structural parameters.

\subsubsection{Transfer Learning Neural Networks}
Transfer learning models attempt to incorporate the knowledge learned from a source domain into a related task with a different dataset, also called a target domain. The knowledge from the source domain is expected to improve the generalizability of the model in the task domain. \cite{chen2023teaching} apply this method to the option pricing problem. Given a structural option pricing model $g(\mathbf{X}; \phi)$, transfer learning pre-trains a deep surrogate model that can decently approximate $g(\mathbf{X}; \phi)$ in a first step. This step is the same as the construction of DSNN-based SPSKRs for SKINNs. In a second step, the knowledge of a DSNN is transferred to learn from the target domain dataset by slightly updating its parameters.

Such a transfer learning method still faces some limitations. Since the NNs need both the observable features $\mathbf{X}^{\text{SK}}$ and the latent parameters $\phi$ to calculate an output, a separate estimation of $\phi$ for the target domain is needed before the target domain learning. This makes it infeasible when the surrogate model requires high-dimensional latent parameters. In addition, there is no clear fine-tuning routine for how many NN parameters to update. Updating the parameters from all layers may result in catastrophic forgetting, but only updating the parameters from certain layers before the final output layer sacrifices expressivity.

\begin{table}[htbp]
    \centering
    \caption{The full list of the NN models we consider in this paper. Only the PINNs, the TLNNs, and the SKINNs introduce the structural model as prior knowledge from option pricing to the NN. Others introduce either the boundary constraints or the model-free shape constraints to the NNs.}
    \label{nn_model_list}
    \resizebox{\textwidth}{!}{
    \begin{tabular}{@{}cccc@{}}
    \toprule
    \textbf{Model} & \textbf{Structured Knowledge}\tablefootnote{``No'' in the column structured knowledge means that there is no prior knowledge; ``Bnd'' means that the boundary conditions are incorporated; ``Shape'' means that the model-free shape constraints are incorporated; and ``AE-BSM'' means the machine-learned AE-based NPSKR structured-knowledge function from the BSM model with noise; ``AE-MIX'' means such a structured-knowledge function from the mixture of models with noise.} & \makecell{\textbf{$\text{dim}(\phi)$} \\ \textbf{(learnable/total)}} & \textbf{Activation function} \\
    \midrule
    \multicolumn{4}{c}{\textbf{Benchmark NNs}} \\
    \midrule
                
                NN & No & 0/0 & ReLU \\
                NN+Bnd & Bnd & 0/0 & ReLU \\
                NN+Shape
                 & Shape & 0/0 & SiLU \\
                PINN+BSM & BSM & 1/1 & SiLU\\
                TLNN+HSV & HSV & 0/5 & ReLU \\

    \midrule
    \multicolumn{4}{c}{\textbf{(PSKR) SKINNs}}\\
    \midrule

                SKINN+BSM & BSM & 1/1 & ReLU \\
                SKINN+HSV & HSV & 5/5 & ReLU \\
                SKINN+ABSM & ABSM & 6/6 & ReLU \\
                SKINN+HSVJ & HSVJ & 9/9 & ReLU \\
                SKINN+SABR & SABR & 722/722 & ReLU \\

    \midrule
    \multicolumn{4}{c}{\textbf{(SPSKR) SKINNs}} \\
    \midrule
                SKINN+DSNN-HSV & DSNN-HSV & 5/5 & ReLU \\
                SKINN+DSNN-NASV & DSNN-NASV & 6/6 & ReLU \\

    \midrule
    \multicolumn{4}{c}{\textbf{(NPSKR) SKINNs}} \\
    \midrule
                SKINN+AE-BSM & AE-BSM & 2/2 & ReLU \\
                SKINN+AE-MIX & AE-MIX & 50/50 & ReLU \\
                SKINN+MOPA & MOPA & 2,000/2,000 & ReLU \\

    \bottomrule
    \end{tabular}
        }
\end{table}

Table (\ref{nn_model_list}) lists all the NN models for option pricing that we consider in this paper. The upper panel includes all the benchmark NN models that we compare our SKINNs against. The lower panel includes 10 variants of SKINN with different specifications of structured-knowledge representations, for which we have elaborated their constructions in Section (\ref{sec:option_pricing_structured_knowledge_representations}).

\section{Empirical Findings}\label{sec_empirical}
We implement our proposed SKINNs framework to learn from the S\&P 500 index options with structured-knowledge representations, over a long time period, from 1996 to 2022, using daily transaction quotes. This allows SKINNs to be tested through several major recessions, which include the dot-com bubble, the 2007-2009 global financial crisis, and the COVID-19 crisis. We document statistically significant improvements of SKINNs in both the pricing and hedging capabilities, compared with benchmark models, especially in the out-of-sample data.

\subsection{Data}
We use the daily transaction quotes of the S\&P 500 index options from OptionMetrics. The sample period, to be specific, starts from 04 January 1996 and ends on 31 December 2022. Instead of using the standardized option dataset from OptionMetrics, where options are recorded on a regular meshgrid of strike prices and maturities, and are smoothed by a proprietary kernel smoothing algorithm, we use the raw option dataset. Though there are missing values and erroneous prices in the raw option dataset, which challenge data-driven models, it reflects the true market and also stress-tests the models.

We consider only European call options in this paper for convenience, but all models can be easily adapted to European put options. Additionally, for each transaction day, we include only the options that expire after 7 calendar days but before 365 calendar days, as these options are more liquid and are expected to contain more useful information than others.


\subsection{Configurations, Training and Testing Schedule}\label{sec:network_configs_training_schedule}
To ensure a fair comparison across different models, we apply the same configuration for the networks. All networks have a fully-connected, feed-forward architecture, with 3 hidden layers and 32 neurons in each hidden layer, except for those pre-trained specific DSNNs that serve as structured-knowledge functions. To ensure that all networks start with the same parameters before training, we initialize them with a fixed random seed. For the models that do not require partial differentiation operations, we use ReLU as the activation function; and for others, e.g., NN+Shape, PINN, we use SiLU, i.e., the sigmoid linear unit, $h(x)=\frac{x}{1+\exp(-x)}$, as the activation function, which, according to \cite{chen2023teaching}, helps to alleviate the pathological gradient problem.

\begin{figure}[htbp]
\caption{The forward rolling model training schedule. The first training period covers the daily option panels from 04 Jan 1996 to 29 Mar 1996. The corresponding shorter test horizon (Test 1) covers options from 01 Apr 1996 to 30 Apr 1996, and the longer test horizon (Test 2) covers options from 01 May 1996 to 31 May 1996. For each of the periods from period 2 to period 317, all the Train, Test 1, and Test 2 are rolled one month forward based on the preceding period.}
\centering
\begin{tikzpicture}[x=0.6cm, y=0.8cm] 
\tikzstyle{every node}=[font=\normalsize]

\draw [fill={rgb,255:red,240; green,161; blue,71}, line width=1.0pt, rounded corners=4] (4,6) rectangle (10,5.2) node[pos=.5] {Train};
\draw [fill={rgb,255:red,65; green,153; blue,236}, line width=1.0pt, rounded corners=4, dashed] (10,6) rectangle (12,5.2) node[pos=.5] {Test 1};
\draw [fill={rgb,255:red,65; green,153; blue,236}, line width=1.0pt, rounded corners=4, dashed] (12,6) rectangle (14,5.2) node[pos=.5] {Test 2};

\draw [fill={rgb,255:red,240; green,161; blue,71}, line width=1.0pt, rounded corners=4] (6,5) rectangle (12,4.2) node[pos=.5] {Train};
\draw [fill={rgb,255:red,65; green,153; blue,236}, line width=1.0pt, rounded corners=4, dashed] (12,5) rectangle (14,4.2) node[pos=.5] {Test 1};
\draw [fill={rgb,255:red,65; green,153; blue,236}, line width=1.0pt, rounded corners=4, dashed] (14,5) rectangle (16,4.2) node[pos=.5] {Test 2};

\draw [fill={rgb,255:red,240; green,161; blue,71}, line width=1.0pt, rounded corners=4] (8,4) rectangle (14,3.2) node[pos=.5] {Train};
\draw [fill={rgb,255:red,65; green,153; blue,236}, line width=1.0pt, rounded corners=4, dashed] (14,4) rectangle (16,3.2) node[pos=.5] {Test 1};
\draw [fill={rgb,255:red,65; green,153; blue,236}, line width=1.0pt, rounded corners=4, dashed] (16,4) rectangle (18,3.2) node[pos=.5] {Test 2};

\node at (13, 2.7) {\Large\textbf{\vdots}};
\node at (1, 2.7) {\Large\textbf{\vdots}};

\draw [fill={rgb,255:red,240; green,161; blue,71}, line width=1.0pt, rounded corners=4] (17,2.2) rectangle (23,1.4) node[pos=.5] {Train};
\draw [fill={rgb,255:red,65; green,153; blue,236}, line width=1.0pt, rounded corners=4, dashed] (23,2.2) rectangle (25,1.4) node[pos=.5] {Test 1};
\draw [fill={rgb,255:red,65; green,153; blue,236}, line width=1.0pt, rounded corners=4, dashed] (25,2.2) rectangle (27,1.4) node[pos=.5] {Test 2};

\draw [line width=0.8pt, ->, >=stealth] (0,0.5) -- (28,0.5);
\node [font=\small\bfseries] at (14,0) {Time};

\node [anchor=west, font=\small\bfseries] at (0,5.6) {Period 1};
\node [anchor=west, font=\small\bfseries] at (0,4.6) {Period 2};
\node [anchor=west, font=\small\bfseries] at (0,3.6) {Period 3};
\node [anchor=west, font=\small\bfseries] at (0,1.8) {Period 317};

\end{tikzpicture}
\label{fig:training_schedule}
\end{figure}
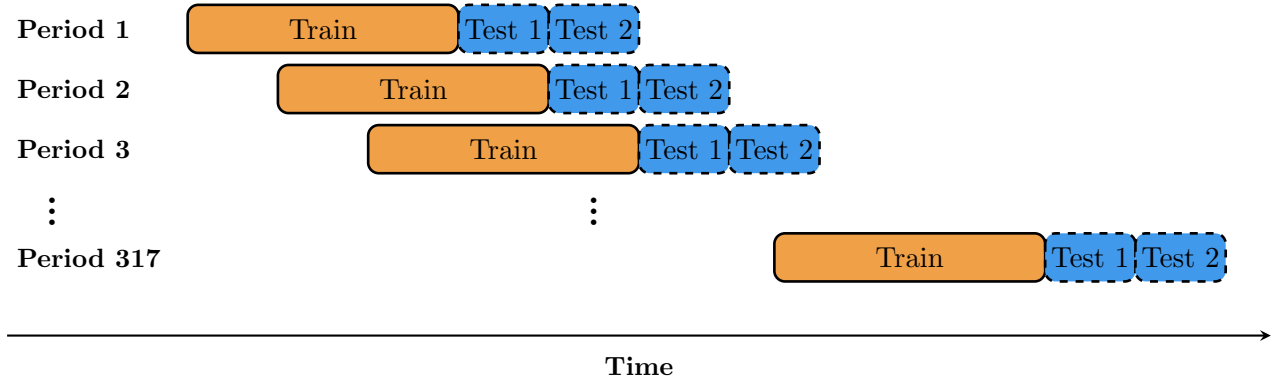

Starting from 04 January 1996, we train models using a panel of three-month options data, on a forward rolling basis with a one-month rolling window. That is, the starting date of a training period is always one month later than the starting date of its preceding training period. Once the models are trained using a three-month panel of European call options of the S\&P 500 index, we test their out-of-sample pricing and hedging performance using the options in the following two consecutive months. We refer to the first testing months as the shorter prediction horizons, and the second testing months as the longer prediction horizons. Options within the longer prediction horizons are more difficult to price, as the DGP in these periods can be changed drastically relative to the corresponding training periods. In total, there are 317 model training and testing periods (317 iterations of shorter and longer prediction horizons) from our option dataset. Figure (\ref{fig:training_schedule}) illustrates the training and testing schedule with the forward rolling basis.

\subsection{Out-of-Sample Model Performance}\label{sec:oos_performance_dm_test}
We evaluate the out-of-sample option pricing and hedging capability using the \cite{diebold2002comparing} test. 
This test is a statistical test to evaluate the out-of-sample predictive accuracy between two forecasting models, and is adopted by \cite{gu2020empirical} to evaluate the machine learning based asset pricing models. We denote $e^{1}_{j}$ and $e^{2}_{j}$ the out-of-sample pricing or hedging error of two candidate models, Model 1 and Model 2, respectively, for the period $j$, $1\leq j\leq 317$. The Diebold-Mariano test statistically compares the out-of-sample model error differences $d_{j}=\{e^{1}_{j}-e^{2}_{j}\}_{j=1}^{317}$ with a zero series. Since we aim to test whether Model 1 provides a statistically smaller out-of-sample error than Model 2, we use the one-sided test with the null hypothesis that $d_{j}$ is distributed with an expectation greater than zero. For robustness check, we also perform the non-parametric \cite{wilcoxon1945individual} signed-rank test, for which we report the results in Appendix (\ref{sec_oos_wilcoxon}). 

\subsubsection{Option Pricing Performance}
While there is a wide range of error metrics to measure the option pricing accuracy of a model, e.g., mean squared error, median absolute error, mean absolute percentage error, median absolute percentage error \citep[see, e.g.,][]{ruf2019neural}, we use the root mean squared error (RMSE) throughout this paper. The choice of different error metrics makes an insignificant difference.

\paragraph{Shorter prediction horizons.} We first consider the option pricing performance in the shorter prediction horizons. Table (\ref{tab:oos_pricing_dm_shorter}) reports Diebold-Mariano test statistics for pairwise option pricing error comparisons of a column model versus a row model. A negative Diebold-Mariano test statistic indicates that the column model outperforms (has a lower RMSE than) the row model, and a positive test statistic otherwise. 


Without any prior domain knowledge from the option pricing theory, NN in fact provides decent out-of-sample pricing performance, in the shorter prediction horizons. From Table (\ref{tab:oos_pricing_dm_shorter}), it outperforms NN+Shape and PINN+BSM, two models that require non-linear differentiation operations and hence suffer from the gradient pathologies, and also statistically outperforms the classical structural option pricing models, including the BSM, the ABSM, and the HSV model.

It is not surprising that a completely data-driven NN statistically outperforms completely theory-driven structural models. This is actually the advantage of NNs, as the over-parameterization and the high non-linear nature allow for capturing the predictive option data patterns that are not easily incorporated by rigid structural models, especially within the shorter prediction horizons in which data matters more. Adding the boundary conditions further improves the out-of-sample pricing performance of NN.

However, both NN and NN+Bnd are outperformed by SKINN+HSV, SKINN+HSVJ, as well as SKINN+MOPA, and SKINN+AE-BSM. This implies that SKINNs with more sophisticated $g_{\phi}$ tend to benefit the out-of-sample option pricing performance, when there is a slight shift between the training and the evaluation dataset. 
One explanation is that, in this case, the patterns uncovered by data-driven models already overlap the patterns that are prescribed in the relatively simple structured-knowledge representations $g_{\phi}$, and hence only more sophisticated $g_{\phi}$ can provide marginal improvements.



\paragraph{Longer prediction horizons.} We then consider the option pricing performance in the longer prediction horizons. In this case, the performance of the data-driven models decay significantly, as it becomes difficult for them to learn the generalizable predictive information, as the patterns in the training dataset can be shifted significantly in the evaluation dataset. Table (\ref{tab:oos_pricing_dm_longer}) reports pairwise test statistics from Diebold-Mariano tests. Again, NN+Shape and PINN+BSM are outperformed by other models due to their gradient pathologies. In the longer horizons, NN fails to outperform structural models. Particularly, it statistically underperforms the BSM and the ABSM model. This confirms that the option price patterns uncovered by data-driven models become less effective for the out-of-sample options in the longer prediction horizons. Adding the boundary conditions still helps to improve the performance of NN, but it can not outperform structural models at this time. In this longer prediction horizon case, all variants of SKINNs statistically outperform NN and NN+Bnd, at the 5\% significance level at least. This indicates that, by embedding the structured knowledge borrowed from theories, SKINNs offer significantly improved generalizability, especially when there is a considerable shift in the data patterns.

\begin{table}[p]

\centering
\caption{This table reports pairwise Diebold-Mariano test statistics comparing the out-of-sample option pricing performance, measured in terms of RMSE, of different models. We compare 8 variants of SKINNs with different specifications of structured-knowledge representations, with 4 classical structural models, and 4 benchmark neural networks. The out-of-sample option pricing performance is evaluated in the 317 shorter prediction horizons. We square all the pricing errors to penalize large errors.}
\label{tab:oos_pricing_dm_shorter}

    \resizebox{\textwidth}{!}{
    \begin{tabular}{@{}ccccccccccccc@{}}
    \toprule
    \textbf{Model}&\multicolumn{4}{c}{\textbf{Structural models}} & \multicolumn{4}{c}{\textbf{Benchmark NNs}} \\
        \cmidrule(lr){2-5}\cmidrule(lr){6-9}
    \textbf{Panel (A)} & \multirow{2}{*}{BSM} & \multirow{2}{*}{ABSM} & \multirow{2}{*}{HSV} & \multirow{2}{*}{HSVJ} & \multirow{2}{*}{NN} & NN & NN & PINN\\ 
     &  &  &  &  &  & +Bnd & +Shape & +BSM\\ 
    \midrule
    BSM & -- & 1.99** & -2.16** & -9.24*** & -1.72** & -2.42*** & 0.16 & 0.21 \\ 
    ABSM & -- & -- & -2.89*** & -8.91*** & -1.88** & -2.58*** & -0.07 & -0.02 \\ 
    HSV & -- & -- & -- & -12.95*** & -1.49* & -2.22** & 0.59 & 0.64 \\ 
    HSVJ & -- & -- & -- & -- & -0.61 & -1.35* & 1.81** & 1.85** \\ 
    NN & -- & -- & -- & -- & -- & -2.41*** & 3.01*** & 2.79*** \\ 
    NN+Bnd & -- & -- & -- & -- & -- & -- & 4.49*** & 4.26*** \\ 
    NN+Shape & -- & -- & -- & -- & -- & -- & -- & 0.13 \\ 
    PINN+BSM & -- & -- & -- & -- & -- & -- & -- & -- \\

    \midrule
    \textbf{Model} & \multicolumn{4}{c}{\textbf{PSKR}} & \multicolumn{2}{c}{\textbf{SPSKR}} & \multicolumn{2}{c}{\textbf{NPSKR}} \\
    \cmidrule(lr){2-5}\cmidrule(lr){6-7}\cmidrule(lr){8-9}
    \textbf{Panel (B)} & SKINN & SKINN & SKINN & SKINN & SKINN & SKINN & SKINN & SKINN \\
                       & +BSM & +ABSM & +HSV & +HSVJ & +DSNN-HSV & +DSNN-NASV & +MOPA & +AE-BSM \\
    \midrule
    BSM & -4.45*** & -4.79*** & -4.58*** & -5.27*** & -3.47*** & -4.20*** & -4.18*** & -2.98***\\ 
    ABSM & -4.67*** & -5.03*** & -4.80*** & -5.51*** & -3.68*** & -4.43*** & -4.38*** & -3.17***\\ 
    HSV & -4.40*** & -4.76*** & -4.64*** & -5.44*** & -3.40*** & -4.20*** & -4.18*** & -2.84***\\ 
    HSVJ & -2.80*** & -2.94*** & -3.09*** & -3.66*** & -1.98** & -2.54*** & -2.97*** & -1.73**\\
    NN & -1.27 & -1.16 & -1.65** & -1.68** & -0.88 & -0.99 & -2.42*** & -1.79**\\ 
    NN+Bnd & -0.24 & -0.11 & -0.57 & -0.63 & 0.15 & 0.01 & -1.16 & -0.06 \\ 
    NN+Shape & -3.92*** & -4.10*** & -4.53*** & -4.62*** & -3.30*** & -3.66*** & -4.70*** & -4.34***\\ 
    PINN+BSM & -3.76*** & -3.88*** & -4.26*** & -4.37*** & -3.16*** & -3.48*** & -4.46*** & -4.03***\\ 
    
    \bottomrule
    \end{tabular}
    }

\end{table}

\begin{table}[p]
\centering
\caption{This table reports pairwise Diebold-Mariano test statistics comparing the out-of-sample option pricing performance, measured in terms of RMSE, of different models. We compare 8 variants of SKINNs with different specifications of structured-knowledge representations, with 4 classical structural models, and 4 benchmark neural networks. The out-of-sample option pricing performance is evaluated in the 317 longer prediction horizons. We square all the pricing errors to penalize large errors.}
\label{tab:oos_pricing_dm_longer}

    \resizebox{\textwidth}{!}{
    \begin{tabular}{@{}ccccccccccccc@{}}
    \toprule
    \textbf{Model}&\multicolumn{4}{c}{\textbf{Structural models}} & \multicolumn{4}{c}{\textbf{Benchmark NNs}} \\
        \cmidrule(lr){2-5}\cmidrule(lr){6-9}
    \textbf{Panel (A)} & \multirow{2}{*}{BSM} & \multirow{2}{*}{ABSM} & \multirow{2}{*}{HSV} & \multirow{2}{*}{HSVJ} & \multirow{2}{*}{NN} & NN & NN & PINN\\ 
     &  &  &  &  &  & +Bnd & +Shape & +BSM\\ 
    \midrule
    BSM & -- & 2.49*** & 0.62 & -5.23*** & 2.67*** & 1.36* & 2.51*** & 2.54*** \\ 
    ABSM & -- & -- & -0.57 & -5.83*** & 2.56*** & 1.26 & 2.40*** & 2.42*** \\ 
    HSV & -- & -- & -- & -12.95*** & 2.63*** & 1.33* & 2.45*** & 2.47*** \\ 
    HSVJ & -- & -- & -- & -- & 3.09*** & 1.74** & 2.94*** & 2.97*** \\ 
    NN & -- & -- & -- & -- & -- & -2.51*** & -0.44 & -0.38 \\ 
    NN+Bnd & -- & -- & -- & -- & -- & -- & 1.30* & 1.14 \\ 
    NN+Shape & -- & -- & -- & -- & -- & -- & -- & 0.01 \\ 
    PINN+BSM & -- & -- & -- & -- & -- & -- & -- & -- \\

    \midrule
    \textbf{Model} & \multicolumn{4}{c}{\textbf{PSKR}} & \multicolumn{2}{c}{\textbf{SPSKR}} & \multicolumn{2}{c}{\textbf{NPSKR}} \\
    \cmidrule(lr){2-5}\cmidrule(lr){6-7}\cmidrule(lr){8-9}
    \textbf{Panel (B)} & SKINN & SKINN & SKINN & SKINN & SKINN & SKINN & SKINN & SKINN \\
                       & +BSM & +ABSM & +HSV & +HSVJ & +DSNN-HSV & +DSNN-NASV & +MOPA & +AE-BSM \\
    \midrule
    BSM          & -3.61*** & -3.61*** & -3.06*** & -3.56*** & -2.09** & -2.78*** & -2.53*** & 0.04    \\ 
    ABSM        & -3.85*** & -3.87*** & -3.30*** & -3.80*** & -2.31** & -3.01*** & -2.74*** & -0.10   \\ 
    HSV         & -3.89*** & -3.91*** & -3.37*** & -3.89*** & -2.33** & -3.08*** & -2.74*** & -0.03   \\ 
    HSVJ     & -2.69*** & -2.63*** & -2.23** & -2.69*** & -1.26    & -1.88** & -1.76** & 0.56    \\ 
    NN          & -3.98*** & -3.95*** & -3.83*** & -4.02*** & -3.51*** & -3.70*** & -3.83*** & -3.61***\\ 
    NN+Bnd      & -2.72*** & -2.69*** & -2.63*** & -2.77*** & -2.28** & -2.44*** & -2.61*** & -2.54***\\ 
    NN+Shape       & -3.55*** & -3.57*** & -3.44*** & -3.66*** & -3.06*** & -3.27*** & -3.43*** & -3.05***\\ 
    PINN+BSM & -3.50*** & -3.53*** & -3.42*** & -3.62*** & -3.02*** & -3.23*** & -3.41*** & -2.82***\\ 
    
    \bottomrule
    \end{tabular}
    }

\end{table}

\subsubsection{Option Hedging Performance}
Another important purpose of option pricing models is hedging. For structural models, hedging is straightforward, as the option Greeks, e.g., Delta, can be effectively calculated using semi-closed or closed-form solutions. We use all our models to Delta-hedge the short positions in the S\&P 500 call options, as an evaluation of their hedging performance. We rely on automatic differentiation to derive the Delta ratios for the NN-based models. Specifically, we use the relation according to the homogeneous-of-degree-one property:
\begin{equation}
    \Delta_{\text{NN}}:=\frac{\partial f_{\theta}(\mathbf{X})}{\partial m}\frac{K}{S^{2}}.
\end{equation}

To evaluate the out-of-sample Delta-hedging performance, we initialize a Delta-hedged portfolio on the trading date $t_{j}$ within the evaluation horizon, by entering the following positions:
\begin{align}
    \Pi^{(i)}_{\text{Stock}}(t_{j}) &= S_{t_{j}}\Delta^{(i)}(t_{j}), \\
    \Pi^{(i)}_{\text{Call}}(t_{j}) &= -C^{(i)}_{t_{j}}, \\
    \Pi^{(i)}_{\text{Bond}}(t_{j}) &= -\left(\Pi^{(i)}_{\text{Stock}}(t_{j}) + \Pi_{\text{Call}}^{(i)}(t_{j})\right),
\end{align}
where $1\leq j\leq m$, and $m$ is the total number of trading days in the prediction horizon; $1\leq i\leq n_{j}$, and $n_{j}$ is the total number of options from day $t_{j}$; $\Delta^{(i)}(t_{j})$ is the Delta, estimated by a model, for the option $i$ at day $t_{j}$; $\Pi^{(i)}_{\text{Stock}}(t_{j})$ is the value of the underlying asset position; $\Pi^{(i)}_{\text{Bond}}(t_{j})$ is the value of the zero-coupon bond position; $\Pi^{(i)}_{\text{Call}}(t_{j})$ is the value of the short call option position that we aim to hedge. For all $t_{j}, 1\leq j\leq m$, the initial Delta-hedged portfolio perfectly hedges, as $\sum_{i=1}^{n_{j}}\Pi^{(i)}_{\text{Stock}}(t_{j}) + \Pi^{(i)}_{\text{Call}}(t_{j}) + \Pi^{(i)}_{\text{Bond}}(t_{j})=0$ by construction. We then evaluate the Delta-hedging performance on the next day of the hedged portfolio construction, $t_{j}+1=t_{j+1}, 1\leq j\leq m$. We calculate the next-day Delta-hedged portfolio value by:
\begin{align}
    \Pi^{(i)}_{\text{Stock}}(t_{j}+1) &= S_{t_{j+1}}\Delta^{(i)}(t_{j}), \\
    \Pi^{(i)}_{\text{Call}}(t_{j}+1) &= -C^{(i)}_{t_{j+1}}, \\
    \Pi^{(i)}_{\text{Bond}}(t_{j}+1) &= \Pi^{(i)}_{\text{Bond}}(t_{j})e^{r\times\frac{1}{252}}.
\end{align}
Considering all the options we hedge on $t_{j}$, the overall Delta-hedge portfolio has the value:
\begin{equation}
    \Pi(t_{j}+1) = \frac{1}{n_{j}}\sum_{i=1}^{n_{j}}\Pi^{(i)}_{\text{Stock}}(t_{j} + 1) + \Pi^{(i)}_{\text{Call}}(t_{j+1}) + \Pi^{(i)}_{\text{Bond}}(t_{j} + 1),
\end{equation}
on all the hedging performance evaluation date $t_{j}+1, 1\leq j\leq m$. For different models, we plug in the corresponding Delta hedge ratio $\Delta_{\text{Model}}(t_{j})$ to construct and assess the hedged portfolios.

Ideally, a Delta-hedged portfolio should hedge the risk of the underlying asset price fluctuation, and hence the next-day Delta-hedged portfolio is expected to have a zero value, i.e., $\Pi(t_{j}+1) = 0$. In reality, however, option pricing models have residual risk, and $\Pi(t_{j}+1)$ can deviate from zero. To compare the out-of-sample Delta-hedging performance of each model, we define the following hedging error metric:
\begin{equation}\label{hedging_error}
    \text{HE}^{\text{Model}} = \frac{1}{m}\sum_{j=1}^{m}\left\lvert\Pi(t_{j}+1)\right\rvert.
\end{equation}
We compare the out-of-sample hedging accuracy of models using Diebold-Mariano test with a series of hedging error differences $d_{j}=\{e^{1}_{j} - e^{2}_{j}; e^{1}_{j}=\text{HE}^{1}_{j}, e^{2}_{j}=\text{HE}^{2}_{j}\}_{j=1}^{317}$.



\paragraph{Shorter prediction horizons.} Table (\ref{tab:oos_hedging_dm_shorter}) reports the pairwise Diebold-Mariano test statistics for the out-of-sample hedging performance comparisons, considering the shorter prediction horizons. Interestingly, in this test, NN is generally outperformed by classical structural models significantly. At this time, NN+Bnd fails to provide any improvements by adding the boundary conditions. All SKINNs significantly outperform all the benchmark NN models in terms of hedging purpose, thanks to the embedding of the structured-knowledge representations. We observe that classical structural models are robust to the change of the evaluation objective, i.e., estimating them by minimizing the pricing errors insignificantly influences their hedging capability. However, the over-parameterization nature of NNs makes them disadvantaged when changing the evaluation objective, as the weights and biases are optimized by minimizing the empirical option pricing errors, which can be distant from the optimal weights and biases by minimizing the empirical hedging errors. Though disadvantaged, unlike the benchmark NNs, SKINNs do not underperform the structural models. For SKINN+BSM and SKINN+ABSM, their hedging performance can even statistically outperform the structural models in the shorter prediction horizons.

\begin{table}[p]
\centering
\caption{This table reports pairwise Diebold-Mariano test statistics comparing the out-of-sample option hedging performance, measured in terms of MHE, of different models. We compare 8 variants of SKINNs with different specifications of structured-knowledge representations, with 4 classical structural models, and 4 benchmark neural networks. The out-of-sample option hedging performance is evaluated in the 317 shorter prediction horizons. We square all the hedging errors to penalize large errors.}
\label{tab:oos_hedging_dm_shorter}

    \resizebox{\textwidth}{!}{
    \begin{tabular}{@{}ccccccccccccc@{}}
    \toprule
    \textbf{Model}&\multicolumn{4}{c}{\textbf{Structural models}} & \multicolumn{4}{c}{\textbf{Benchmark NNs}} \\
        \cmidrule(lr){2-5}\cmidrule(lr){6-9}
    \textbf{Panel (A)} & \multirow{2}{*}{BSM} & \multirow{2}{*}{ABSM} & \multirow{2}{*}{HSV} & \multirow{2}{*}{HSVJ} & \multirow{2}{*}{NN} & NN & NN & PINN\\ 
     &  &  &  &  &  & +Bnd & +Shape & +BSM\\ 
    \midrule
    BSM          & -- & -3.62*** & 6.63*** & 6.56*** & 2.64*** & 2.82*** & 2.99*** & 2.48*** \\ 
    ABSM        & -- & --       & 6.92*** & 6.83*** & 2.90*** & 3.19*** & 3.17*** & 2.66*** \\ 
    HSV         & -- & --       & --      & -1.31* & -1.65** & -2.59*** & -0.12   & -0.31   \\ 
    HSVJ     & -- & --       & --      & --      & -1.55* & -2.49*** & 0.01    & -0.21   \\ 
    NN          & -- & --       & --      & --      & --      & -1.15    & 2.25** & 1.11    \\ 
    NN+Bnd      & -- & --       & --      & --      & --      & --       & 2.50*** & 1.57* \\ 
    NN+Shape       & -- & --       & --      & --      & --      & --       & --      & -0.31   \\ 
    PINN+BSM & -- & --       & --      & --      & --      & --       & --      & --      \\

    \midrule
    \textbf{Model} & \multicolumn{4}{c}{\textbf{PSKR}} & \multicolumn{2}{c}{\textbf{SPSKR}} & \multicolumn{2}{c}{\textbf{NPSKR}} \\
    \cmidrule(lr){2-5}\cmidrule(lr){6-7}\cmidrule(lr){8-9}
    \textbf{Panel (B)} & SKINN & SKINN & SKINN & SKINN & SKINN & SKINN & SKINN & SKINN \\
                       & +BSM & +ABSM & +HSV & +HSVJ & +DSNN-HSV & +DSNN-NASV & +MOPA & +AE-BSM \\
    \midrule
    BSM          & -6.58*** & -6.34*** & -0.16    & -0.71    & 0.70     & 0.79     & 0.07     & 0.68     \\ 
    ABSM        & -5.59*** & -5.57*** & 1.11     & 0.63     & 1.75** & 1.80** & 1.44* & 1.48* \\ 
    HSV         & -7.01*** & -6.90*** & -7.32*** & -6.78*** & -6.26*** & -6.51*** & -6.69*** & -5.68*** \\ 
    HSVJ     & -6.88*** & -6.77*** & -7.11*** & -6.56*** & -6.06*** & -6.43*** & -6.66*** & -5.61*** \\ 
    NN          & -3.94*** & -4.02*** & -2.74*** & -2.78*** & -2.25** & -2.34*** & -2.80*** & -3.14*** \\ 
    NN+Bnd      & -4.44*** & -4.58*** & -2.87*** & -2.94*** & -2.33** & -2.47*** & -2.96*** & -3.59*** \\ 
    NN+Shape       & -3.86*** & -3.89*** & -3.11*** & -3.10*** & -2.70*** & -2.77*** & -3.03*** & -3.34*** \\ 
    PINN+BSM & -3.23*** & -3.29*** & -2.41*** & -2.41*** & -2.22** & -2.34*** & -2.53*** & -2.46*** \\ 
    
    \bottomrule
    \end{tabular}
    }

\end{table}

\begin{table}[p]
\centering
\caption{This table reports pairwise Diebold-Mariano test statistics comparing the out-of-sample option hedging performance, measured in terms of MHE, of different models. We compare 8 variants of SKINNs with different specifications of structured-knowledge representations, with 4 classical structural models, and 4 benchmark neural networks. The out-of-sample option hedging performance is evaluated in the 317 longer prediction horizons. We square all the hedging errors to penalize large errors.}
\label{tab:oos_hedging_dm_longer}

    \resizebox{\textwidth}{!}{
    \begin{tabular}{@{}ccccccccccccc@{}}
    \toprule
    \textbf{Model}&\multicolumn{4}{c}{\textbf{Structural models}} & \multicolumn{4}{c}{\textbf{Benchmark NNs}} \\
        \cmidrule(lr){2-5}\cmidrule(lr){6-9}
    \textbf{Panel (A)} & \multirow{2}{*}{BSM} & \multirow{2}{*}{ABSM} & \multirow{2}{*}{HSV} & \multirow{2}{*}{HSVJ} & \multirow{2}{*}{NN} & NN & NN & PINN\\ 
     &  &  &  &  &  & +Bnd & +Shape & +BSM\\ 
    \midrule
    BSM          & -- & -4.30*** & 7.69*** & 7.67*** & 3.44*** & 3.48*** & 3.52*** & 3.20*** \\ 
    ABSM        & -- & --       & 7.83*** & 7.81*** & 3.68*** & 3.83*** & 3.72*** & 3.43*** \\ 
    HSV         & -- & --       & --      & -1.42* & 0.25    & -1.15   & 1.11    & 0.20    \\ 
    HSVJ     & -- & --       & --      & --      & 0.39    & -0.99   & 1.21    & 0.32    \\ 
    NN          & -- & --       & --      & --      & --      & -2.12** & 1.40* & -0.05   \\ 
    NN+Bnd      & -- & --       & --      & --      & --      & --      & 2.76*** & 1.47* \\ 
    NN+Shape       & -- & --       & --      & --      & --      & --      & --      & -1.94** \\ 
    PINN+BSM & -- & --       & --      & --      & --      & --      & --      & --      \\

    \midrule
    \textbf{Model} & \multicolumn{4}{c}{\textbf{PSKR}} & \multicolumn{2}{c}{\textbf{SPSKR}} & \multicolumn{2}{c}{\textbf{NPSKR}} \\
    \cmidrule(lr){2-5}\cmidrule(lr){6-7}\cmidrule(lr){8-9}
    \textbf{Panel (B)} & SKINN & SKINN & SKINN & SKINN & SKINN & SKINN & SKINN & SKINN \\
                       & +BSM & +ABSM & +HSV & +HSVJ & +DSNN-HSV & +DSNN-NASV & +MOPA & +AE-BSM \\
    \midrule
    BSM          & -5.07*** & -6.33*** & 0.84     & 0.35     & 0.72     & 1.45* & 0.26     & 0.68     \\ 
    ABSM        & -3.66*** & -4.96*** & 1.78** & 1.55* & 1.45* & 2.01** & 1.25     & 1.35* \\ 
    HSV         & -8.32*** & -8.12*** & -8.42*** & -8.11*** & -6.37*** & -5.63*** & -6.69*** & -5.64*** \\ 
    HSVJ     & -8.26*** & -8.03*** & -8.34*** & -8.07*** & -6.18*** & -5.58*** & -6.47*** & -5.57*** \\ 
    NN          & -4.71*** & -4.59*** & -3.44*** & -3.63*** & -3.41*** & -2.89*** & -3.58*** & -3.83*** \\ 
    NN+Bnd      & -5.07*** & -5.12*** & -3.40*** & -3.71*** & -3.14*** & -2.56*** & -3.50*** & -4.07*** \\ 
    NN+Shape       & -4.41*** & -4.46*** & -3.49*** & -3.67*** & -3.31*** & -3.03*** & -3.60*** & -3.88*** \\ 
    PINN+BSM & -4.32*** & -4.29*** & -3.01*** & -3.23*** & -2.94*** & -2.45*** & -3.29*** & -3.34*** \\ 
    
    \bottomrule
    \end{tabular}
    }

\end{table}

\paragraph{Longer prediction horizons.} We continue to test whether the hedging performance improvements by SKINNs are robust to longer prediction horizons. Table (\ref{tab:oos_hedging_dm_longer}) reports the pairwise Diebold-Mariano test statistics in this case. Similar to the result for the shorter prediction horizons, all SKINNs continue to statistically outperform all the benchmark NNs. The benchmark NNs statistically underperform the classical structural models, as is expected for the longer prediction horizons, where the data-driven patterns are less effective. Again, in the longer prediction horizons, SKINN+BSM and SKINN+ABSM not only outperform the benchmark NNs but also outperform the structural models significantly.

\subsection{SKINNs Versus Transfer Learning Models}
In this section, we compare our SKINNs against the transfer learning neural networks (TLNNs) in more detail. We choose the HSV model as the domain knowledge in this comparison because of its high dimensionality and complexity in the latent parameters. \cite{chen2023teaching}, however, employs a very deep NN to incorporate the one-parameter BSM model. To facilitate a fair comparison, we match the network size of the TLNNs and SKINNs by using 6 hidden layers with 64 neurons each layer for both. By not going that deep, we do not suffer from the vanishing gradient problem as encountered in \cite{chen2023teaching}. Having said that, a sufficiently deep NN is compulsory for TLNNs, as they demand a large capacity to memorize the source domain information, which differentiates them from SKINNs, whose NN components can flexibly depend on the needs by empirical data only.

We first train a deep surrogate model to learn completely from the theoretical HSV model with simulation. We then transfer the deep surrogate model to the market data. In \cite{chen2023teaching}, they only fine-tune the deep surrogate model on the market data with a very small learning rate and very few epochs, based on the assumption that the market data has a DGP close to the theoretical model. We find this is, nevertheless, not the case empirically. Data patterns can contradict the theoretical structure of the HSV model significantly. We therefore train a series of TLNN+HSV with different numbers of frozen (or trainable) layers and epochs. Table (\ref{tab:TLNN_configs}) lists all the TLNN+HSV with different frozen layers and epochs, along with the deep surrogate HSV model (i.e., DSNN+HSV) and our SKINN+HSV, and their respective training configurations. TLNNs require latent parameters in their inputs, and one has to provide the estimations of latent parameters before learning from the target domain. This is the limitation of the transfer learning approach when applied to high-dimensional settings, as the separate calibration of the latent parameters becomes inaccurate. Although it is unfair to SKINN+HSV, we feed the SKINN-learned latent parameters as the TLNN+HSV inputs, as they present improved numerical stability.

\begin{table}[htbp]
\centering
\resizebox{0.9\textwidth}{!}{
\begin{tabular}{ccccccc}
        \toprule
                \textbf{Model} & Frozen/Trainable layers & Learning rate & Epochs & Inputs & $\phi$ \\
        \midrule
             DSNN+HSV & 6/0 & 2E-4 & 500 & $\mathbf{X}, \phi$ & Simulated \\
             TLNN+HSV1 & 0/6 & 6E-6 & 6 & $\mathbf{X}, \phi$ & SKINN-learned\\
             TLNN+HSV2 & 4/2 & 6E-6 & 6 & $\mathbf{X}, \phi$ & SKINN-learned\\
             TLNN+HSV3 & 4/2 & 6E-6 & 100 & $\mathbf{X}, \phi$ & SKINN-learned\\
             TLNN+HSV4 & 4/2 & 6E-6 & 500 & $\mathbf{X}, \phi$ & SKINN-learned\\
             SKINN+HSV & 0/3 & 1E-3 & 500 & $\mathbf{X}$ & --\\
        \bottomrule
        \end{tabular}
        }
        \captionof{table}{This table summarizes the network configurations applied to train TLNNs and SKINN+HSV. We employ a smaller network size in order to allow TLNNs (6 hidden layers, 64 neurons each) to be comparable to our SKINNs (3 hidden layers, 32 neurons each). The different number of frozen layers for TLNNs corresponds to the different compliance with the theory.}
        \label{tab:TLNN_configs}
\end{table}

Table (\ref{tab:TLNN_SKINN_DM}) reports the pairwise Deibold-Mariano test statistics, which compare the out-of-sample option pricing RMSE of the different TLNN+HSV configurations with SKINN+HSV. For the TLNN+HSV models, allowing more trainable layers, or sufficiently training the model on the market data with more epochs, generally can improve the performance. This implies that there exists a considerable gap between the source domain, i.e., a theoretical option pricing model, and the target domain, i.e., the market option prices. SKINN+HSV statistically outperforms all the TLNN+HSV models with different configurations. The out-of-sample option pricing performance of each model is evaluated using the options in the longer prediction horizons.

\begin{table}[htbp]
        \centering
        \resizebox{0.85\textwidth}{!}{
        \begin{tabular}{ccccccc}
        \toprule
                \multirow{2}{*}{\textbf{Model}} & DSNN & TLNN & TLNN & TLNN & TLNN & SKINN \\
                      & +HSV & +HSV1 & +HSV2 & +HSV3 & +HSV4 & +HSV          \\
        \midrule
                DSNN+HSV & -- & -10.48*** & -10.88*** & -10.48*** & -11.39*** & -13.87**** \\
                TLNN+HSV1 & -- & -- & 9.12*** & -10.37*** & -13.71*** & -11.30*** \\
                TLNN+HSV2 & -- & -- & -- & -10.59*** & -13.70*** & -11.60***\\
                TLNN+HSV3 & -- & -- & -- & -- & -22.46*** & -6.63***\\
                TLNN+HSV4 & -- & -- & -- & -- & -- & -2.80***\\
                SKINN+HSV & -- & -- & -- & -- & -- & -- \\
        \bottomrule
        \end{tabular}
        }
        \caption{The pair-wise Diebold-Mariano test statistics, which compare the out-of-sample option pricing errors (measured by RMSE) across all TLNNs and SKINN+HSV. The out-of-sample option pricing errors are evaluated using the options within the 317 periods of longer prediction horizons. A negative number indicates the column model outperforms the row model, and vice versa.}
        \label{tab:TLNN_SKINN_DM}
\end{table}

\begin{figure}
    \centering


        

        \includegraphics[width=\textwidth]{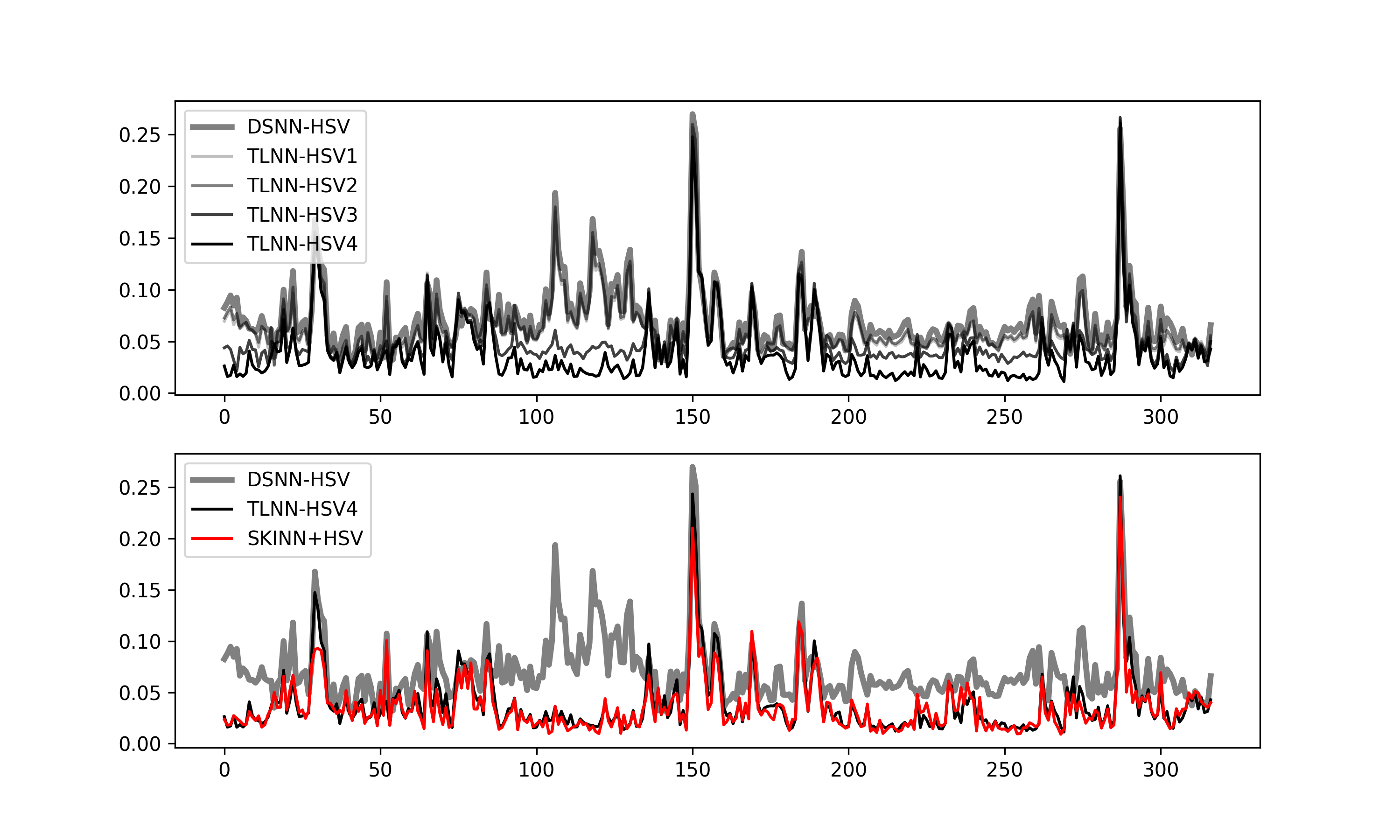}
        
        \captionof{figure}{The out-of-sample option pricing performance of TLNNs and SKINN+HSV, over the 317 periods of longer prediction horizons. Among all the TLNNs, TLNN+HSV4 performs the best.}
        \label{fig:finn_hsv_v1_vs_tlnn_hsv}
        
    
\end{figure}

\section{Economic Interpretations of SKINNs}\label{sec_interpret}
Our empirical study has shown that, by embedding the structured-knowledge from the option pricing domain, SKINNs can provide statistically superior option pricing capability, compared with both the existing NN option pricing models and the classical structural models, especially in the longer prediction horizons. Although we do not train SKINNs to hedge options, their hedging capability can outperform the structural models and is significantly better than the benchmark NNs. In this section, we dissect the reasons for the superior performance of SKINNs.

\subsection{Model Performances and Market Volatilities}
It is important to understand under what market conditions SKINNs are able to provide significant marginal improvements for option pricing, as there is no one model that is suitable for all market conditions. We treat NN as the baseline, and for other models, we calculate the difference of their out-of-sample option pricing error from that of NN:
\begin{equation}
    \Delta\text{RMSE}^{\text{Model}}(t_{j}) := \text{RMSE}^{\text{Model}}(t_{j}) - \text{RMSE}^{\text{NN}}(t_{j}).
\end{equation}
The smaller the $\Delta\text{RMSE}^{\text{Model}}(t_{j})$ is, the model outperforms more in out-of-sample compared with NN, and vice versa. We investigate the $\Delta\text{RMSE}^{\text{Model}}(t_{j})$ of each model in both the shorter and the longer prediction horizons, but our main focus is the latter, as these horizons challenge data-driven models more due to domain shift.

We use the averaged VIX index over a testing period as the proxy of the market condition, which is defined as:
\begin{equation}
    \text{AvgVIX}(t_{j}) := \frac{1}{N_{t_{j}}}\sum_{i=1}^{N_{t_{j}}} \frac{\text{VIX}_{i}}{100},
\end{equation}
where $N_{t_{j}}$ is the number of trading days in the testing period $t_{j}$, $1\leq j\leq 317$, and $\text{VIX}_{i}$ is the daily close price of the VIX index. $\text{AvgVIX}(t_{j})$ measures the volatility of the market during the period $t_{j}$. A higher $\text{AvgVIX}(t_{j})$ is usually associated with the feared sentiment, high market uncertainty, and the tendency of mispricing, which means more noise in the observed prices. We regress the out-of-sample option pricing RMSE differences from the testing periods against the averaged VIX index over these periods:
\begin{equation}\label{reg_avgvix}
    \Delta\text{RMSE}^{\text{Model}} = \beta_{0} + \beta_{1}\text{AvgVIX} + \varepsilon.
\end{equation}
Since the marginal improvements of a model against NN, $\Delta \text{RMSE}^{\text{Model}}$, is measured from the same testing periods as the market condition indicator, $\text{AvgVIX}$, the regression in Equation (\ref{reg_avgvix}) tests how the market volatility impacts the out-of-sample option pricing performance of a model, relative to an NN.

\begin{table}[htbp]

\caption{We run the regression (\ref{reg_avgvix}) for each model. A statistically significant negative $\beta_{1}$ implies that the model performs better compared with a plain vanilla neural network when the market is more volatile. The results are based on all 317 test periods of the shorter prediction horizons.}
\label{tab:dRMSE_pricing_shorter}

\resizebox{\textwidth}{!}{
\begin{tabular}{@{}cccccccccc@{}}
\toprule


\textbf{Model} & \multirow{2}{*}{BSM} & \multirow{2}{*}{ABSM} & \multirow{2}{*}{HSV} & \multirow{2}{*}{HSVJ} & NN  & -- & -- & -- \\
\textbf{Panel (A)} &    &      &    & & +Bnd\\
\midrule
$\beta_{1}$ & 0.0843\sym{***} & 0.0891\sym{***} & 0.0563\sym{***} & 0.0720\sym{***} & -0.0116\sym{*} & -- & -- & -- \\
 & (0.0141) & (0.0140) & (0.0146) & (0.0143) & (0.0069)  & -- & -- & -- \\
$\beta_{0}$ & -0.0134\sym{***} & -0.0141\sym{***} & -0.0082\sym{**} & -0.0151\sym{***} & 0.0005  & -- & -- & -- \\
 & (0.0031) & (0.0031) & (0.0032) & (0.0032) & (0.0015)  & -- & -- & -- \\

\midrule
Obs & 317 & 317 & 317 & 317 & 317  & -- & -- & -- \\
Adj. R\sym{2} & 0.0994 & 0.1111 & 0.0423 & 0.0712 & 0.0058  & -- & -- & -- \\

\midrule

\textbf{Model} & SKINN & SKINN & SKINN & SKINN & SKINN & SKINN & SKINN & SKINN  \\
\textbf{Panel (B)} & +BSM & +ABSM & +HSV & +HSVJ & +DSNN-HSV & DSNN-NASV & +MOPA & +AE-BSM \\
\midrule

$\beta_{1}$ & 0.0133 & 0.0264\sym{**} & 0.0118 & 0.0107 & -0.0372\sym{***} & 0.0021 & -0.0083 & -0.0195\sym{**}\\
       & (0.0128) & (0.0128) & (0.0130) & (0.0133) & (0.0137) & (0.0131) & (0.0116) & (0.0094)\\
$\beta_{0}$ & -0.0039 & -0.0069\sym{**} & -0.0051\sym{*} & -0.0051\sym{*} & 0.0087\sym{***} & -0.0014 & -0.0013 & 0.0041\sym{*}\\
         & (0.0028) & (0.0028) & (0.0029) & (0.0029) & (0.0030) & (0.0029) & (0.0025) & (0.0021)\\

\midrule
Obs & 317 & 317 & 317 & 317 & 317 & 317 & 317 & 317\\
Adj. R\sym{2} & 0.0003 & 0.0102 & -0.0005 & -0.0011 & 0.0199 & -0.0031 & -0.0015 & 0.0104\\

\bottomrule
\end{tabular}
}

\end{table}

Table (\ref{tab:dRMSE_pricing_shorter}) reports the regression results for 4 structural models, NN+Bnd, and 8 SKINN variants. We do not consider NN+Shape and PINN+BSM in this test due to their gradient pathologies. When considering the shorter prediction horizons as the testing periods, NN+Bnd and SKINNs perform closely, as it is hard for all of them to provide large and significant option pricing improvements against NN. NN+Bnd provides a tiny marginal improvement when the market is volatile, which is significant at only the 10\% significance level. SKINN+AE-BSM, in this case, provides a slightly larger improvement, which is significant at the 5\% significance level. In these shorter horizons, structural models, however, present significantly larger errors than NN when the market in the testing period has a higher volatility.

\begin{table}[htbp]

\caption{We run the regression (\ref{reg_avgvix}) for each model. A statistically significant negative $\beta_{1}$ implies that the model performs better compared with a plain vanilla neural network when the market is more volatile. The results are based on all 317 test periods of the longer prediction horizons.}
\label{tab:dRMSE_pricing_longer}

\resizebox{\textwidth}{!}{
\begin{tabular}{@{}lcccccccc@{}}
\toprule

\textbf{Model} & \multirow{2}{*}{BSM} & \multirow{2}{*}{ABSM} & \multirow{2}{*}{HSV} & \multirow{2}{*}{HSVJ} & NN & -- & -- & -- \\
\textbf{Panel (A)} & & & & & +Bnd & & & \\
\midrule
$\beta_{1}$ & -0.0317 & -0.0274 & -0.0483\sym{**} & -0.0340 & -0.0303\sym{***} & -- & -- & -- \\
 & (0.0214) & (0.0214) & (0.0215) & (0.0216) & (0.0096) & -- & -- & -- \\
$\beta_{0}$ & 0.0019 & 0.0013 & 0.0053 & -0.0008 & 0.0027 & -- & -- & -- \\
 & (0.0047) & (0.0047) & (0.0047) & (0.0048) & (0.0021) & -- & -- & -- \\

\midrule
Obs & 317 & 317 & 317 & 317 & 317 & -- & -- & -- \\
Adj. R\sym{2} & 0.0038 & 0.0020 & 0.0127 & 0.0047 & 0.0273 & -- & -- & -- \\

\midrule

\textbf{Model} & SKINN & SKINN & SKINN & SKINN & SKINN & SKINN & SKINN & SKINN \\
\textbf{Panel (B)} & +BSM & +ABSM & +HSV & +HSVJ & +DSNN-HSV & +DSNN-NASV & +MOPA & +AE-BSM \\
\midrule

$\beta_{1}$ & -0.1000\sym{***} & -0.0868\sym{***} & -0.0944\sym{***} & -0.0948\sym{***} & -0.1369\sym{***} & -0.1031\sym{***} & -0.0923\sym{***} & -0.0944\sym{***} \\
 & (0.0205) & (0.0201) & (0.0199) & (0.0196) & (0.0195) & (0.0202) & (0.0177) & (0.0132) \\
$\beta_{0}$ & 0.0119\sym{***} & 0.0091\sym{**} & 0.0102\sym{**} & 0.0101\sym{**} & 0.0227\sym{***} & 0.0135\sym{***} & 0.0106\sym{***} & 0.0155\sym{***} \\
 & (0.0045) & (0.0044) & (0.0044) & (0.0043) & (0.0043) & (0.0045) & (0.0039) & (0.0029) \\

\midrule
Obs & 317 & 317 & 317 & 317 & 317 & 317 & 317 & 317 \\
Adj. R\sym{2} & 0.0671 & 0.0528 & 0.0634 & 0.0661 & 0.1329 & 0.0733 & 0.0764 & 0.1374 \\

\bottomrule
\end{tabular}
}

\end{table}

Table (\ref{tab:dRMSE_pricing_longer}) reports the results for the longer prediction horizons. At this time, all the coefficients of AvgVIX for NN+Bnd, and SKINNs are significantly negative at 1\% level. This aligns with our finding from Table (\ref{tab:oos_pricing_dm_longer}) that SKINNs, including NN+Bnd, are capable of improving the out-of-sample pricing accuracy, compared with NN. The negative coefficients indicate that NN+Bnd and SKINNs provide marginal improvements against NN when the market is more volatile. Such countercyclical option pricing performance is economically desirable, as volatile market conditions, associated with noisier option price patterns, usually deteriorate plain data-driven as well as structural option pricing models. Structural models can not provide significantly lower out-of-sample pricing errors, except that the HSV model offers limited improvement. Though statistically significant, the magnitude of the marginal improvement during high volatility market conditions provided by NN+Bnd is small as well, which is around 0.03. Meanwhile, the marginal improvement magnitude of all SKINN variants is 3--4 times larger than that of NN+Bnd. Interestingly, for SKINN+AE-BSM, it does not statistically outperform the structural models in the longer prediction horizons, according to the Diebold-Mariano test; however, according to the regression, it offers significantly larger marginal improvement than the structural models.

We then perform zoom-in analysis for the marginal improvements of models against NN in the longer prediction horizons, where the option price patterns shift significantly from the patterns in the model estimation periods. We inspect whether the magnitude of the marginal option pricing improvement of SKINNs in volatile market conditions differs by the level of volatility itself. We divide the 317 longer horizon testing periods into three groups: low volatility periods, medium volatility periods, and high volatility periods.\footnote{Low volatility periods, medium volatility periods, and high volatility periods contain the periods for which the AvgVIX is below the 20\% percentile, from the 20\% to the 80\% percentile, and greater than the 80\% percentile of the AvgVIX over all periods, respectively.}

Table (\ref{tab:dRMSE_pricing_longer_lowvol}) reports the regression results for the low volatility periods of the longer prediction horizons. In these periods, all models can not be statistically differentiated from NN, as all coefficients of AvgVIX are statistically insignificant. Except that SKINN+DSNN-HSV displays a significantly large improvement, which is driven by outliers in this small sample. This is expected, since the option price patterns from low volatility market conditions are relatively easier to learn, due to less mispricing and less noise in prices.

\begin{table}[htbp]

\caption{We run the regression (\ref{reg_avgvix}) for each model. A statistically significant negative $\beta_{1}$ implies that the model performs better compared with a plain vanilla neural network when the market is more volatile. The results are based on the 64 low volatility testing periods in the longer prediction horizons, where $\text{AvgVIX}(t)$ is below the 20\% quantile of AvgVIX across all the longer horizon testing periods.}
\label{tab:dRMSE_pricing_longer_lowvol}

\resizebox{\textwidth}{!}{
\begin{tabular}{@{}lcccccccc@{}}
\toprule

\textbf{Model} & \multirow{2}{*}{BSM} & \multirow{2}{*}{ABSM} & \multirow{2}{*}{HSV} & \multirow{2}{*}{HSVJ} & NN & -- & -- & -- \\
\textbf{Panel (A)} & & & & & +Bnd & & & \\
\midrule
$\beta_{1}$ & -0.1979 & -0.1172 & -0.1354 & -0.1288 & -0.0346 & -- & -- & -- \\
 & (0.1657) & (0.1644) & (0.1646) & (0.1671) & (0.1192) & -- & -- & -- \\
$\beta_{0}$ & 0.0230 & 0.0125 & 0.0167 & 0.0119 & 0.0024 & -- & -- & -- \\
 & (0.0205) & (0.0204) & (0.0204) & (0.0207) & (0.0148) & -- & -- & -- \\

\midrule
Obs & 64 & 64 & 64 & 64 & 64 & -- & -- & -- \\
Adj. R$^{2}$ & 0.0067 & -0.0079 & -0.0052 & -0.0065 & -0.0148 & -- & -- & -- \\

\midrule

\textbf{Model} & SKINN & SKINN & SKINN & SKINN & SKINN & SKINN & SKINN & SKINN \\
\textbf{Panel (B)} & +BSM & +ABSM & +HSV & +HSVJ & +DSNN-HSV & +DSNN-NASV & +MOPA & +AE-BSM \\
\midrule

$\beta_{1}$ & -0.2303 & -0.1690 & -0.0628 & 0.0021 & -0.6827\sym{***} & -0.1875 & -0.1741 & -0.1343 \\
 & (0.1504) & (0.1528) & (0.1614) & (0.1836) & (0.1828) & (0.1524) & (0.1270) & (0.1275) \\
$\beta_{0}$ & 0.0252 & 0.0162 & 0.0037 & -0.0035 & 0.0912\sym{***} & 0.0221 & 0.0175 & 0.0167 \\
 & (0.0186) & (0.0189) & (0.0200) & (0.0227) & (0.0226) & (0.0189) & (0.0157) & (0.0158) \\

\midrule
Obs & 64 & 64 & 64 & 64 & 64 & 64 & 64 & 64 \\
Adj. R$^{2}$ & 0.0209 & 0.0035 & -0.0137 & -0.0161 & 0.1705 & 0.0081 & 0.0138 & 0.0018 \\

\bottomrule
\end{tabular}
}

\end{table}

Table (\ref{tab:dRMSE_pricing_longer_mediumvol}) reports the regression results for the medium volatility market conditions. SKINNs, together with NN+Bnd, provide significant marginal option pricing improvement against NN when the AvgVIX increases, across the medium volatility testing periods. Though statistically significant, the marginal improvement magnitude of NN+Bnd is much smaller than that of SKINNs. SKINN+DSNN-NASV, in this case, provides the largest marginal improvement in magnitude for the NN-based models.

\begin{table}[htbp]

\caption{We run the regression (\ref{reg_avgvix}) for each model. A statistically significant negative $\beta_{1}$ implies that the model performs better compared with a plain vanilla neural network when the market is more volatile. The results are based on the 189 medium volatility testing periods in the longer prediction horizons, where $\text{AvgVIX}(t)$ is between the 20\% quantile and the 80\% quantile of AvgVIX across all the longer horizon testing periods.}
\label{tab:dRMSE_pricing_longer_mediumvol}

\resizebox{\textwidth}{!}{
\begin{tabular}{@{}lcccccccc@{}}
\toprule

\textbf{Model} & \multirow{2}{*}{BSM} & \multirow{2}{*}{ABSM} & \multirow{2}{*}{HSV} & \multirow{2}{*}{HSVJ} & NN & -- & -- & -- \\
\textbf{Panel (A)} & & & & & +Bnd & & & \\
\midrule
$\beta_{1}$ & -0.0708 & -0.0630 & -0.1299\sym{**} & -0.0974\sym{*} & -0.0686\sym{**} & -- & -- & -- \\
 & (0.0525) & (0.0525) & (0.0524) & (0.0532) & (0.0275) & -- & -- & -- \\
$\beta_{0}$ & 0.0086 & 0.0075 & 0.0202\sym{**} & 0.0104 & 0.0098\sym{*} & -- & -- & -- \\
 & (0.0102) & (0.0103) & (0.0102) & (0.0104) & (0.0054) & -- & -- & -- \\

\midrule
Obs & 189 & 189 & 189 & 189 & 189 & -- & -- & -- \\
Adj. R\sym{2} & 0.0043 & 0.0023 & 0.0266 & 0.0124 & 0.0270 & -- & -- & -- \\

\midrule

\textbf{Model} & SKINN & SKINN & SKINN & SKINN & SKINN & SKINN & SKINN & SKINN \\
\textbf{Panel (B)} & +BSM & +ABSM & +HSV & +HSVJ & +DSNN-HSV & +DSNN-NASV & +MOPA & +AE-BSM \\
\midrule

$\beta_{1}$ & -0.1027\sym{**} & -0.0992\sym{**} & -0.1062\sym{**} & -0.1114\sym{**} & -0.1295\sym{***} & -0.1303\sym{***} & -0.0758\sym{*} & -0.1298\sym{***} \\
 & (0.0471) & (0.0464) & (0.0473) & (0.0458) & (0.0449) & (0.0466) & (0.0420) & (0.0323) \\
$\beta_{0}$ & 0.0130 & 0.0120 & 0.0126 & 0.0131 & 0.0202\sym{**} & 0.0188\sym{**} & 0.0080 & 0.0222\sym{***} \\
 & (0.0092) & (0.0090) & (0.0092) & (0.0089) & (0.0088) & (0.0091) & (0.0082) & (0.0063) \\

\midrule
Obs & 189 & 189 & 189 & 189 & 189 & 189 & 189 & 189 \\
Adj. R\sym{2} & 0.0195 & 0.0187 & 0.0211 & 0.0255 & 0.0374 & 0.0350 & 0.0119 & 0.0747 \\

\bottomrule
\end{tabular}
}

\end{table}

Table (\ref{tab:dRMSE_pricing_longer_highvol}) reports the regression results for the high volatility market conditions. Within these most volatile market periods, only SKINNs still survive to provide statistically significant marginal option pricing improvements against NN when the AvgVIX increases, at 10\% level (SKINN+BSM, SKINN+ABSM, and SKINN+DSNN-NASV); at 5\% level (SKINN+HSV, SKINN+HSVJ, as well as SKINN+DSNN-HSV, and SKINN+MOPA); and at 1\% level (SKINN+AE-BSM). More importantly, the marginal improvements offered by SKINNs in this case are the largest among all volatility groups. SKINN+AE-BSM provides the greatest marginal improvement for the option pricing performance when the market volatility is even higher in an already volatile condition. 

\begin{table}[htbp]

\caption{We run the regression (\ref{reg_avgvix}) for each model. A statistically significant negative $\beta_{1}$ implies that the model performs better compared with a plain vanilla neural network when the market is more volatile. The results are based on the 64 high volatility testing periods in the longer prediction horizons, where $\text{AvgVIX}(t)$ is above the 80\% quantile of AvgVIX across all the longer horizon testing periods.}
\label{tab:dRMSE_pricing_longer_highvol}

\resizebox{\textwidth}{!}{
\begin{tabular}{@{}lcccccccc@{}}
\toprule

\textbf{Model} & \multirow{2}{*}{BSM} & \multirow{2}{*}{ABSM} & \multirow{2}{*}{HSV} & \multirow{2}{*}{HSVJ} & NN & -- & -- & -- \\
\textbf{Panel (A)} & & & & & +Bnd & & & \\
\midrule
$\beta_{1}$ & -0.0364 & -0.0416 & -0.0345 & -0.0272 & -0.0466 & -- & -- & -- \\
 & (0.0783) & (0.0787) & (0.0788) & (0.0791) & (0.0287) & -- & -- & -- \\
$\beta_{0}$ & 0.0051 & 0.0077 & 0.0025 & -0.0013 & 0.0094 & -- & -- & -- \\
 & (0.0262) & (0.0264) & (0.0264) & (0.0265) & (0.0096) & -- & -- & -- \\

\midrule
Obs & 64 & 64 & 64 & 64 & 64 & -- & -- & -- \\
Adj. R\sym{2} & -0.0126 & -0.0116 & -0.0130 & -0.0142 & 0.0255 & -- & -- & -- \\

\midrule

\textbf{Model} & SKINN & SKINN & SKINN & SKINN & SKINN & SKINN & SKINN & SKINN \\
\textbf{Panel (B)} & +BSM & +ABSM & +HSV & +HSVJ & +DSNN-HSV & +DSNN-NASV & +MOPA & +AE-BSM \\
\midrule

$\beta_{1}$ & -0.1464\sym{*} & -0.1372\sym{*} & -0.1516\sym{**} & -0.1474\sym{**} & -0.1555\sym{**} & -0.1373\sym{*} & -0.1526\sym{**} & -0.1825\sym{***} \\
 & (0.0780) & (0.0761) & (0.0739) & (0.0724) & (0.0720) & (0.0767) & (0.0660) & (0.0453) \\
$\beta_{0}$ & 0.0282 & 0.0270 & 0.0308 & 0.0295 & 0.0307 & 0.0263 & 0.0315 & 0.0478\sym{***} \\
 & (0.0261) & (0.0255) & (0.0248) & (0.0243) & (0.0241) & (0.0257) & (0.0221) & (0.0152) \\

\midrule
Obs & 64 & 64 & 64 & 64 & 64 & 64 & 64 & 64 \\
Adj. R\sym{2} & 0.0385 & 0.0344 & 0.0484 & 0.0475 & 0.0549 & 0.0337 & 0.0645 & 0.1946 \\

\bottomrule
\end{tabular}
}

\end{table}

\subsection{The Estimator Side of SKINNs}

The SKINNs framework not only regularizes the NN component but also functions as an estimator for the latent parameters in complicated and high-dimensional models, converging to a GMM estimator under certain conditions. In this section, we proceed to examine the econometric estimator side of SKINNs. Different from the sequential approach adopted by \cite{chen2023teaching} in transfer learning models, which requires a separate non-convex and non-linear calibration, SKINNs learn $\phi$ simultaneously with the NN component, during the same online training process. We demonstrate that the learned $\phi$ by SKINNs possesses strong economic interpretations, rather than uninterpretable nuisance parameters, particularly in high-dimensional settings.

\subsubsection{The Accuracy of the SKINN-learned Structural Parameters}
Figure (\ref{fig:SKINNs_structural_g_phi}) illustrates the time-series of the learned latent structural parameters from the SKINNs embedded with PSKRs, over the 317 model training periods. We compare the SKINN-learned $\phi$ with the conventionally calibrated ones. The dynamics of most of the SKINN-learned latent structural parameters evolve with obvious economic regimes, signalling that those quantities carry economic information. For the BSM case, the SKINN-learned $\sigma$ closely tracks the calibrated implied volatility. For more sophisticated higher-dimensional cases, the SKINN-learned $\phi$ can diverge from the conventionally calibrated counterpart. In general, the SKINN-learned $\phi$ shows smoother evolution over time. 

The latent structural parameter vector learned from SKINN+MOPA, which uses an NPSKR of an unknown distribution, is very high-dimensional (2,000-dimensional). It is impossible to visualize the evolution of such an ultra-high-dimensional vector. As these parameters can be interpreted as the risk-neutral probabilities, we instead visualize the distributions that consist of these learned probabilities. Figure (\ref{fig:SKINN_MOPA_example}) is an example of the learned distributions from a selected training period.

According to \cite{chen2021deep}, if one option pricing model is well-specified, then we should expect small variations in the model parameters across time. We calculate the $l_{2}$-norm of the difference between the latent parameters $\phi_{t+1}$ and $\phi_{t}$, $1\leq t\leq 316$, from two consecutive periods. Figure (\ref{fig:hidden_econ_states_time_variation_distribution}) shows the histograms of the periodic parameter variations (the $l_{2}$-norms) across our 317 training periods. For the more sophisticated specifications, e.g., the HSV and the HSVJ ones, the SKINN-learned $\phi$ presents smaller variations than the counterparts that are directly calibrated from market prices. This confirms the improved stability of the SKINN-learned structural parameters.

\subsubsection{The Economic Inference of the SKINN-learned Structural Parameters}
The latent structural parameters learned with SKINNs are expected to carry some economic information if they are not nuisance parameters. We examine this by regressing the AvgVIX over the shorter prediction horizons against the calibrated $\phi$ and the SKINN-learned $\phi$ separately, from all training periods. Specifically, we run the regression:
\begin{equation}
    \text{AvgVIX}_{t+1} = \beta_{0} + \beta_{1}\phi^{(d)}_{t} + \varepsilon, \;\;\; 1\leq t\leq317,\;d\in\{1, 2, 3, \cdots d_{\phi}\},
\end{equation}
We only include one latent parameter at a time to examine its individual predictive power for the market volatility in the subsequent months. 

Table (\ref{tab:sturctural_model_phi_vix}) reports the results of such univariate regressions by using the calibrated $\phi$ from structural models. Table (\ref{tab:parametric_skinns_phi_vix}) reports the results for the learned $\phi$ from SKINNs. Most of the SKINN-learned latent parameters carry significant predictive information for the market volatility (AvgVIX) in the subsequent months, as most of their regression coefficients are significantly non-zero. Additionally, for more than half of the parameters, the SKINN-learned one has a relatively higher adjusted R\sym{2} than the calibrated one.

\begin{table}[htbp]
    \centering
    \caption{The predictive information contained in each of the parameters calibrated according to the standard procedure, for the lagged averaged AvgVIX over one month. $\{\sigma\}$ is from the BSM model, $\{a_{1}, a_{2}, a_{3}, a_{4}, a_{5}, a_{6}\}$ is from the ABSM model, $\{\bar{v}, v_{0}, \sigma_{v}, \rho, \kappa\}$ is from the HSV model, and $\{\bar{v}, v_{0}, \sigma_{v}, \rho, \kappa, \lambda, \eta_{0}, \eta_{1}, p_{up}\}$ is from the HSVJ model.}
    \label{tab:sturctural_model_phi_vix}

    \resizebox{\textwidth}{!}{

    \begin{tabular}{@{}ccccccccccccc@{}}
    \toprule
    \textbf{BSM} & $\sigma$ & -- & -- & -- & -- & --  \\

    \midrule
    $\beta_{1}$ & 1.1181*** & -- & -- & -- & -- & -- \\
           & (18.781) & -- & -- & -- & -- & -- \\

    $\beta_{0}$ & 0.0255** & -- & -- & -- & -- & -- \\
                 & (2.543) & -- & -- & -- & -- & -- \\
    \midrule
    Obs & 317 & -- & -- & -- & -- & -- \\
    Adj.\ $R^{2}$ & 52.7\% & -- & -- & -- & -- & -- \\

    \midrule

    \textbf{ABSM} & $a_{1}$ & $a_{2}$ & $a_{3}$ & $a_{4}$ & $a_{5}$ & $a_{6}$ \\
    \midrule

    $\beta_{1}$ & 0.7555\sym{***} & -0.1997 & -0.2349 & -1.4114\sym{***} & -2.4120\sym{***} & -3.1940\sym{***} \\
 & (0.0521) & (0.2302) & (0.1733) & (0.3630) & (0.9176) & (0.9057) \\
$\beta_{0}$ & 0.0944\sym{***} & 0.2053\sym{***} & 0.2061\sym{***} & 0.2159\sym{***} & 0.2084\sym{***} & 0.2090\sym{***} \\
 & (0.0084) & (0.0046) & (0.0046) & (0.0053) & (0.0047) & (0.0046) \\
    
    \midrule
    Obs & 317 & 317 & 317 & 317 & 317 & 317 \\
    Adj. R\sym{2} & 39.85\% & -0.08\% & 0.26\% & 4.28\% & 1.84\% & 3.49\% \\

    \midrule
    \textbf{HSV} & $\bar{v}$ & $v_{0}$ & $\sigma_{v}$ & $\rho$ & $\kappa$ & --\\
    \midrule

    $\beta_{1}$ & -0.0106 & 1.3240*** & -0.0024*** & -0.0412*** & 0.0022 & --\\
              & (-0.250) & (13.915) & (-3.321) & (-3.038) & (1.314) & --\\

    $\beta_{0}$ & 0.2063*** & 0.1502*** & 0.2156*** & 0.1826*** & 0.2000*** & --\\
    & (26.263) & (28.394) & (39.002) & (21.413) & (34.945) & --\\
    \midrule
    Obs & 317 & 317 & 317 & 317 & 317 & --\\
    Adj. R\sym{2} & -0.3\% & 37.9\% & 3.1\% & 2.5\% & 0.2\% & --\\

    \midrule
    \textbf{HSVJ} & $\bar{v}$ & $v_{0}$ & $\sigma_{v}$ & $\rho$ & $\kappa$ & $\lambda$ \\
    \midrule

    $\beta_{1}$ & 1.1575\sym{***} & 1.7827\sym{***} & 0.0546\sym{***} & 0.1118\sym{***} & -0.0022\sym{**} & 0.0257\sym{***} \\
     & (0.2342) & (0.1088) & (0.0077) & (0.0224) & (0.0009) & (0.0044) \\
    $\beta_{0}$ & 0.1737\sym{***} & 0.1502\sym{***} & 0.1784\sym{***} & 0.2933\sym{***} & 0.2183\sym{***} & 0.1967\sym{***}\\
     & (0.0076) & (0.0047) & (0.0056) & (0.0183) & (0.0069) & (0.0045) \\
    
    \midrule
    Obs & 317 & 317 & 317 & 317 & 317 & 317\\
    Adj. R\sym{2} & 6.90\% & 45.83\% & 13.38\% & 7.03\% & 1.75\% & 9.51\% \\

    \midrule
    \textbf{HSVJ Cont'd} & $\eta_{0}$ & $\eta_{1}$ & $p_{up}$ & -- & -- & --\\
    \midrule

    $\beta_{1}$ & -0.0001 & -0.0000 & -0.0099 & -- & -- & --\\
     & (0.0001) & (0.0003) & (0.0119) & -- & -- & -- \\
    $\beta_{0}$ & 0.2053\sym{***} & 0.2049\sym{***} & 0.2076\sym{***} & -- & -- & --\\
     & (0.0046) & (0.0049) & (0.0057) & -- & -- & --\\
    
    \midrule
    Obs &317 & 317 & 317 & -- & -- & --\\
    Adj. R\sym{2} & -0.16\% & -0.31\% & -0.10\% & -- & -- & --\\
        
    \bottomrule
    \end{tabular}
    }

\end{table}

\begin{table}[htbp]
    \centering
    \caption{The predictive information contained in each of the parameters learned from the SKINNs, for the lagged averaged AvgVIX over one month. $\{\sigma^{*}\}$ is from $g_{\phi}^{\text{BSM}}$, $\{a_{1}^{*}, a_{2}^{*}, a_{3}^{*}, a_{4}^{*}, a_{5}^{*}, a_{6}^{*}\}$ is from $g_{\phi}^{\text{ABSM}}$, $\{\bar{v}^{*}, v_{0}^{*}, \sigma_{v}^{*}, \rho^{*}, \kappa^{*}\}$ is from $g_{\phi}^{\text{HSV}}$, and $\{\bar{v}^{*}, v_{0}^{*}, \sigma_{v}^{*}, \rho^{*}, \kappa^{*}, \lambda^{*}, \eta_{0}^{*}, \eta_{1}^{*}, p_{up}^{*}\}$ is from $g_{\phi}^{\text{HSVJ}}$.}
    \label{tab:parametric_skinns_phi_vix}

    \resizebox{\textwidth}{!}{

    \begin{tabular}{@{}ccccccccccccc@{}}
    \toprule
    
    \textbf{BSM} & $\sigma^{*}$ & -- & -- & -- & -- & -- \\
    \midrule

    $\beta_{1}$ & 1.0896*** & -- & -- & -- & -- & --\\
                & (17.329) & -- & -- & -- & -- & -- \\

    $\beta_{0}$ & 0.0323*** & -- & -- & -- & -- & -- \\
                & (3.087) & -- & -- & -- & -- & -- \\
    \midrule
    Obs & 317 & -- & -- & -- & -- & -- \\
    Adj.\ $R^{2}$ & 48.6\% & -- & -- & -- & -- & -- \\

    \midrule

    \textbf{ABSM} & $a^{*}_{1}$ & $a^{*}_{2}$ & $a^{*}_{3}$ & $a^{*}_{4}$ & $a^{*}_{5}$ & $a^{*}_{6}$ \\
    \midrule

    $\beta_{1}$ & 1.0456*** & -13.3473* & -25.9413* & -1.8508*** & -3.0981 & -32.0793** \\ 
                & (17.481) & (-1.656) & (-1.916) & (-5.932) & (-1.453) & (-2.572) \\

    $\beta_{0}$ & 0.0458*** & 0.2109*** & 0.2096*** & 0.2256*** & 0.2071*** & 0.2108*** \\
                & (4.750)   & (36.071)  & (40.597)  & (40.700)  & (43.138)  & (41.685) \\
    \midrule
    Obs & 317 & 317 & 317 & 317 & 317 & 317 \\
    Adj.\ $R^{2}$ & 49.1\% & 0.5\% & 0.8\% & 9.8\% & 0.4\% & 1.7\% \\

    \midrule
    \textbf{HSV} & $\bar{v}^{*}$ & $v_{0}^{*}$ & $\sigma_{v}^{*}$ & $\rho^{*}$ & $\kappa^{*}$ & -- \\
    \midrule

    $\beta_{1}$ & 0.1047 & 1.8552*** & 0.1744*** & 0.0424* & 0.0441*** & --\\
                & (1.578) & (15.838) & (10.453) & (1.881) & (3.878)& -- \\

    $\beta_{0}$ & 0.1938*** & 0.1478*** & 0.1510*** & 0.2348*** & 0.1889*** & --\\
                & (23.568)  & (30.009)  & (23.446)  & (14.112)  & (31.475) & --\\
    \midrule
    Obs & 317 & 317 & 317 & 317 & 317 & --\\
    Adj.\ $R^{2}$ & 0.5\% & 44.2\% & 25.5\% & 0.8\% & 4.3\% & --\\

    \midrule
    \textbf{HSVJ} & $\bar{v}^{*}$ & $v_{0}^{*}$ & $\sigma_{v}^{*}$ & $\rho^{*}$ & $\kappa^{*}$ & $\lambda^{*}$ \\
    \midrule

    $\beta_{1}$ & 0.1999\sym{***} & 1.8239\sym{***} & 0.1990\sym{***} & 0.0323\sym{*} & 0.0356\sym{**} & 0.4998 \\
     & (0.0585) & (0.1151) & (0.0144) & (0.0183) & (0.0151) & (0.4753) \\
    $\beta_{0}$ & 0.1793\sym{***} & 0.1491\sym{***} & 0.1452\sym{***} & 0.2261\sym{***} & 0.1963\sym{***} & 0.2022\sym{***} \\
     & (0.0086) & (0.0049) & (0.0056) & (0.0130) & (0.0057) & (0.0051) \\
    
    \midrule
    Obs & 317 & 317 & 317 & 317 & 317 & 317 \\
    Adj. R\sym{2} & 3.27\% & 44.19\% & 37.57\% & 0.66\% & 1.43\% & 0.03\% \\

    \midrule
    \textbf{HSVJ Cont'd} & $\eta_{0}^{*}$ & $\eta_{1}^{*}$ & $p_{up}^{*}$ & -- & -- & --\\
    \midrule

    $\beta_{1}$ & -0.0701 & 0.0274\sym{***} & -0.1316\sym{***} & -- & -- & --\\
     & (0.0469) & (0.0067) & (0.0234) & -- & -- & --\\
    $\beta_{0}$ & 0.2195\sym{***} & 0.1803\sym{***} & 0.2374\sym{***} & -- & -- & --\\
     & (0.0109) & (0.0074) & (0.0072) & -- & -- & -- \\
    
    \midrule
    Obs & 317 & 317 & 317 & -- & -- & --\\
    Adj. R\sym{2} & 0.39\% & 4.77\% & 8.83\% & -- & -- & --\\
        
    \bottomrule
    \end{tabular}
    }
\end{table}


For the case of SKINN-MOPA, we translate the 2,000-dimensional learned risk-neutral probabilities to the variances over 10 tenors, then we examine the economic information of the variances over the 10 tenors, via the following regression:
\begin{equation}\label{eq:econ_info_reg_SKINN_MOPA}
    \text{AvgVIX}_{t+1} = \beta_{0} + \beta_{1}\mathbb{V}\text{ar}[\phi^{(T_{k})}_{t}] + \varepsilon, \;\;\; 1\leq t\leq317,\;k\in\{1, 2, 3, \cdots 10\},
\end{equation}
where $T_{k}$ denotes the tenor $k$, and $\phi^{(T_{k})}_{t}$ represents the 200-dimensional learned risk-neutral probabilities for this tenor, for each model training period $t$.

Figure (\ref{fig:SKINN_MOPA_Var_vs_VIX}) shows the translated risk-neutral variances over the 317 model training periods, and the corresponding AvgVIX over the subsequent months. Table (\ref{tab:skinn_mopa_phi_vix}) reports the regression results. Even in this ultra-high-dimensional setting, the learned latent parameters still carry substantial predictive economic information.

\begin{figure}[htbp]
    \centering
    \includegraphics[width=\textwidth]{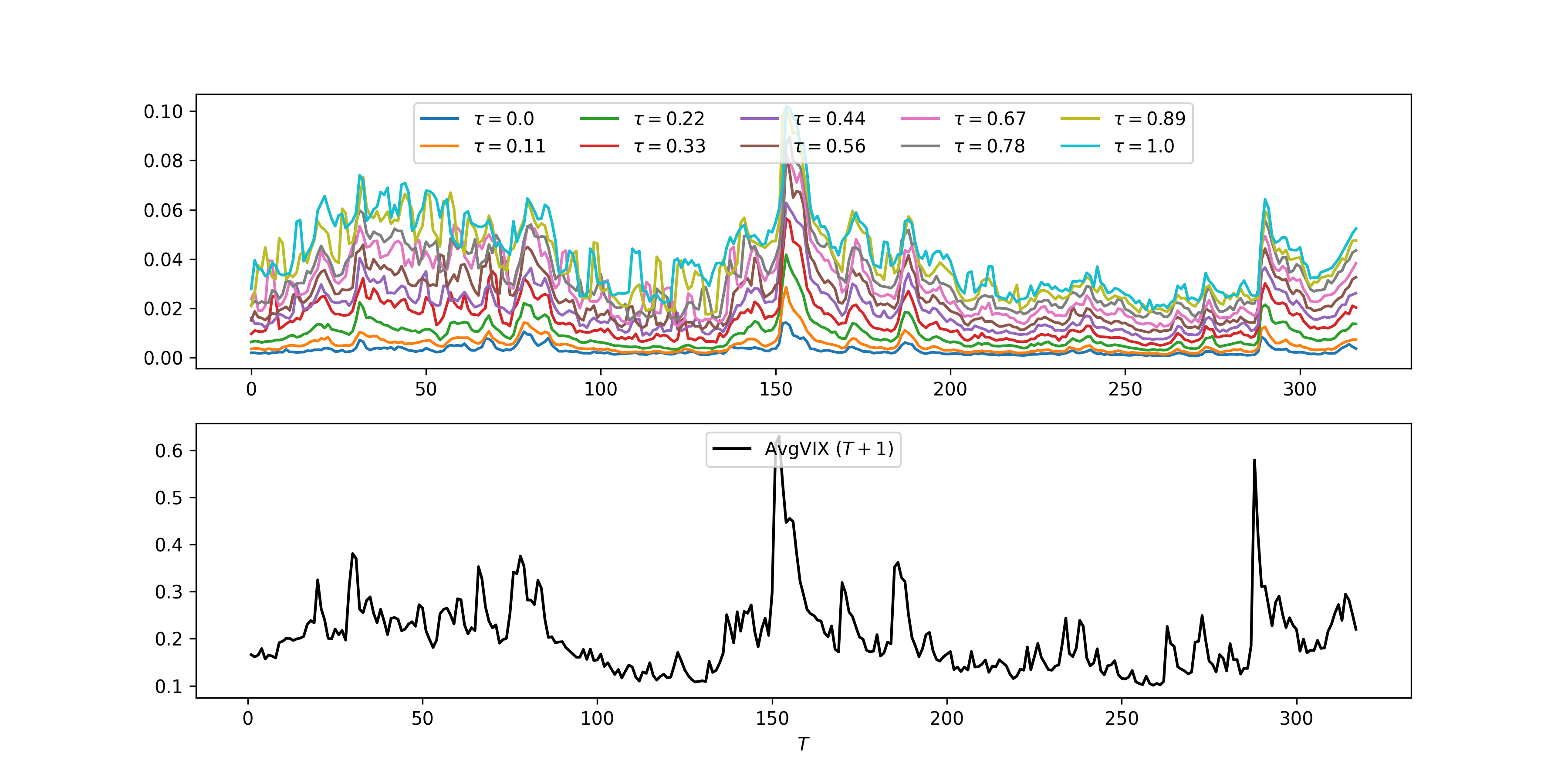}
    \caption{The risk-neutral variances over 10 different tenors. Each risk-neutral variance $\mathbb{V}\text{ar}[\phi^{(T_{k})}_{t}]$ is translated from the learned risk-neutral probabilities by SKINN+MOPA for the $k$-th tenor, $1\leq k\leq10$. The upper panel illustrates the learned risk-neutral variances, and the lower panel illustrates the AvgVIX over the subsequent months after the SKINN+MOPA training periods.}
    \label{fig:SKINN_MOPA_Var_vs_VIX}
\end{figure}
\begin{table}[htbp]

    \centering
    \caption{The predictive information contained in the risk-neural variances converted from the latent parameters (i.e., the 2,000 risk-neutral probabilities over 10 tenors) learned from SKINN+MOPA, for the lagged averaged VIX over one month. $\mathbb{V}\text{ar}[\phi^{(T_{k})}_{t}]$ is the risk-neutral variance converted according to the learned risk-neutral probabilities for the $k$-th tenor.}
    \label{tab:skinn_mopa_phi_vix}

    \resizebox{0.8\textwidth}{!}{

    \begin{tabular}{@{}cccccc@{}}
    \toprule
    \textbf{MOPA} & $\mathbb{V}\text{ar}[\phi^{(T_{1})}]$ & $\mathbb{V}\text{ar}[\phi^{(T_{2})}]$ & $\mathbb{V}\text{ar}[\phi^{(T_{3})}]$ & $\mathbb{V}\text{ar}[\phi^{(T_{4})}]$ & $\mathbb{V}\text{ar}[\phi^{(T_{5})}]$ \\
    \midrule

    $\beta_{1}$ & 26.2949\sym{***} & 16.3890\sym{***} & 10.2331\sym{***} & 6.6412\sym{***} & 5.9800\sym{***} \\
                & (1.6483) & (0.9353) & (0.5712) & (0.4056) & (0.3548)   \\

    $\beta_{0}$ & 0.1292\sym{***} & 0.1217\sym{***} & 0.1098\sym{***} & 0.1041\sym{***} & 0.0884\sym{***} \\
                & (0.0058) & (0.0057) & (0.0062) & (0.0070) & (0.0076) \\
    \midrule
    Obs & 317        & 317        & 317        & 317        & 317        \\
    Adj.\ R\sym{2} & 44.5\%     & 49.2\%     & 50.3\%     & 45.8\%     & 47.3\%     \\

    \midrule

    \textbf{MOPA Cont'd} & $\mathbb{V}\text{ar}[\phi^{(T_{6})}]$ & $\mathbb{V}\text{ar}[\phi^{(T_{7})}]$ & $\mathbb{V}\text{ar}[\phi^{(T_{8})}]$ & $\mathbb{V}\text{ar}[\phi^{(T_{9})}]$ & $\mathbb{V}\text{ar}[\phi^{(T_{10})}]$ \\
    \midrule

    $\beta_{1}$ & 4.7931\sym{***} & 4.0436\sym{***} & 4.1399\sym{***} & 3.4236\sym{***} & 3.3714\sym{***} \\
                & (0.2970) & (0.2599) & (0.2573) & (0.2226) & (0.2269) \\

    $\beta_{0}$ & 0.0889\sym{***} & 0.0849\sym{***} & 0.0694\sym{***} & 0.0720\sym{***} & 0.0629\sym{***} \\
                & (0.0079) & (0.0084) & (0.0090) & (0.0093) & (0.0101) \\
    \midrule
    Obs & 317       & 317        & 317        & 317        & 317        \\
    Adj.\ R\sym{2} & 45.1\%    & 43.3\%     & 44.9\%     & 42.7\%     & 41.0\%     \\
        
    \bottomrule
    \end{tabular}
    }

\end{table}

\begin{figure}[htbp]
    \centering

    \begin{subfigure}[t]{\textwidth}
        \centering
        \includegraphics[width=0.8\textwidth]{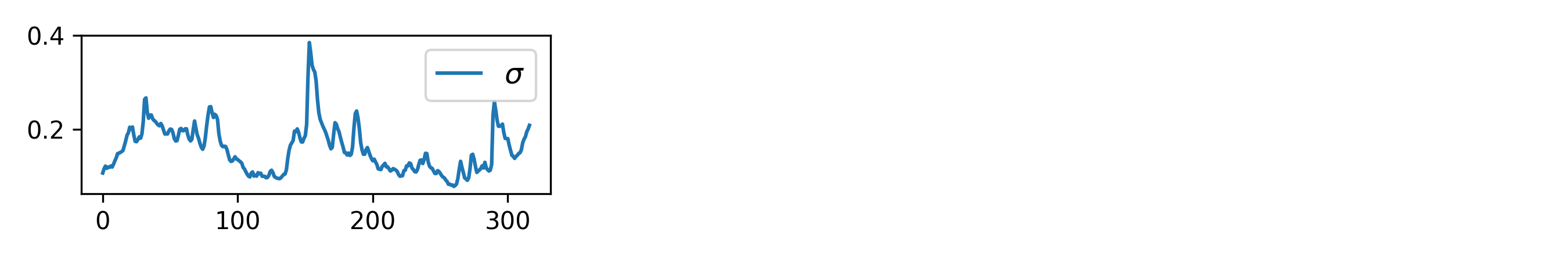}
        \caption{\footnotesize SKINN+BSM}
    \end{subfigure}

    \begin{subfigure}[t]{\textwidth}
        \centering
        \includegraphics[width=0.8\textwidth]{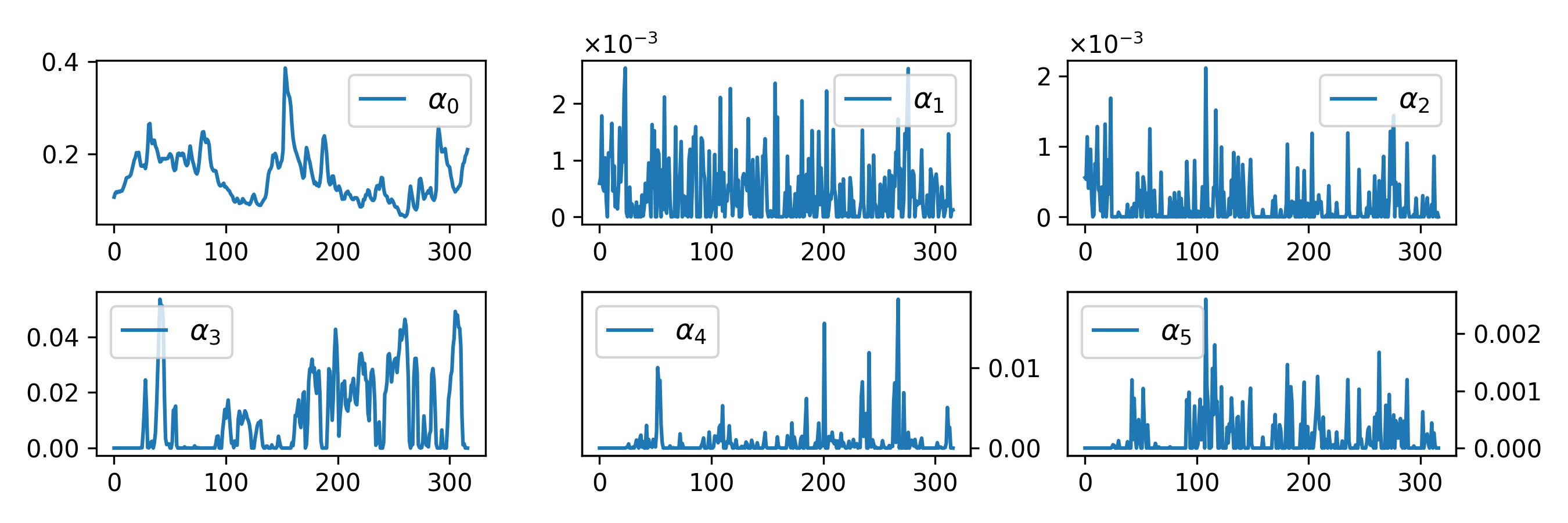}
        \caption{\footnotesize SKINN+ABSM}
    \end{subfigure}

    \begin{subfigure}[t]{\textwidth}\label{fig:skinn_hsv_hes}
        \centering
        \includegraphics[width=0.8\textwidth]{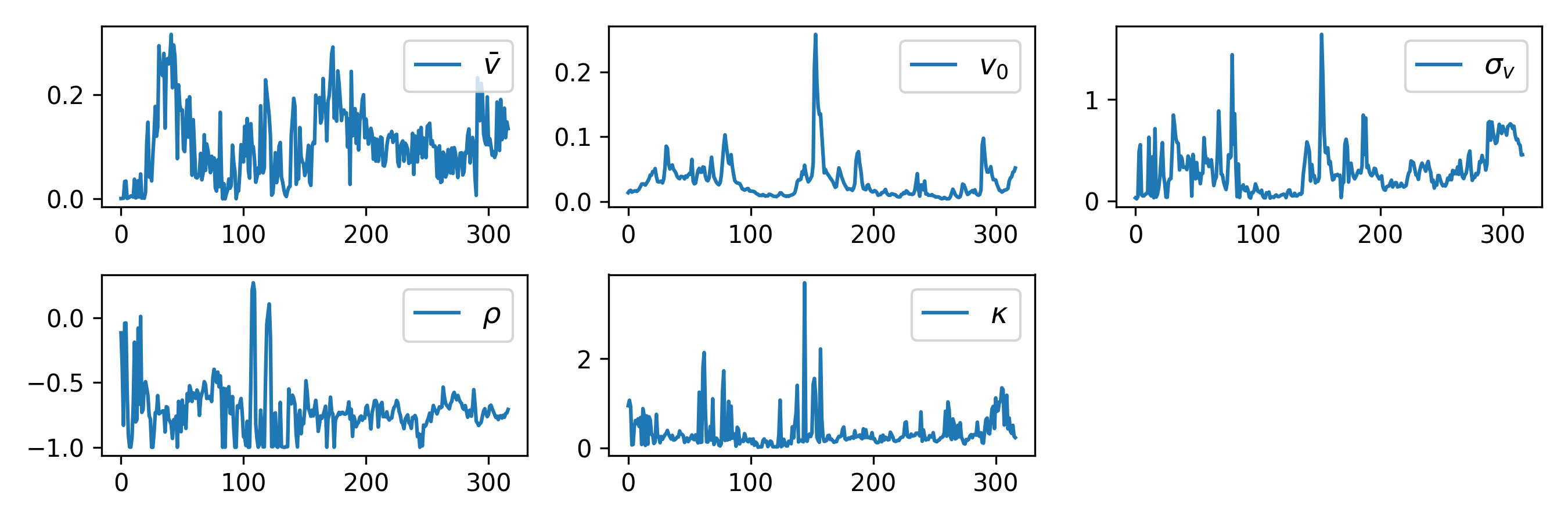}
        \caption{\footnotesize SKINN+HSV}
    \end{subfigure}

    \begin{subfigure}[t]{\textwidth}
        \centering
        \includegraphics[width=0.8\textwidth]{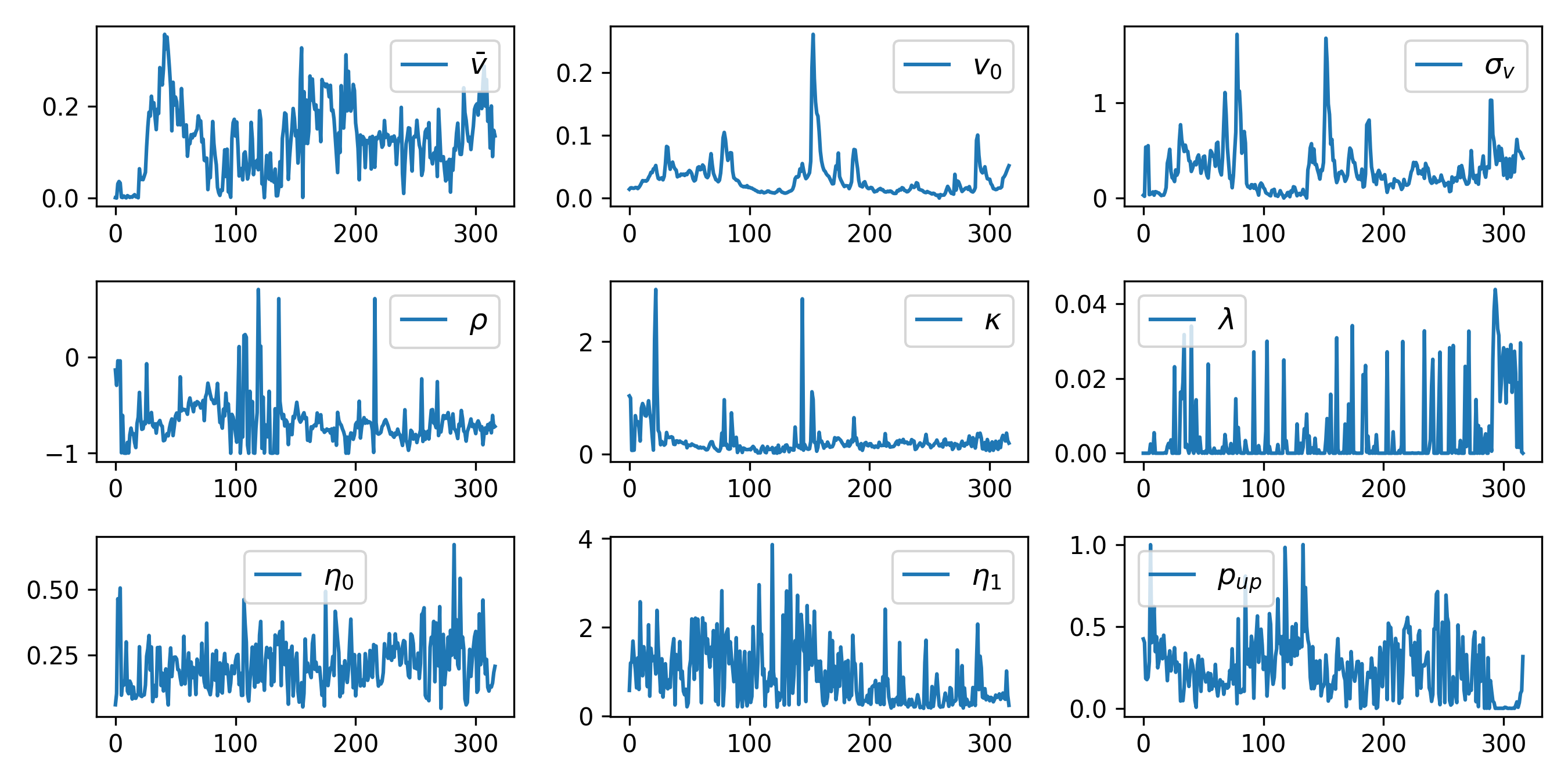}
        \caption{\footnotesize SKINN+HSVJ}
    \end{subfigure}

    \caption{The latent structural parameters learned by SKINNs with PSKRs.}
    \label{fig:SKINNs_structural_g_phi}

\end{figure}


\begin{figure}[htbp]
    \centering
    
    \begin{subfigure}{\textwidth}
        \centering
        \includegraphics[width=\textwidth]{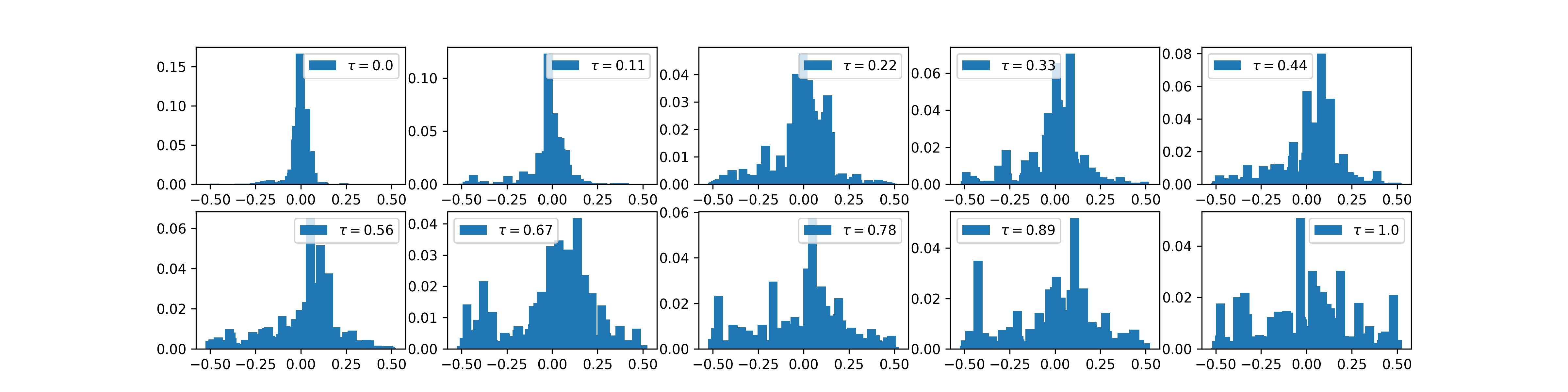}
        \caption{Train period id: 152 (01 Aug 2008 to 31 Oct 2008).}  
    \end{subfigure}\hfill
    
    \begin{subfigure}{\textwidth}
        \centering
        \includegraphics[width=\textwidth]{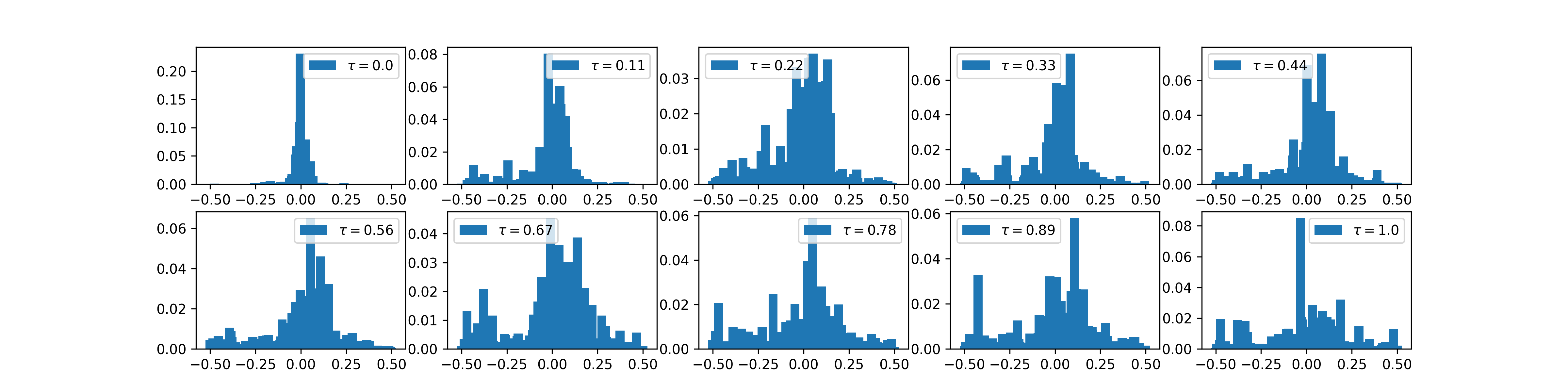}
        \caption{Train period id: 290 (03 Feb 2020 to 30 Apr 2020).}  
    \end{subfigure}
    
    \caption{The latent structural parameters (risk-neutral probabilities) learned by SKINN+MOPA, which uses an NPSKR with an unknown distribution. We select two training periods to demonstrate the learned probabilities, from the 2007-2009 global financial crisis and the COVID-19 crisis, respectively.}
    \label{fig:SKINN_MOPA_example}
\end{figure}

\begin{figure}[htbp]
    \centering
    
    \begin{subfigure}{\textwidth}
        \centering
        \includegraphics[width=0.9\textwidth]{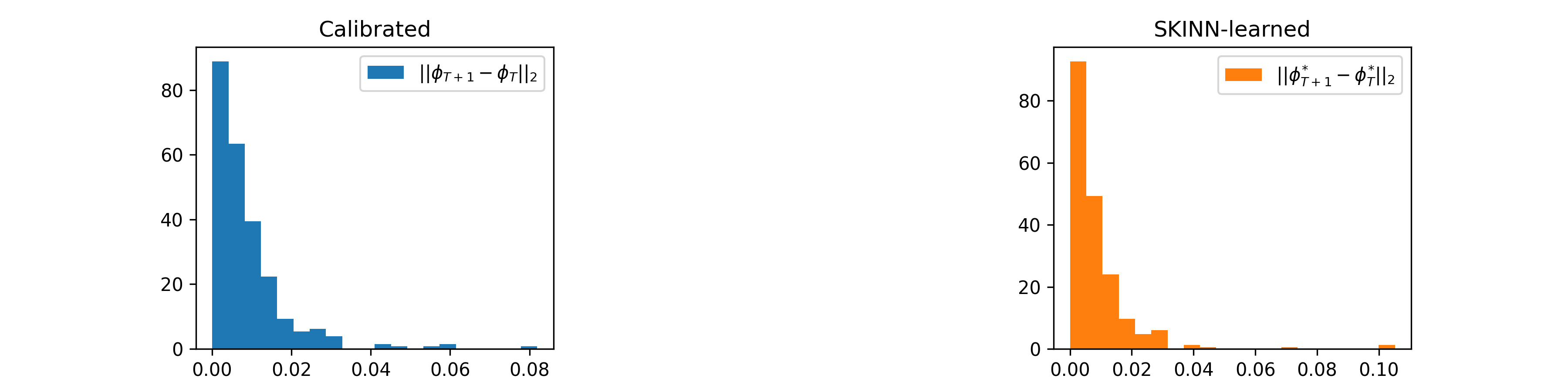}
        \caption{SKINN+BSM}  
    \end{subfigure}
    
    \begin{subfigure}{\textwidth}
        \centering
        \includegraphics[width=0.9\textwidth]{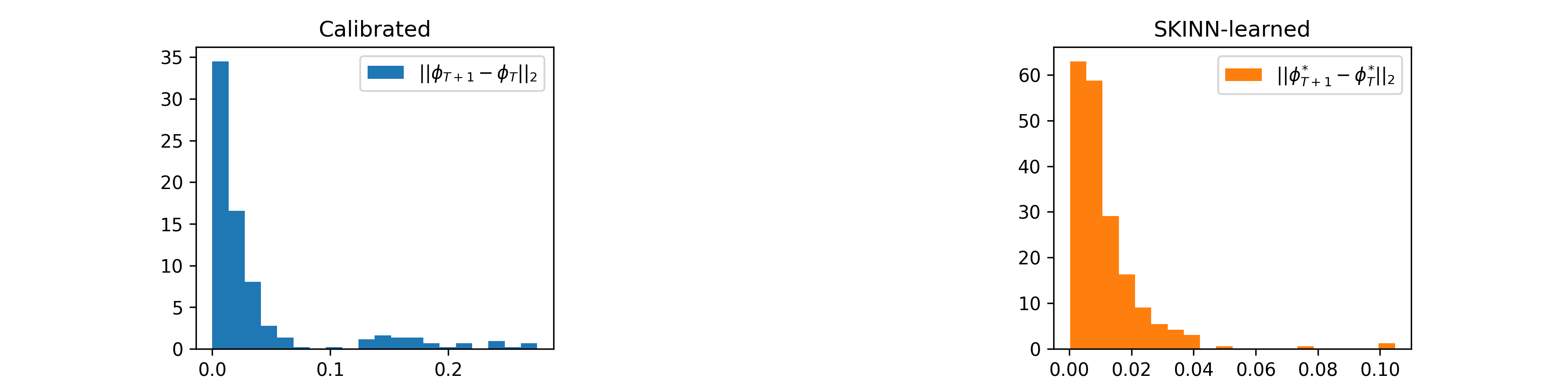}
        \caption{SKINN+ABSM}  
    \end{subfigure}

    \begin{subfigure}{\textwidth}
        \centering
        \includegraphics[width=0.9\textwidth]{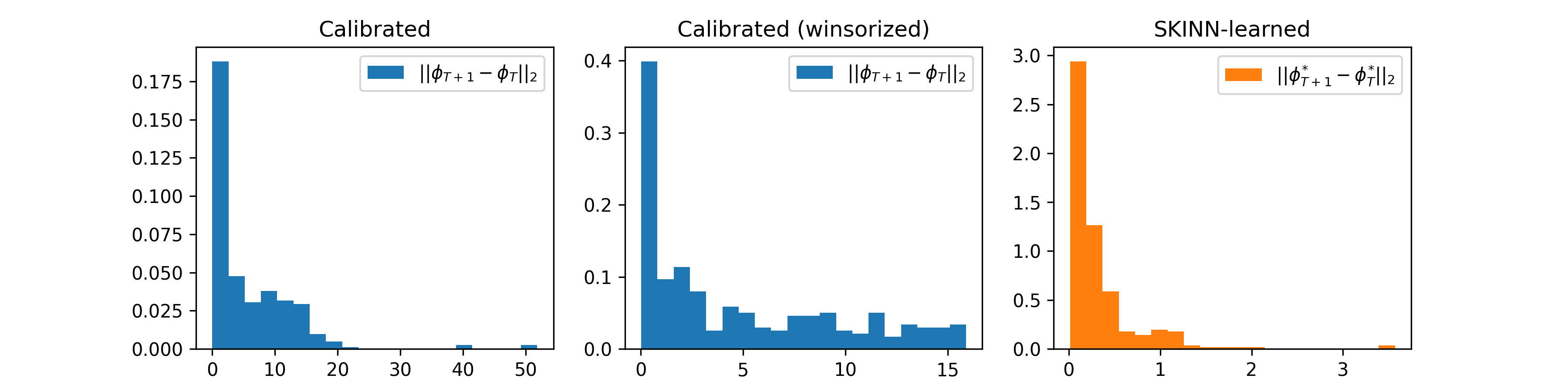}
        \caption{SKINN+HSV}  
    \end{subfigure}

    \begin{subfigure}{\textwidth}
        \centering
        \includegraphics[width=0.9\textwidth]{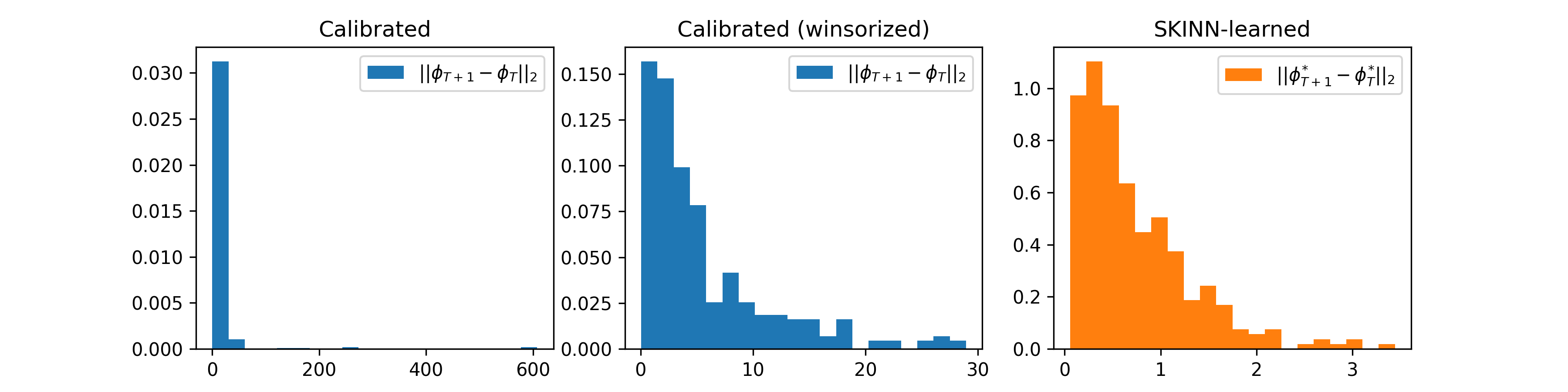}
        \caption{SKINN+HSVJ}  
    \end{subfigure}

    \caption{The histograms of the variation of the SKINN-learned (orange) and the calibrated (blue) latent structural parameters, between two consecutive training periods over time. The SKINN-learned latent structural parameters exhibit a higher density around zero, indicating higher numerical stability compared to those from the conventional calibration.}
    \label{fig:hidden_econ_states_time_variation_distribution}
    
\end{figure}

Lastly, we examine the learned latent parameters from the SKINNs with SPSKRs and the NPSKR with an AE. Figure (\ref{fig:SKINNs_nn_g_phi}) illustrates the evolution of their learned latent parameters over time. Table (\ref{tab:skinn_dsnn_ae_phi_vix}) reports the regression results for these learned parameters. Though the structured-knowledge function in this case is represented by NNs, instead of parametric expressions, our SKINNs framework can still identify the unknown latent parameters successfully. Interestingly, the two DSNNs (DSNN-HSV, DSNN-NASV) and the machine-derived proxy (AE-BSM) are imperfect approximations for the option price DGP, but all the latent parameters learned from them are statistically significant and economically informative, with considerably high adjusted R\sym{2}, which even surpasses the latent parameters from PSKRs.

\begin{figure}[htbp]
    \centering

    \begin{subfigure}[t]{\textwidth}
        \centering
        \includegraphics[width=0.8\textwidth]{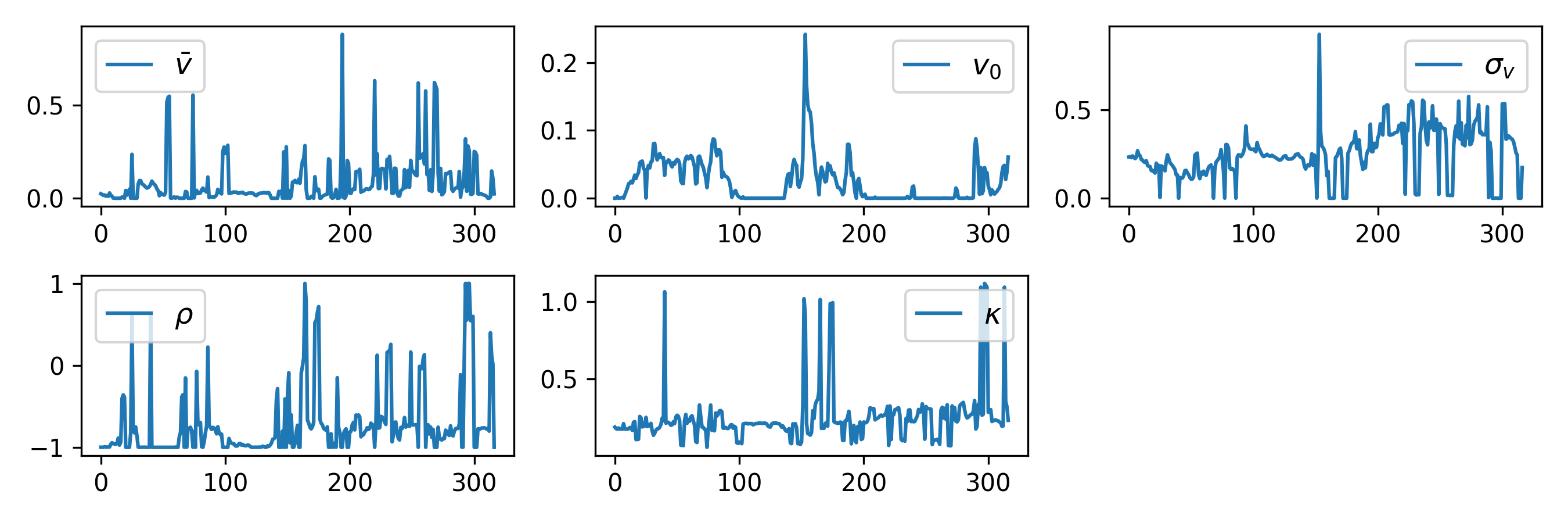}
        \caption{\footnotesize SKINN+DSNN-HSV}
    \end{subfigure}
    \hfill

    \begin{subfigure}[t]{\textwidth}
        \centering
        \includegraphics[width=0.8\textwidth]{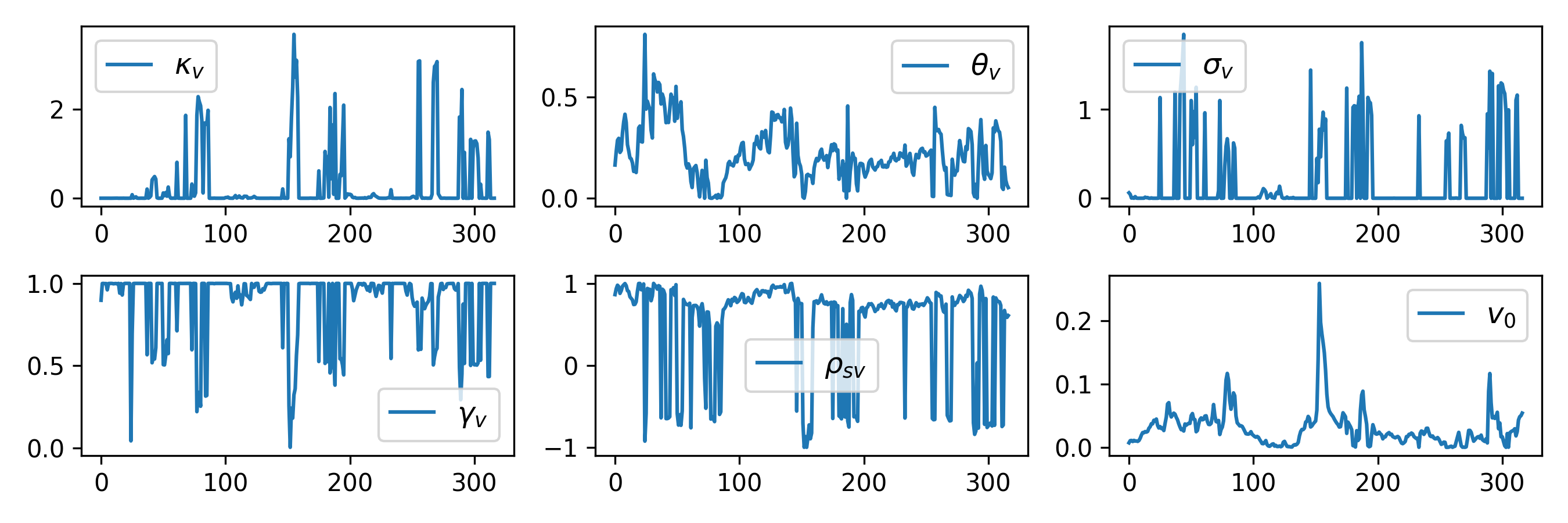}
        \caption{\footnotesize SKINN+DSNN-NASV}
    \end{subfigure}

    \begin{subfigure}[t]{\textwidth}\label{fig:skinn_hsv_hes}
        \centering
        \includegraphics[width=0.8\textwidth]{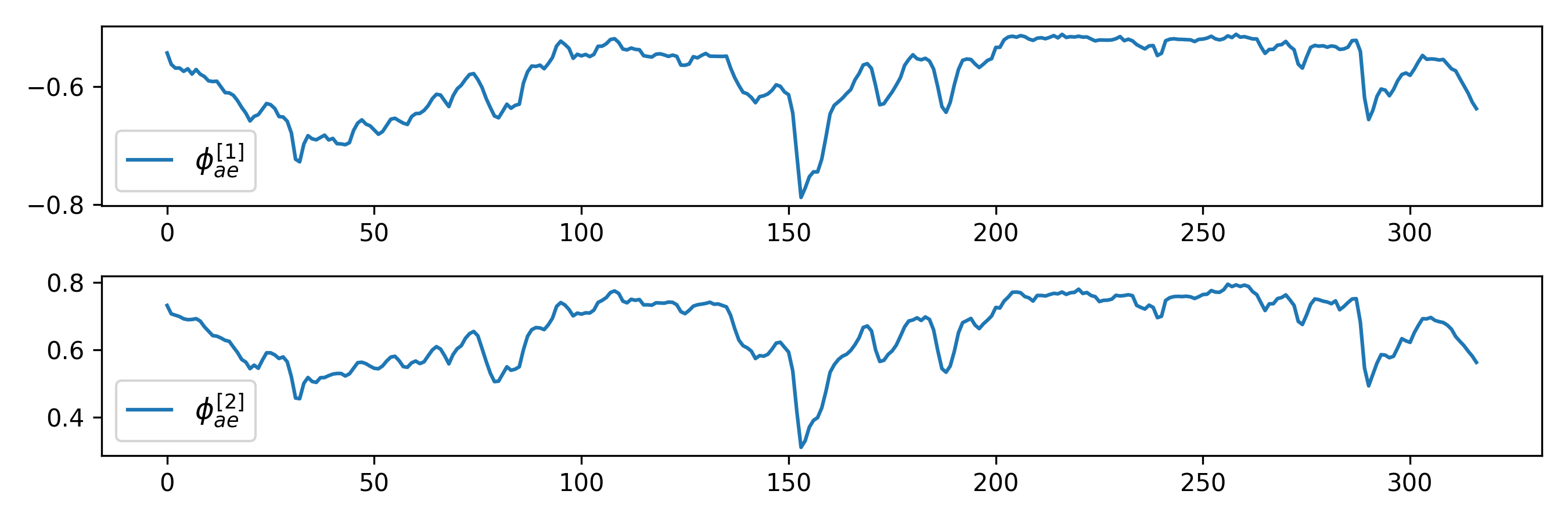}
        \caption{\footnotesize SKINN+AE-BSM}
    \end{subfigure}

    \caption{The latent structural parameters learned from the SKINNs with two SPSKRs and one NPSKR, all of which employ an NN to proxy the structured-knowledge function $g_{\phi}$. Specifically, the SKINNs are SKINN+DSNN-HSV, SKINN+DSNN-HSVJ, and SKINN+AE-BSM.}
    \label{fig:SKINNs_nn_g_phi}

\end{figure}

\begin{table}[p]
    \centering
    \caption{The predictive information contained in the latent parameters learned from SKINN+DSNN-HSV, SKINN+DSNN-NASV, and SKINN+AE-BSM, for the lagged averaged AvgVIX over one month. $\{\bar{v}, v_{0}, \sigma_{v}, \rho, \kappa\}$ is from $g_{\phi}^{\text{DSNN-HSV}}$, $\{\kappa_{v}, \theta_{v}, \sigma_{v}, \gamma_{v}, \rho_{sv}, v_{0}\}$ is from $g_{\phi}^{\text{DSNN-NASV}}$, and $\{\phi_{\text{AE}}^{(1)}, \phi_{\text{AE}}^{(2)}\}$ is from $g_{\phi}^{\text{AE-BSM}}$.}
    \label{tab:skinn_dsnn_ae_phi_vix}
    \resizebox{\textwidth}{!}{

    \begin{tabular}{@{}ccccccccccccc@{}}
    \toprule
    \textbf{DSNN-HSV} & $\bar{v}$ & $v_{0}$ & $\sigma_{v}$ & $\rho$ & $\kappa$ & -- \\
    \midrule

    $\beta_{1}$ & -7.8261\sym{**} & 167.5138\sym{***} & -12.7938\sym{***} & 1.9625\sym{*} & 6.7324\sym{***} & --\\
     & (3.6186) & (10.7461) & (3.1635) & (1.0485) & (2.5814) \\
    
    $\beta_{0}$ & 21.1113\sym{***} & 16.3179\sym{***} & 23.6611\sym{***} & 21.8788\sym{***} & 18.8241\sym{***} & --\\
     & (0.5378) & (0.4312) & (0.9039) & (0.8771) & (0.7726) & --\\
        \midrule
        Obs & 317 & 317 & 317 & 317 & 317 & --\\
    Adj.\ R\sym{2} & 1.15\% & 43.37\% & 4.63\% & 0.79\% & 1.80\% & -- \\

    \midrule

    \textbf{DSNN-NASV} & $\kappa_{v}$ & $\theta_{v}$ & $\sigma_{v}$ & $\gamma_{v}$ & $\rho_{sv}$ & $v_{0}$\\
    \midrule

    $\beta_{1}$ & 3.4702\sym{***} & -8.2771\sym{**} & 4.1704\sym{***} & -12.8520\sym{***} & -4.3081\sym{***} & 165.5858\sym{***} \\
     & (0.6080) & (3.3347) & (1.0834) & (1.9883) & (0.7073) & (11.3337) \\
    $\beta_{0}$ & 19.4509\sym{***} & 22.2949\sym{***} & 19.6744\sym{***} & 31.8214\sym{***} & 22.6236\sym{***} & 15.4147\sym{***} \\
     & (0.4653) & (0.8610) & (0.4870) & (1.8069) & (0.5546) & (0.4911) \\

        \midrule
        Obs & 317 & 317 & 317 & 317 & 317 & 317 \\
    Adj.\ R\sym{2} & 9.08\% & 1.61\% & 4.19\% & 11.43\% & 10.25\% & 40.20\% \\

    \midrule

    \textbf{AE-BSM} & $\phi_{\text{AE}}^{(1)}$ & $\phi_{\text{AE}}^{(2)}$ & -- & -- & -- & --\\
    \midrule

    $\beta_{1}$ & -89.4417\sym{***} & -59.3616\sym{***} & -- & -- & -- & --\\
     & (5.9282) & (3.4113) & -- & -- & -- & --\\
    $\beta_{0}$ & -31.5612\sym{***} & 59.4865\sym{***} & -- & -- & -- & --\\
     & (3.4657) & (2.2653) & -- & -- & -- & --\\
    
    \midrule
    Obs & 317 & 317 & -- & -- & -- & --\\
    Adj.\ R\sym{2} & 41.77\% & 48.85\% & -- & -- & -- & --\\

    \bottomrule
    \end{tabular}
    }
\end{table}

\section{An Extended Application: SKINNs Asset Pricing Models}\label{sec_skinns_ap}
Besides option pricing tasks, SKINNs can be conveniently extended to many others. Generally speaking, the SKINNs framework is capable of solving any two-stage estimation problems. Traditionally, these two-stage problems require a properly-trained prediction device at the first stage to predict the target variables, which are passed as the parameters of the estimation problems at the second stage. Such a ``predict-then-optimize'' paradigm can yield suboptimal solutions, as the prediction device at the first stage receives no information about the utility of the downstream problem and can easily overfit to the provided training data. This problem is particularly acute in asset pricing tasks, where asset managers typically solve downstream mean-variance optimizations based on the learned asset pricing models from only data to manage their portfolios. Methods such as ``SPO+'' loss and differentiable optimization layers are proposed to force the learning process of the predictive device to be aware of the downstream estimation/optimization utility \citep[see,][]{amos2017optnet,agrawal2019differentiable,elmachtoub2022smart,wang2026machine}. These methods require iterative solving of the optimization problem for the optimal decision variables during training, which is computationally expensive, especially when the decision variables become high-dimensional. The SKINNs framework nests them by casting the decision variables to solve as the learnable latent parameters $\phi$, in addition to the nuisance NN parameters.

To see how the SKINNs framework integrates downstream problems, we inform an NN asset pricing model with mean-variance optimization:
\begin{align}
    \mathbf{w}^{*}:=\underset{\mathbf{w}}{\arg\min}&\;\;-\mathbf{w}^{\top}\hat{\mathbf{r}} + \eta\,\mathbf{w}^{\top}\bm{\Sigma}\,\mathbf{w} \label{eq_mv_objective}\\
                s.t &\;\;\mathbf{w}^{\top}\mathbf{1} = \mathbf{1}, \label{eq_mv_cons_equality}\\
                    \;\;&\mathbf{l}\leq\mathbf{w}\leq\mathbf{u}, \label{eq_mv_cons_inequality}
\end{align}
where $\mathbf{w}$ is the weight vector of each asset; $\hat{\mathbf{r}}=f_{\theta}(\mathbf{X})$ denotes the return predictions by NN; $\mathbf{l},\mathbf{u}$ are the minimum and maximum weight vectors, respectively; $\eta$ denotes the scalar of investor's risk aversion; $\bm{\Sigma}$ denotes the covariance matrix calculated using historical returns. We cast the mean-variance objective in Equation (\ref{eq_mv_objective}) as a structured-knowledge loss:
\begin{equation}
    \mathcal{L}_{\text{SK}}\left(\theta,\phi\in\{\mathbf{w}^{\star}\};\mathbf{X}\right) = -\;(\mathbf{w}^{\star})^{\top}\;f_{\theta}(\mathbf{X}) + \eta\;(\mathbf{w}^{\star})^{\top}\;\bm{\Sigma}\;\mathbf{w}^{\star}.
\end{equation}
The equality and inequality constraints in Equation (\ref{eq_mv_cons_equality})-(\ref{eq_mv_cons_inequality}) can be enforced in SKINNs using clamping functions.

We randomly select 50 stocks from the S\&P 500 constituent stocks. We employ their daily stock data from 2010 to 2023. The asset pricing models are learned using the SKINN framework with one-year data, and tested over the subsequent one-month data. The training and testing procedure is rolled-forward with a one-month window, generating 143 training and testing periods in total. Since our goal is to demonstrate how the asset pricing SKINNs encoded with mean-variance optimizations improve upon plain-vanilla NN, rather than identifying the correct return predictors, we choose the standardized OHLC prices as the daily stock characteristics.\footnote{OHLC represents the open price, the highest price, the lowest price, and the close price, respectively.} At each trading day within the testing periods, we sort 50 stocks into deciles in descending order of the predicted returns by SKINNs. A long-short portfolio can be constructed by buying the top-decile and selling the bottom-decile. We follow \cite{wang2026machine} to set the minimum weight $\mathbf{l}=\mathbf{0}$ and the maximum weight $\mathbf{u}=0.2\times\mathbf{1}$, which reflects the long-only and equal-weight constraints that asset managers often face.

\begin{table}[htbp]
    \centering
    \caption{The comparison of the SKINN asset pricing models and the plain-vanilla NN asset pricing models. All portfolio performance metrics are evaluated with the out-of-sample stock data in the testing periods. We train asset pricing models using 5 different architectures (1-5 hidden layers with 32 neurons). The results are averaged across 20 different random initializations of an NN. Group H denotes the top-decile portfolio, and Group L denotes the bottom-decile.}
    \label{tab:skinn_ap_models}
    \begin{tabular}{@{}lcccccccccccc@{}}
    \toprule
    \textbf{Model} & & \multicolumn{3}{c}{\textbf{NN}} & & \multicolumn{3}{c}{\textbf{SKINN ($\eta=0$)}} & & \multicolumn{3}{c}{\textbf{SKINN ($\eta=10$)}} \\
    \cmidrule(lr){3-5} \cmidrule(lr){7-9} \cmidrule(lr){11-13}
    & Group & Avg \% & SD \% & SR & & Avg \% & SD \% & SR & & Avg \% & SD \% & SR \\
    \midrule
    \multirow{2}{*}{1-32} & H & 60.73 & 130.07 & 0.47 & & 68.22 & 109.78 & 0.62 & & 70.83 & 102.91 & 0.69 \\
                          & L & 59.29 & 136.54 & 0.43 & & 54.78 & 133.41 & 0.41 & & 57.77 & 134.48 & 0.43 \\
    \midrule
    \multirow{2}{*}{2-32} & H & 57.01 & 127.73 & 0.45 & & 71.86 & 119.78 & 0.60 & & 65.32 & 99.71  & 0.66 \\
                          & L & 72.12 & 145.66 & 0.50 & & 70.98 & 141.31 & 0.50 & & 57.19 & 140.62 & 0.41 \\
    \midrule
    \multirow{2}{*}{3-32} & H & 56.57 & 127.98 & 0.44 & & 68.94 & 122.60 & 0.56 & & 59.45 & 100.10 & 0.59 \\
                          & L & 76.98 & 132.41 & 0.58 & & 71.67 & 132.81 & 0.54 & & 60.34 & 134.04 & 0.45 \\
    \midrule
    \multirow{2}{*}{4-32} & H & 57.71 & 124.46 & 0.46 & & 67.11 & 118.09 & 0.57 & & 61.96 & 97.56  & 0.64 \\
                          & L & 65.51 & 136.01 & 0.48 & & 68.98 & 133.11 & 0.52 & & 54.35 & 134.98 & 0.40 \\
    \midrule
    \multirow{2}{*}{5-32} & H & 56.54 & 125.43 & 0.45 & & 70.35 & 122.43 & 0.57 & & 65.48 & 100.00 & 0.65 \\
                          & L & 77.84 & 134.57 & 0.58 & & 73.08 & 132.93 & 0.55 & & 70.90 & 135.16 & 0.52 \\
    \bottomrule
    \end{tabular}
\end{table}

Table (\ref{tab:skinn_ap_models}) reports the performance of the learned asset pricing models. We focus on Group H (top-decile portfolios), as in real-world trading, the top-decile stocks with positive weights are preferred by asset managers since they are within the feasible region even when there is a short-selling restriction. By embedding the mean-variance optimization through the structured-knowledge loss $\mathcal{L}_{\text{SK}}(\theta,\phi\in\{\mathbf{w}^{\star}\};\mathbf{X})$, SKINNs are capable of guiding the learning of the asset pricing model with the utility-based information from the downstream mean-variance optimization. Plain-vanilla NNs do not receive such utility-based information during the training loop, and hence do not particularly optimize the top-decile portfolios formed according to their out-of-sample return predictions. We consider two scenarios: (i) $\eta=0$, the risk-insensitive case where the optimization in Equation (\ref{eq_mv_objective}) only guides NN to favour stocks with high predicted return; (ii) $\eta=10$, the strongly risk-averse case where the optimization guides NN to favour stocks with high predicted return but low risk. Both SKINNs demonstrate consistently improved annualized average daily return, as well as the Sharpe ratio, on the Group H across all 5 network architectures, compared with NNs. More importantly, by incorporating a higher risk-aversion coefficient, $\eta=10$, SKINNs successfully reduce the return volatilities of Group H further, yielding the highest Sharpe ratios.

\section{Conclusion}\label{sec_conclude}
Despite recent successes of AI and machine learning algorithms in economics, they often lack the transparency, interpretability, and clarity on economic intuition and mechanism that economists often require of conventional econometric and theory models. We introduce a Structured-Knowledge-Informed Neural Networks (SKINNs) framework for embedding existing domain knowledge, in flexible formats (analytical models, deep surrogate models, and general implicit/non-parametric models), into deep-learning-based data-driven analyses. Improving upon and going beyond existing methods of encoding prior knowledge into neural networks, including transfer learning and PINNs, SKINNs serve as a general tool for scientific machine learning, or AI for science or social science, especially when the latent parameters for structured-knowledge are high-dimensional and hard to estimate or interpret using conventional methods. 

On the theoretical side, we provide a unified econometric foundation for SKINNs. Beyond consistency and asymptotic normality, we establish identification of structural parameters under joint flexibility, derive generalization and target-risk bounds under distributional shift in a convex proxy, and characterize the regularization parameter as a restricted-optimal weighting choice within a GMM interpretation. Under orthogonal moment conditions, the structural parameter estimator achieves the optimal asymptotic variance relative to the imposed moment restrictions. These results demonstrate that integrating structured knowledge into neural networks is not merely a heuristic regularization device, but defines a statistically well-behaved estimator with formal inferential guarantees.

For illustration, we apply SKINNs to option pricing, a setting featuring structured theoretical, conceptual, or evidence-based knowledge in a wide variety of formats. SKINNs statistically improve out-of-sample pricing performance and hedging capability for the S\&P 500 index options, compared with the plain-vanilla neural network, neural network with the model-free constraints, standard structural models, and even transfer learning informed by the Heston model. The outperformance of SKINNs is greater when pricing options in periods more distant from the training sample period and in highly volatile markets, where a pure data-driven approach tends to overfit or ill-adjust for distributional shifts. We also find that the learned latent parameters for structured-knowledge representations are not merely nuisance parameters; they carry economic meaning with improved stability over time, over the estimations from standard calibration procedures. Moreover, the dimensionality of the latent parameters can efficiently scale under SKINNs, something prohibitive in other approaches, such as transfer learning.

\newpage

\begin{appendices}

    \section{The Wilcoxon Signed-Rank Test}\label{sec_oos_wilcoxon}

Besides the \cite{diebold2002comparing} test results we report in Section (\ref{sec:oos_performance_dm_test}), we also perform the \cite{wilcoxon1945individual} signed-rank test to compare the performance difference in pricing and hedging among the considered models, for the robustness check. While the Diebold-Mariano test adjusts for the possible autocorrelation in the series of error differentials $d_{j}=\{e^{1}_{j}-e^{2}_{j}\}_{j=1}^{317}$, the non-parametric Wilcoxon signed-rank test does not account for this. We report the results of the Wilcoxon signed-rank test in this section. The results of using two tests are close. Since the test statistic of the Wilcoxon signed-rank test is always a positive integer, we place the asterisks on the right-top of a number to indicate that the column model outperforms the row model, and on the left-top of a number to indicate that the row model outperforms the column model. Table (\ref{tab:oos_pricing_wilcoxon_shorter}) and Table (\ref{tab:oos_pricing_wilcoxon_longer}) report the results for the out-of-sample pricing performance comparisons. Table (\ref{tab:oos_hedging_wilcoxon_shorter}) and Table (\ref{tab:oos_hedging_wilcoxon_longer}) report the results for the out-of-sample hedging performance comparisons.

\begin{table}[htbp]

\caption{The \cite{wilcoxon1945individual} test results for the out-of-sample option pricing accuracy comparisons in shorter prediction horizons. We square all the pricing errors to penalize large errors.}
\label{tab:oos_pricing_wilcoxon_shorter}

    \resizebox{\textwidth}{!}{
    \begin{tabular}{@{}ccccccccccccc@{}}
    \toprule
    \textbf{Model}&\multicolumn{4}{c}{\textbf{Structural models}} & \multicolumn{4}{c}{\textbf{Benchmark NNs}} \\
        \cmidrule(lr){2-5}\cmidrule(lr){6-9}
    \textbf{Panel (A)} & \multirow{2}{*}{BSM} & \multirow{2}{*}{ABSM} & \multirow{2}{*}{HSV} & \multirow{2}{*}{HSVJ} & \multirow{2}{*}{NN} & NN & NN & PINN\\ 
     &  &  &  &  &  & +Bnd & +Shape & +BSM\\ 
    \midrule
    BSM          & -- & 22.48** & 23.43   & 5.68*** & 17.71*** & 14.13*** & 25.18    & 25.34    \\ 
    ABSM        & -- & --      & 22.58* & 5.20*** & 17.15*** & 13.82*** & 24.53    & 25.12    \\ 
    HSV         & -- & --      & --      & 2.94*** & 19.13*** & 16.19*** & 25.89    & 26.53    \\ 
    HSVJ     & -- & --      & --      & --      & 25.74    & 23.54    & 35.30*** & 35.12*** \\ 
    NN          & -- & --      & --      & --      & --       & 20.75*** & 35.23*** & 35.50*** \\ 
    NN+Bnd      & -- & --      & --      & --      & --       & --       & 39.29*** & 38.64*** \\ 
    NN+Shape       & -- & --      & --      & --      & --       & --       & --       & 26.41    \\ 
    PINN+BSM & -- & --      & --      & --      & --       & --       & --       & --       \\

    \midrule
    \textbf{Model} & \multicolumn{4}{c}{\textbf{PSKR}} & \multicolumn{2}{c}{\textbf{SPSKR}} & \multicolumn{2}{c}{\textbf{NPSKR}} \\
    \cmidrule(lr){2-5}\cmidrule(lr){6-7}\cmidrule(lr){8-9}
    \textbf{Panel (B)} & SKINN & SKINN & SKINN & SKINN & SKINN & SKINN & SKINN & SKINN \\
                       & +BSM & +ABSM & +HSV & +HSVJ & +DSNN-HSV & +DSNN-NASV & +MOPA & +AE-BSM \\
    \midrule
    BSM          & 8.60*** & 7.35*** & 8.40*** & 7.25*** & 18.79*** & 11.51*** & 11.18*** & 16.19*** \\ 
    ABSM        & 9.21*** & 7.53*** & 8.26*** & 7.21*** & 18.72*** & 11.91*** & 10.92*** & 16.16*** \\ 
    HSV         & 13.13*** & 12.08*** & 10.07*** & 8.68*** & 19.84*** & 14.02*** & 12.35*** & 17.76*** \\ 
    HSVJ     & 24.71    & 23.36    & 20.17*** & 18.58*** & 30.37*** & 25.80    & 21.24*** & 29.03*** \\ 
    NN          & 25.90    & 24.46    & 22.80* & 21.48** & 28.91** & 26.28    & 21.60** & 28.38** \\ 
    NN+Bnd      & 28.98** & 27.45* & 23.50    & 22.81* & 31.38*** & 28.40** & 22.96* & 32.26*** \\ 
    NN+Shape       & 14.64*** & 13.50*** & 10.88*** & 10.56*** & 19.76*** & 14.02*** & 10.96*** & 15.74*** \\ 
    PINN+BSM & 14.40*** & 12.80*** & 11.27*** & 10.15*** & 19.13*** & 14.32*** & 9.91*** & 15.99*** \\ 
    
    \bottomrule
    \end{tabular}
    }
    
\end{table}

\begin{table}[htbp]

\caption{The \cite{wilcoxon1945individual} test results for the out-of-sample option pricing accuracy comparisons in longer prediction horizons. We square all the pricing errors to penalize large errors.}
\label{tab:oos_pricing_wilcoxon_longer}

    \resizebox{\textwidth}{!}{
    \begin{tabular}{@{}ccccccccccccc@{}}
    \toprule
    \textbf{Model}&\multicolumn{4}{c}{\textbf{Structural models}} & \multicolumn{4}{c}{\textbf{Benchmark NNs}} \\
        \cmidrule(lr){2-5}\cmidrule(lr){6-9}
    \textbf{Panel (A)} & \multirow{2}{*}{BSM} & \multirow{2}{*}{ABSM} & \multirow{2}{*}{HSV} & \multirow{2}{*}{HSVJ} & \multirow{2}{*}{NN} & NN & NN & PINN\\ 
     &  &  &  &  &  & +Bnd & +Shape & +BSM\\ 
    \midrule
    BSM          & -- & 25.23 & 25.75 & 9.68*** & 27.71* & 22.11** & 27.52* & 28.20** \\ 
    ABSM        & -- & --    & 24.24 & 8.48*** & 27.44* & 21.81** & 27.11    & 27.78* \\ 
    HSV         & -- & --    & --    & 3.79*** & 27.75* & 22.40** & 27.20    & 28.15** \\ 
    HSVJ     & -- & --    & --    & --      & 32.41*** & 27.46* & 33.16*** & 34.08*** \\ 
    NN          & -- & --    & --    & --      & --        & 19.16*** & 26.05    & 26.34    \\ 
    NN+Bnd      & -- & --    & --    & --      & --        & --       & 31.56*** & 31.65*** \\ 
    NN+Shape       & -- & --    & --    & --      & --        & --       & --       & 27.27    \\ 
    PINN+BSM & -- & --    & --    & --      & --        & --       & --       & --       \\

    \midrule
    \textbf{Model} & \multicolumn{4}{c}{\textbf{PSKR}} & \multicolumn{2}{c}{\textbf{SPSKR}} & \multicolumn{2}{c}{\textbf{NPSKR}} \\
    \cmidrule(lr){2-5}\cmidrule(lr){6-7}\cmidrule(lr){8-9}
    \textbf{Panel (B)} & SKINN & SKINN & SKINN & SKINN & SKINN & SKINN & SKINN & SKINN \\
                       & +BSM & +ABSM & +HSV & +HSVJ & +DSNN-HSV & +DSNN-NASV & +MOPA & +AE-BSM \\
    \midrule
    BSM          & 15.94*** & 16.10*** & 15.62*** & 15.30*** & 23.95     & 19.26*** & 18.37*** & 25.17    \\ 
    ABSM        & 16.02*** & 15.78*** & 15.32*** & 14.97*** & 23.22     & 19.05*** & 17.93*** & 24.55    \\ 
    HSV         & 16.26*** & 16.23*** & 15.72*** & 15.54*** & 23.85     & 18.93*** & 18.27*** & 25.39    \\ 
    HSVJ     & 25.28    & 25.00    & 23.89    & 23.23    & 30.64*** & 27.94** & 24.84    & 32.38*** \\ 
    NN          & 19.65*** & 18.28*** & 17.07*** & 16.21*** & 21.68** & 19.72*** & 15.96*** & 22.69* \\ 
    NN+Bnd      & 22.23** & 22.03** & 19.91*** & 18.16*** & 25.81     & 22.69* & 19.03*** & 27.66* \\ 
    NN+Shape       & 15.96*** & 15.66*** & 13.88*** & 13.90*** & 19.95*** & 16.67*** & 13.87*** & 19.55*** \\ 
    PINN+BSM & 14.43*** & 14.56*** & 13.48*** & 12.91*** & 19.24*** & 16.02*** & 12.88*** & 19.35*** \\ 
    
    \bottomrule
    \end{tabular}
    }
    
\end{table}

\begin{table}[htbp]
\caption{The \cite{wilcoxon1945individual} test results for the out-of-sample option hedging accuracy comparisons in shorter prediction horizons. We square all the hedging errors to penalize large errors.}
\label{tab:oos_hedging_wilcoxon_shorter}

    \resizebox{\textwidth}{!}{
    \begin{tabular}{@{}ccccccccccccc@{}}
    \toprule
    \textbf{Model}&\multicolumn{4}{c}{\textbf{Structural models}} & \multicolumn{4}{c}{\textbf{Benchmark NNs}} \\
        \cmidrule(lr){2-5}\cmidrule(lr){6-9}
    \textbf{Panel (A)} & \multirow{2}{*}{BSM} & \multirow{2}{*}{ABSM} & \multirow{2}{*}{HSV} & \multirow{2}{*}{HSVJ} & \multirow{2}{*}{NN} & NN & NN & PINN\\ 
     &  &  &  &  &  & +Bnd & +Shape & +BSM\\ 
    \midrule
    BSM          & -- & 5.16*** & 46.50*** & 47.13*** & 22.27** & 21.40*** & 25.08    & 24.21    \\ 
    ABSM        & -- & --      & 46.79*** & 47.31*** & 23.79    & 23.39    & 27.21    & 25.96    \\ 
    HSV         & -- & --      & --       & 20.19** & 9.80*** & 7.81*** & 12.32*** & 10.56*** \\ 
    HSVJ     & -- & --      & --       & --       & 9.59*** & 8.12*** & 12.60*** & 10.81*** \\ 
    NN          & -- & --      & --       & --       & --       & 24.41    & 30.24*** & 27.59* \\ 
    NN+Bnd      & -- & --      & --       & --       & --       & --       & 32.07*** & 28.85** \\ 
    NN+Shape       & -- & --      & --       & --       & --       & --       & --       & 18.98*** \\ 
    PINN+BSM & -- & --      & --       & --       & --       & --       & --       & --       \\

    \midrule
    \textbf{Model} & \multicolumn{4}{c}{\textbf{PSKR}} & \multicolumn{2}{c}{\textbf{SPSKR}} & \multicolumn{2}{c}{\textbf{NPSKR}} \\
    \cmidrule(lr){2-5}\cmidrule(lr){6-7}\cmidrule(lr){8-9}
    \textbf{Panel (B)} & SKINN & SKINN & SKINN & SKINN & SKINN & SKINN & SKINN & SKINN \\
                       & +BSM & +ABSM & +HSV & +HSVJ & +DSNN-HSV & +DSNN-NASV & +MOPA & +AE-BSM \\
    \midrule
    BSM          & 4.84*** & 4.66*** & 15.47*** & 14.85*** & 16.36*** & 16.69*** & 17.03*** & 14.50*** \\ 
    ABSM        & 7.49*** & 6.30*** & 18.35*** & 18.52*** & 19.29*** & 20.05*** & 19.88*** & 17.28*** \\ 
    HSV         & 1.41*** & 1.41*** & 2.42*** & 2.35*** & 3.63*** & 3.62*** & 2.90*** & 3.78*** \\ 
    HSVJ     & 1.35*** & 1.33*** & 2.42*** & 2.14*** & 3.98*** & 3.81*** & 3.06*** & 3.67*** \\ 
    NN          & 11.61*** & 11.50*** & 20.50*** & 19.68*** & 20.98*** & 19.82*** & 21.25*** & 18.91*** \\ 
    NN+Bnd      & 9.61*** & 9.27*** & 20.46*** & 19.87*** & 20.21*** & 21.12*** & 21.07*** & 17.88*** \\ 
    NN+Shape       & 6.92*** & 7.12*** & 15.35*** & 14.76*** & 16.16*** & 16.99*** & 15.39*** & 14.01*** \\ 
    PINN+BSM & 9.39*** & 9.29*** & 17.70*** & 17.80*** & 18.38*** & 19.21*** & 16.76*** & 16.61*** \\ 
    
    \bottomrule
    \end{tabular}
    }
    
\end{table}

\begin{table}[htbp]
\caption{The \cite{wilcoxon1945individual} test results for the out-of-sample option hedging accuracy comparisons in longer prediction horizons. We square all the hedging errors to penalize large errors.}
\label{tab:oos_hedging_wilcoxon_longer}

    \resizebox{\textwidth}{!}{
    \begin{tabular}{@{}ccccccccccccc@{}}
    \toprule
    \textbf{Model}&\multicolumn{4}{c}{\textbf{Structural models}} & \multicolumn{4}{c}{\textbf{Benchmark NNs}} \\
        \cmidrule(lr){2-5}\cmidrule(lr){6-9}
    \textbf{Panel (A)} & \multirow{2}{*}{BSM} & \multirow{2}{*}{ABSM} & \multirow{2}{*}{HSV} & \multirow{2}{*}{HSVJ} & \multirow{2}{*}{NN} & NN & NN & PINN\\ 
     &  &  &  &  &  & +Bnd & +Shape & +BSM\\ 
    \midrule
    BSM          & -- & 6.19*** & 46.34*** & 46.83*** & 22.52* & 21.78** & 24.95    & 24.75    \\ 
    ABSM        & -- & --      & 46.59*** & 47.12*** & 23.76    & 23.64    & 27.32* & 26.56    \\ 
    HSV         & -- & --      & --       & 20.55** & 12.36*** & 10.91*** & 14.20*** & 12.84*** \\ 
    HSVJ     & -- & --      & --       & --       & 12.51*** & 11.17*** & 14.49*** & 13.24*** \\ 
    NN          & -- & --      & --       & --       & --       & 23.63    & 29.38*** & 28.98** \\ 
    NN+Bnd      & -- & --      & --       & --       & --       & --       & 32.50*** & 31.19*** \\ 
    NN+Shape       & -- & --      & --       & --       & --       & --       & --       & 20.48*** \\ 
    PINN+BSM & -- & --      & --       & --       & --       & --       & --       & --       \\

    \midrule
    \textbf{Model} & \multicolumn{4}{c}{\textbf{PSKR}} & \multicolumn{2}{c}{\textbf{SPSKR}} & \multicolumn{2}{c}{\textbf{NPSKR}} \\
    \cmidrule(lr){2-5}\cmidrule(lr){6-7}\cmidrule(lr){8-9}
    \textbf{Panel (B)} & SKINN & SKINN & SKINN & SKINN & SKINN & SKINN & SKINN & SKINN \\
                       & +BSM & +ABSM & +HSV & +HSVJ & +DSNN-HSV & +DSNN-NASV & +MOPA & +AE-BSM \\
    \midrule
    BSM          & 7.08*** & 6.07*** & 16.40*** & 16.85*** & 18.39*** & 17.97*** & 18.96*** & 17.60*** \\ 
    ABSM        & 9.79*** & 8.27*** & 19.07*** & 19.67*** & 20.70*** & 20.71*** & 21.65** & 20.35*** \\ 
    HSV         & 2.23*** & 2.19*** & 3.46*** & 3.12*** & 4.77*** & 4.77*** & 5.70*** & 5.62*** \\ 
    HSVJ     & 2.25*** & 2.18*** & 3.78*** & 3.12*** & 5.02*** & 5.00*** & 5.98*** & 5.93*** \\ 
    NN          & 12.25*** & 12.10*** & 20.85*** & 20.02*** & 21.87** & 21.18*** & 20.05*** & 20.37*** \\ 
    NN+Bnd      & 9.94*** & 9.11*** & 21.32*** & 20.13*** & 22.27** & 22.03** & 20.45*** & 19.93*** \\ 
    NN+Shape       & 7.82*** & 7.62*** & 15.41*** & 15.19*** & 16.14*** & 16.59*** & 15.83*** & 15.74*** \\ 
    PINN+BSM & 8.84*** & 9.29*** & 17.80*** & 17.35*** & 18.23*** & 17.83*** & 16.86*** & 16.30*** \\ 
    
    \bottomrule
    \end{tabular}
    }
    
\end{table}
    \clearpage

    \section{High-dimensional SKINNs Additional Results}\label{sec_more_high_dim_sk_specs}

In this section, we present the option pricing performance of two high-dimensional SKINNs—SKINN+SABR and SKINN+AE-MIX, which are in our model list of Table (\ref{nn_model_list}). These two SKINNs utilize PSKR and NPSKR, respectively, and both are characterized by high latent parameter dimensionality (722-dimensional and 50-dimensional). SKINN+AE-MIX employs a machine-derived NPSKR that combines the BSM, HSV, and HSVJ models.

\begin{table}[htbp]

\centering
\caption{The Diebold-Mariano test of the out-of-sample option pricing performance, measured in terms of RMSE, of the two SKINNs. We square all the pricing errors to penalize large errors.}
\label{}

    \resizebox{0.9\textwidth}{!}{
    \begin{tabular}{@{}lcccc@{}}
    \toprule
     & \multicolumn{2}{c}{\textbf{Shorter Prediction Horizons}} & \multicolumn{2}{c}{\textbf{Longer Prediction Horizons}} \\
    \cmidrule(lr){2-3} \cmidrule(l){4-5}
    \textbf{Model} & SKINN+SABR & SKINN+AE-MIX & SKINN+SABR & SKINN+AE-MIX \\
    \midrule
    BSM          & -1.22  & -2.74*** & -1.59* & 0.51     \\ 
    ABSM        & -1.36* & -2.93*** & -1.81** & 0.37     \\ 
    HSV         & -0.96  & -2.58*** & -1.76* & 0.45     \\ 
    HSVJ     & -0.17  & -1.51* & -0.81     & 1.03     \\ 
    NN          & 0.30   & -1.65** & -3.68*** & -3.23*** \\ 
    NN+Bnd      & 0.88   & 0.27    & -2.53*** & -1.99** \\ 
    NN+Shape       & -1.35* & -4.20*** & -3.33*** & -2.71*** \\ 
    PINN+BSM & -1.37* & -3.77*** & -3.24*** & -2.50*** \\
    \bottomrule
    \end{tabular}
    }

\end{table}

\begin{table}[htbp]
\centering
\caption{The marginal outperformance of SKINN against the plain-vanilla NN given a one-unit increase in AvgVIX. We consider different volatility regimes from the 317 longer prediction horizons.}
\resizebox{0.9\textwidth}{!}{
\begin{tabular}{lcccccc}
\toprule
 & \multicolumn{2}{c}{\textbf{All}} & & \multicolumn{2}{c}{\textbf{Low Volatility}} \\
\cmidrule(lr){2-3} \cmidrule(lr){5-6}
\textbf{Model} & SKINN+SABR & SKINN+AE-MIX & & SKINN+SABR & SKINN+AE-MIX \\
\midrule
AvgVIX      & -0.0835*** & -0.0821*** & & -0.2368    & -0.0765    \\
            & (0.019)    & (0.012)    & & (0.161)    & (0.118)    \\
Constant    & 0.0096** & 0.0139*** & & 0.0259     & 0.0097     \\
            & (0.004)    & (0.003)    & & (0.020)    & (0.015)    \\
\midrule
Obs         & 317        & 317        & & 64         & 64         \\
Adj. R\sym{2}  & 5.2\%      & 12.0\%     & & 1.8\%      & -0.9\%     \\
\midrule
 & \multicolumn{2}{c}{\textbf{Medium Volatility}} & & \multicolumn{2}{c}{\textbf{High Volatility}} \\
\cmidrule(lr){2-3} \cmidrule(lr){5-6}
\textbf{Model} & SKINN+SABR & SKINN+AE-MIX & & SKINN+SABR & SKINN+AE-MIX \\
\midrule
AvgVIX      & -0.1242*** & -0.0955*** & & -0.1712** & -0.1666*** \\
            & (0.047)    & (0.031)    & & (0.070)    & (0.042)    \\
Constant    & 0.0167* & 0.0166*** & & 0.0425* & 0.0444*** \\
            & (0.009)    & (0.006)    & & (0.024)    & (0.014)    \\
\midrule
Obs         & 189        & 189        & & 64         & 64         \\
Adj. R\sym{2}  & 3.1\%      & 4.3\%      & & 7.3\%      & 19.3\%     \\
\bottomrule
\end{tabular}
}
\end{table}
    
    \section{Theoretical Foundations of SKINNs}\label{sec_theoretical}

We begin by formally stating the SKINNs optimization problem. Consider a probability space $(\Omega, \mathcal{F}, \mathbb{P})$ on which all random variables are defined. Let $(\mathbf{X}, \mathbf{y})$ denote a random vector where $\mathbf{X} \in \mathcal{X} \subseteq \mathbb{R}^d$ represents observable input features and $\mathbf{y} \in \mathbb{R}$ represents the target variable. The data-driven component is a deep neural network $f_\theta: \mathcal{X} \to \mathbb{R}$, parameterized by $\theta \in \Theta \subset \mathbb{R}^{p}$, where $\Theta$ is a compact parameter space and $p$ represents the potentially high-dimensional network architecture (all weights and biases across layers). The structured-knowledge component is a function $g_\phi: \mathcal{X}^{\text{SK}} \to \mathbb{R}$, where $\mathcal{X}^{\text{SK}} \subseteq \mathcal{X}$ denotes the subspace of theoretically relevant inputs and $\phi \in \Phi \subset \mathbb{R}^{q}$ represents the latent parameters of the embedded economic model, with $\Phi$ being a compact parameter space.

\subsection{Regularity Assumptions}\label{sec:assumption}

The following assumptions underpin the consistency and asymptotic normality results established in Sections~\ref{sec:general_theory}--\ref{sec:infer_practical}. They are standard in the M-estimation literature \citep{newey1994large, van1996weak} and are satisfied in most practical implementations of the SKINNs framework. When the structured-knowledge component $g_{\phi}$ is misspecified, the minimizer $\phi^{*}$ should be interpreted as a pseudo-true parameter, that is, the value that minimizes the population objective $\mathcal{L}(\theta,\phi)$; in this case, the asymptotic results below characterize convergence to this pseudo-true value, as is standard in M-estimation under misspecification \citep{white1982maximum}.

\begin{assumption}[Identification]
\label{assump:identification}
The composite objective $\mathcal{L}(\theta, \phi)$ has a unique minimizer $(\theta^*, \phi^*)$ in the interior of $\Theta \times \Phi$. Moreover, the minimum is well-separated: for any $\epsilon > 0$, there exists $\delta > 0$ such that 
\[
\inf_{\|(\theta, \phi) - (\theta^*, \phi^*)\| \geq \epsilon} \mathcal{L}(\theta, \phi) > \mathcal{L}(\theta^*, \phi^*) + \delta.
\]
\end{assumption}

Lemma~\ref{lem:profile-id} and Proposition~\ref{prop:closedform-id} in Section~\ref{sec:skinn_optim} provide verifiable primitive conditions under which this assumption holds. In particular, Proposition~\ref{prop:closedform-id} shows that, under squared-error loss with a sufficiently rich function class, identification of $\phi$ reduces to uniqueness of the minimizer of $\mathbb{E}\!\left[\left(\mathbb{E}[\mathbf{y}\mid\mathbf{X}_{\text{obs}}] - g_{\phi}(\mathbf{X}_{\text{obs}}^{\text{SK}})\right)^2\right]$ over $\Phi$, a condition that can be checked on a case-by-case basis for specific structured-knowledge representations $g_{\phi}$.

\begin{assumption}[Compactness]
\label{assump:compactness}
The parameter spaces $\Theta \subset \mathbb{R}^{p}$ and $\Phi \subset \mathbb{R}^{q}$ are compact.
\end{assumption}

Compactness is a technical condition that ensures the existence of minimizers and facilitates uniform convergence arguments. In practice, it can be enforced through bounded parameter constraints or weight clipping. We note that the population identification analysis in Proposition~\ref{prop:closedform-id} works with an unrestricted function class to characterize the profiled objective, whereas the finite-sample asymptotic theory here requires compactness to control the complexity of the parameter space. This two-stage reasoning---population identification under richness, then finite-sample convergence under compactness---is standard in the sieve estimation literature \citep{chen2007large}.

\begin{assumption}[Continuity]
\label{assump:continuity}
The loss function $\ell(f, \mathbf{y})$ is continuous in $f$ for all $\mathbf{y}$, and the neural network $f_\theta(\mathbf{X})$ and structured-knowledge function $g_\phi(\mathbf{X}^{\text{SK}})$ are continuous in their parameters for all $\mathbf{X}$ and $\mathbf{X}^{\text{SK}}$.
\end{assumption}

This ensures that small perturbations in $(\theta, \phi)$ induce only small changes in the objective, as required by standard extremum-estimator convergence arguments \citep{newey1994large}. Neural networks with continuous activation functions (e.g., $\mathrm{ReLU}, \mathrm{sigmoid}, \mathrm{tanh}$) satisfy this condition automatically. Throughout the paper, we use squared-error loss unless stated otherwise.

\begin{assumption}[Collocation Point Growth]
\label{assump:collocation_growth}
The number of collocation points $M_N$ is a deterministic sequence satisfying $M_N \to \infty$ as $N \to \infty$. For the asymptotic normality result (Theorem~\ref{thm:asymptotic_normality}), we further require proportional growth:
\begin{equation*}
M_N / N \to c \quad \text{for some constant } c \in (0, \infty).
\end{equation*}
\end{assumption}

The first condition ensures that the empirical approximation of the structured-knowledge loss $\mathcal{L}_{\text{SK}}$ converges to its population counterpart, which suffices for consistency (Theorem~\ref{thm:consistency}). The proportional growth condition is stronger and ensures that discretization error from the finite collocation grid diminishes at the same rate as sampling error, preserving the parametric $N^{-1/2}$ convergence rate. If $M_N$ remained fixed as $N$ grows, persistent discretization bias would degrade the convergence rate below $N^{-1/2}$. Conversely, if $M_N$ grew too rapidly relative to $N$ (for example, $M_N = N^2$), computational costs would become prohibitive without improving the statistical rate. The proportional growth condition strikes this balance. In the sieve estimation literature \citep{chen2007large}, analogous conditions govern the growth of basis functions with sample size. In practice, researchers typically set $M_N = N$ (reusing observed inputs as collocation points) or $M_N = cN$ for a small constant $c \in [1, 5]$.

\begin{assumption}[Uniform Convergence]
\label{assump:uniform}
As $N \to \infty$ with $M_N$ satisfying Assumption~\ref{assump:collocation_growth}, the empirical loss $\hat{\mathcal{L}}_{N,M_N}(\theta, \phi)$ converges uniformly to the population loss $\mathcal{L}(\theta, \phi)$ over $\Theta \times \Phi$:
\[
\sup_{(\theta, \phi) \in \Theta \times \Phi} \left|\hat{\mathcal{L}}_{N,M_N}(\theta, \phi) - \mathcal{L}(\theta, \phi)\right| \xrightarrow{p} 0.
\]
\end{assumption}

Here $\hat{\mathcal{L}}_{N,M_N}$ denotes the sample analog of the composite objective in Equation~\eqref{eq:normalized_objective}, with the data loss averaged over $N$ observations and the structured-knowledge loss averaged over $M_N$ collocation points. This uniform law of large numbers requires that the function class $\{(\theta,\phi) \mapsto \ell(f_\theta(\mathbf{X}), \mathbf{y}) + \lambda\,\ell(f_\theta(\mathbf{X}_{\text{grid}}), g_\phi(\mathbf{X}_{\text{grid}}^{\text{SK}}))\}$ is not too complex. For neural networks with a fixed architecture (finite $p$), bounded activation functions, and compact $\Theta \times \Phi$, this condition is satisfied by standard empirical process theory \citep{van1996weak, vapnik1971uniform}.

\begin{assumption}[Asymptotic Normality of the Score]
\label{assump:clt}
The score function 
$$s(\mathbf{X}, \mathbf{y}; \theta, \phi) = \nabla_{(\theta, \phi)} \left[\ell(f_\theta(\mathbf{X}), \mathbf{y}) + \lambda\, \ell(f_\theta(\tilde{\mathbf{X}}), g_\phi(\tilde{\mathbf{X}}^{\text{SK}}))\right]
$$
satisfies a central limit theorem at the true parameter values:
\begin{equation*}
    \frac{1}{\sqrt{N}} \sum_{i=1}^{N} s(\mathbf{X}_i, \mathbf{y}_i; \theta^*, \phi^*) \xrightarrow{d} \mathcal{N}\!\left(\mathbf{0},\, \Xi\right),
\end{equation*}
where $\Xi = \mathbb{E}\!\left[s(\mathbf{X}, \mathbf{y}; \theta^*, \phi^*)\, s(\mathbf{X}, \mathbf{y}; \theta^*, \phi^*)^\top\right]$ is the score covariance matrix, assumed to be finite and positive definite.
\end{assumption}

This is a standard condition for M-estimator inference \citep{newey1994large, van1996weak}. It holds when the score has finite second moments and the observations are independent and identically distributed, or more generally when the score forms a Donsker class.


\begin{assumption}[Second-Order Differentiability]
\label{assump:hessian}
The objective $\mathcal{L}(\theta, \phi)$ is twice continuously differentiable in a neighborhood of $(\theta^*, \phi^*)$, and the Hessian matrix
\[
H = \nabla^2_{(\theta,\phi)}\, \mathcal{L}(\theta^*, \phi^*)
\]
exists and is positive definite.
\end{assumption}

This assumption requires, in particular, that the structured-knowledge function $g_\phi(\mathbf{X}^{\text{SK}})$ is twice differentiable with respect to $\phi$. For parametric representations with smooth closed-form expressions (e.g., the Black--Scholes formula), this is immediate. For semi-parametric representations based on deep surrogate neural networks, twice differentiability is inherited from the smoothness of the surrogate's activation functions, as discussed in Section~\ref{sec:differntial_curse_dimensionality}. Positive definiteness of $H$ ensures that $(\theta^*, \phi^*)$ is a strict local minimum and that the sandwich covariance $V = H^{-1}\Xi H^{-1}$ is well-defined.

\subsection{Statistical Properties}\label{sec:general_theory}

The SKINNs estimator $(\hat{\theta}_{N}, \hat{\phi}_{N})$ belongs to the class of $M$-estimators, a broad family of estimators defined as minimizers of an empirical criterion function \citep{huber1967behavior, van1996weak}. This class encompasses maximum likelihood, least squares, and method of moments estimators as special cases, and admits a mature asymptotic toolkit that we exploit here. Our presentation follows the standard approach in the econometric and statistical literature \citep{newey1994large, van1996weak}, establishing consistency and asymptotic normality under the regularity conditions stated in Appendix~\ref{sec:assumption}. Importantly, these results are domain-agnostic: they apply whether SKINNs are used to learn physical laws from experimental data, infer economic parameters from market observations, or discover structural regularities in other scientific domains.

\subsubsection{Consistency}

Consistency ensures that, as the sample size grows, the jointly estimated parameters converge in probability to the population minimizers of the composite objective in Equation~\eqref{eq:normalized_objective}.

\begin{theorem}[Consistency]
\label{thm:consistency}
Suppose Assumptions~\ref{assump:identification}--\ref{assump:uniform} hold, and $M_N \to \infty$ as $N \to \infty$. Then the SKINNs estimator is consistent:
\begin{equation*}
(\hat{\theta}_N, \hat{\phi}_N) \xrightarrow{p} (\theta^*, \phi^*) \quad \text{as } N \to \infty.
\end{equation*}
\end{theorem}

\begin{proof}
The argument follows the standard route for extremum estimators \citep[Theorem~2.1]{newey1994large}. By Assumption~\ref{assump:uniform}, the empirical loss surface $\hat{\mathcal{L}}_{N,M_N}$ converges uniformly in probability to the population loss $\mathcal{L}$ over the compact set $\Theta \times \Phi$ (Assumption~\ref{assump:compactness}). By continuity of $\mathcal{L}$ in $(\theta, \phi)$ (Assumption~\ref{assump:continuity}), for any $\epsilon > 0$,
\begin{equation*}
\sup_{(\theta,\phi) \in \Theta \times \Phi} \left|\hat{\mathcal{L}}_{N,M_N}(\theta,\phi) - \mathcal{L}(\theta,\phi)\right| \xrightarrow{p} 0
\end{equation*}
implies that any sequence of approximate minimizers of $\hat{\mathcal{L}}_{N,M_N}$ must eventually enter every neighborhood of the set of minimizers of $\mathcal{L}$. The well-separation condition in Assumption~\ref{assump:identification} guarantees that the minimizer $(\theta^*, \phi^*)$ is unique, so every convergent subsequence of $(\hat{\theta}_N, \hat{\phi}_N)$ converges to $(\theta^*, \phi^*)$. Since $\Theta \times \Phi$ is compact, every subsequence has a further convergent subsequence, and uniqueness of the limit yields convergence of the full sequence. The growth condition $M_N \to \infty$ in Assumption~\ref{assump:collocation_growth} ensures that the collocation-based approximation of the structured-knowledge loss $\mathcal{L}_{\text{SK}}$ contributes to the uniform convergence of the composite objective.
\end{proof}

Under misspecification of the structured-knowledge component $g_\phi$, the limiting parameters should be interpreted as pseudo-true values in the sense of \citet{white1982maximum}: the parameters that minimize the population objective $\mathcal{L}(\theta, \phi)$ even though $g_\phi$ may not coincide with the true data-generating process. The consistency guarantee then ensures convergence to these pseudo-true values. This applies universally across domains. In physics, $\phi$ might represent material constants; in economics, $\phi$ could be risk preferences or volatility parameters. In either case, Theorem~\ref{thm:consistency} guarantees that given sufficient data, the learned parameters approach their population-optimal values, despite the high-dimensional and potentially nonconvex nature of the neural network optimization landscape.


\subsubsection{Asymptotic Normality and Convergence Rates}

We now characterize the sampling distribution of the estimation error, enabling the construction of confidence intervals and hypothesis tests for the learned parameters. Define the score function as in Assumption~\ref{assump:clt}:
\begin{equation}\label{eq:score_def}
s(\mathbf{X}, \mathbf{y}; \theta, \phi) = \nabla_{(\theta, \phi)} \left[\ell(f_\theta(\mathbf{X}), \mathbf{y}) + \lambda\, \ell(f_\theta(\tilde{\mathbf{X}}), g_\phi(\tilde{\mathbf{X}}^{\text{SK}}))\right],
\end{equation}
and let
\begin{equation}\label{eq:Xi_H_def}
\Xi = \mathbb{E}\!\left[s(\mathbf{X}, \mathbf{y}; \theta^*, \phi^*)\, s(\mathbf{X}, \mathbf{y}; \theta^*, \phi^*)^\top\right], \qquad H = \nabla^2_{(\theta,\phi)}\, \mathcal{L}(\theta^*, \phi^*)
\end{equation}
denote the score covariance and the Hessian of the population loss, respectively.

\begin{theorem}[Asymptotic Normality]
\label{thm:asymptotic_normality}
Suppose Assumptions~\ref{assump:identification}--\ref{assump:hessian} hold and $M_N/N \to c$ for some $c \in (0, \infty)$. Then the SKINNs estimator is asymptotically normal at the parametric rate:
\begin{equation}\label{eq:asymp_normal}
\sqrt{N} \begin{pmatrix} \hat{\theta}_N - \theta^* \\ \hat{\phi}_N - \phi^* \end{pmatrix} \xrightarrow{d} \mathcal{N}(\mathbf{0}, V), \qquad V = H^{-1} \Xi\, H^{-1}.
\end{equation}
\end{theorem}

\begin{proof}
Since $(\hat{\theta}_N, \hat{\phi}_N)$ minimizes $\hat{\mathcal{L}}_{N,M_N}$, the first-order condition gives
\begin{equation*}
\nabla_{(\theta,\phi)} \hat{\mathcal{L}}_{N,M_N}(\hat{\theta}_N, \hat{\phi}_N) = \mathbf{0}.
\end{equation*}
A mean-value expansion of this condition around the true parameter $(\theta^*, \phi^*)$ yields
\begin{equation}\label{eq:taylor_foc}
\mathbf{0} = \frac{1}{N}\sum_{i=1}^{N} s(\mathbf{X}_i, \mathbf{y}_i; \theta^*, \phi^*) + \nabla^2_{(\theta,\phi)} \hat{\mathcal{L}}_{N,M_N}(\bar{\theta}_N, \bar{\phi}_N) \begin{pmatrix} \hat{\theta}_N - \theta^* \\ \hat{\phi}_N - \phi^* \end{pmatrix},
\end{equation}
where $(\bar{\theta}_N, \bar{\phi}_N)$ lies on the line segment between $(\hat{\theta}_N, \hat{\phi}_N)$ and $(\theta^*, \phi^*)$. By consistency (Theorem~\ref{thm:consistency}), $(\bar{\theta}_N, \bar{\phi}_N) \xrightarrow{p} (\theta^*, \phi^*)$. Combined with the uniform convergence in Assumption~\ref{assump:uniform} and the smoothness in Assumption~\ref{assump:hessian}, the empirical Hessian converges in probability:
\begin{equation*}
\nabla^2_{(\theta,\phi)} \hat{\mathcal{L}}_{N,M_N}(\bar{\theta}_N, \bar{\phi}_N) \xrightarrow{p} H.
\end{equation*}
Rearranging Equation~\eqref{eq:taylor_foc} gives
\begin{equation*}
\sqrt{N} \begin{pmatrix} \hat{\theta}_N - \theta^* \\ \hat{\phi}_N - \phi^* \end{pmatrix} = -\left[\nabla^2_{(\theta,\phi)} \hat{\mathcal{L}}_{N,M_N}(\bar{\theta}_N, \bar{\phi}_N)\right]^{-1} \frac{1}{\sqrt{N}}\sum_{i=1}^{N} s(\mathbf{X}_i, \mathbf{y}_i; \theta^*, \phi^*).
\end{equation*}
By Assumption~\ref{assump:clt}, the normalized score converges in distribution to $\mathcal{N}(\mathbf{0}, \Xi)$. Slutsky's theorem, together with the positive definiteness of $H$ (Assumption~\ref{assump:hessian}) guaranteeing invertibility, yields
\begin{equation*}
\sqrt{N} \begin{pmatrix} \hat{\theta}_N - \theta^* \\ \hat{\phi}_N - \phi^* \end{pmatrix} \xrightarrow{d} \mathcal{N}(\mathbf{0},\, H^{-1}\Xi\, H^{-1}).
\end{equation*}
The proportional growth condition $M_N / N \to c \in (0, \infty)$ in Assumption~\ref{assump:collocation_growth} ensures that the collocation-based approximation of the structured-knowledge loss introduces no additional bias that would degrade the $O_p(N^{-1/2})$ convergence rate: both the data-loss component and the structured-knowledge-loss component of the score contribute to $\Xi$ at the same order.
\end{proof}


Several features of this result merit emphasis. First, the asymptotic covariance $V = H^{-1}\Xi\, H^{-1}$ takes the sandwich form of \citet{huber1967behavior} and \citet{white1980heteroskedasticity}. This structure is robust to potential misspecification of the structured-knowledge component $g_\phi$: even if the embedded theory is an imperfect description of the true data-generating process, the sandwich covariance provides asymptotically valid inference. Second, the convergence rate is $O_p(N^{-1/2})$, the standard parametric rate. This holds despite the potentially very high dimensionality of the neural network parameters $\theta$. The structured-knowledge component acts as implicit regularization, effectively reducing the complexity of the estimation problem. Third, asymptotic normality enables the construction of confidence intervals and hypothesis tests. For a scalar component $\phi_j$ of the structural parameters, an approximate $(1 - \alpha)$ confidence interval takes the form
\begin{equation}\label{eq:ci_phi}
\hat{\phi}_{j,N} \pm z_{\alpha/2} \sqrt{\frac{V_{jj}}{N}},
\end{equation}
where $V_{jj}$ is the $(j,j)$-th element of $V$ and $z_{\alpha/2}$ is the standard normal critical value.

In Section~\ref{app:sp-eff}, we refine these results for the structural parameters under orthogonal moment conditions, establishing semiparametric efficiency within the given moment model. We note that the formal results above are derived under a fixed neural network architecture, where the dimension $p$ of $\theta$ remains constant as the sample size $N$ increases. This aligns with standard $M$-estimator theory and reflects common empirical practice, where a network architecture is selected and held fixed during training. Extending these results to the growing-width or growing-depth regime, in which the network architecture scales with $N$ as in sieve estimation theory \citep{chen2007large}, would require additional regularity conditions on the rate of growth and is left for future work. Even under a fixed architecture, modern neural networks with millions of parameters provide substantial approximation capacity for practical applications.

\subsubsection{Inference and Practical Considerations}\label{sec:infer_practical}

Implementation of the asymptotic results in Theorems~\ref{thm:consistency}--\ref{thm:asymptotic_normality} requires consistent estimators of the Hessian $H$ and the score covariance $\Xi$. The Hessian can be estimated via the empirical analog
\begin{equation}\label{eq:H_hat}
\hat{H}_N = \frac{1}{N} \sum_{i=1}^{N} \nabla^2_{(\theta,\phi)}\, \hat{\mathcal{L}}_{N,M_N}(\hat{\theta}_N, \hat{\phi}_N;\, \mathbf{X}_i, \mathbf{y}_i),
\end{equation}
computed efficiently through automatic differentiation. The score covariance is estimated by the outer product of gradients:
\begin{equation}\label{eq:Xi_hat}
\hat{\Xi}_N = \frac{1}{N} \sum_{i=1}^{N} s(\mathbf{X}_i, \mathbf{y}_i;\, \hat{\theta}_N, \hat{\phi}_N)\, s(\mathbf{X}_i, \mathbf{y}_i;\, \hat{\theta}_N, \hat{\phi}_N)^\top.
\end{equation}
The sandwich covariance estimator $\hat{V}_N = \hat{H}_N^{-1}\, \hat{\Xi}_N\, \hat{H}_N^{-1}$ is then consistent for $V$ under the conditions of Theorem~\ref{thm:asymptotic_normality}. Because scientific conclusions typically center on the low-dimensional structural parameters $\phi$ rather than the high-dimensional nuisance weights $\theta$, block-matrix inversion can be used to extract the $q \times q$ covariance submatrix for $\phi$ without computing or inverting the full $(p+q) \times (p+q)$ Hessian. This makes confidence intervals as in Equation~\eqref{eq:ci_phi} and standard Wald-type hypothesis tests for $\phi$ computationally feasible even when the neural network contains millions of parameters. As a finite-sample alternative, nonparametric bootstrap resampling of $\mathcal{D}_N$ with re-optimization of the SKINNs objective can be used to estimate the sampling distribution of $(\hat{\theta}_N, \hat{\phi}_N)$ without relying on asymptotic approximations.

\subsection{The GMM Interpretation for Economic Applications}
\label{sec:gmm_interpretation}

The statistical properties established in Section~\ref{sec:general_theory} apply universally across scientific domains. For economic and financial applications, however, the SKINNs framework admits a natural reinterpretation through the lens of the generalized method of moments \citep[GMM;][]{hansen1982large}. This perspective clarifies the economic meaning of the composite objective, connects SKINNs to a rich body of established econometric theory, and provides practical guidance for selecting the regularization parameter $\lambda$. We emphasize that GMM is one of several valid theoretical lenses on SKINNs; Bayesian, variational, and machine learning perspectives are developed in Section~\ref{sec:alternative_perspectives}.

\subsubsection{SKINNs as Regularized Overidentified GMM}

Define the vector-valued moment function
\begin{equation}\label{eq:moment_function}
\mathbf{m}(\mathbf{X}, \mathbf{y}; \theta, \phi) = \begin{bmatrix} 
f_\theta(\mathbf{X}) - \mathbf{y} \\ 
f_\theta(\tilde{\mathbf{X}}) - g_\phi(\tilde{\mathbf{X}}^{\text{SK}}) 
\end{bmatrix},
\end{equation}
where the first component is the prediction error on observed data and the second is the deviation between the neural network and the structured-knowledge function evaluated at collocation points. Under correct specification, both moment conditions are satisfied in expectation:
\begin{equation}\label{eq:moment_conditions}
\mathbb{E}[\mathbf{m}(\mathbf{X}, \mathbf{y}; \theta^*, \phi^*)] = \mathbf{0}.
\end{equation}
The SKINNs objective can then be written as a quadratic form in these moments:
\begin{equation}\label{eq:gmm_objective}
\mathcal{L}(\theta, \phi) = \mathbb{E}\!\left[\mathbf{m}(\mathbf{X}, \mathbf{y}; \theta, \phi)^\top \mathbf{W}\, \mathbf{m}(\mathbf{X}, \mathbf{y}; \theta, \phi)\right],
\end{equation}
with the diagonal weighting matrix $\mathbf{W} = \operatorname{diag}(1, \lambda)$. The empirical counterpart is
\begin{equation}\label{eq:empirical_gmm}
\hat{\mathcal{L}}_{N,M_N}(\theta, \phi) = \mathbf{m}_N(\theta, \phi)^\top \mathbf{W}\, \mathbf{m}_N(\theta, \phi),
\end{equation}
where $\mathbf{m}_N(\theta, \phi)$ is the sample analog of the moment vector, and the estimator $(\hat{\theta}_N, \hat{\phi}_N)$ minimizes this weighted quadratic form, exactly as in standard GMM.

The system is overidentified in the sense that two conceptually distinct sources of information, observed outcomes and theoretical structure, each supply a moment condition for the joint estimation of $(\theta, \phi)$. SKINNs differ from classical GMM in three respects. First, $\theta$ is high-dimensional and approximates an infinite-dimensional function, whereas classical GMM typically involves finite-dimensional parameters. Second, $\mathbf{W}$ is not the efficiency-optimal weight (the inverse of the moment covariance) but a user-specified diagonal matrix encoding a regularization choice. Third, the second moment condition enforces alignment between $f_\theta$ and $g_\phi$ over a potentially broader input domain than the observed data, acting as a soft structural constraint that promotes generalization to regions where observations may be sparse.

This formulation clarifies what SKINNs achieves from an economic standpoint. The first moment condition ensures empirical accuracy; the second ensures theoretical plausibility. The weighting parameter $\lambda$ governs the relative importance of these two objectives: a large $\lambda$ reflects strong confidence in the embedded theory, while a small $\lambda$ prioritizes empirical fit.

\subsubsection{The Role of $\lambda$ as Econometric Weighting}

In classical GMM, efficiency is achieved by setting the weighting matrix to the inverse of the moment covariance:
\begin{equation}\label{eq:optimal_weight}
\mathbf{W}_{\text{optimal}} = \left(\mathbb{E}\!\left[\mathbf{m}(\mathbf{X}, \mathbf{y}; \theta^*, \phi^*)\, \mathbf{m}(\mathbf{X}, \mathbf{y}; \theta^*, \phi^*)^\top\right]\right)^{-1}.
\end{equation}
SKINNs instead use $\mathbf{W} = \operatorname{diag}(1, \lambda)$, which is generally not optimal in the Hansen-Singleton sense. This deviation is deliberate: when the model is complex (high-dimensional $\theta$) and data are noisy, imposing structure through regularization can improve finite-sample performance even at the cost of asymptotic efficiency, a principle well established in the high-dimensional econometrics literature \citep{belloni2014high}. The parameter $\lambda$ governs a bias-variance tradeoff: increasing $\lambda$ reduces variance by constraining the function space but may introduce bias if $g_\phi$ is misspecified.

Even within the restricted diagonal class $\mathbf{W}(\lambda) = \operatorname{diag}(I, \lambda I)$, one can choose $\lambda$ to optimize a well-defined criterion.

\begin{proposition}[Restricted-optimal $\lambda$]\label{prop:lambda-restricted}
Let $V_\phi(\lambda)$ denote the asymptotic covariance matrix of $\hat{\phi}$ implied by the sandwich formula under weighting $\mathbf{W}(\lambda) = \operatorname{diag}(I, \lambda I)$, assuming the regularity conditions for asymptotic normality hold. Define
\begin{equation}\label{eq:lambda_restricted}
\lambda^* \in \arg\min_{\lambda > 0}\; \mathrm{tr}\!\left(V_\phi(\lambda)\right).
\end{equation}
Then $\lambda^*$ is the variance-minimizing weight within the restricted diagonal class.
\end{proposition}

\begin{proof}
This follows directly from the definition of $\lambda^*$ as the minimizer of the stated criterion over the admissible set $\lambda > 0$.
\end{proof}

The next result provides a transparent closed-form expression for $\lambda$ in a stylized setting, clarifying its interpretation as a relative precision ratio.

\begin{proposition}[Closed-form $\lambda$ under two noisy signals]\label{prop:lambda-closedform}
Assume the target function $f_0(\mathbf{X})$ is observed through two conditionally unbiased noisy signals:
\begin{equation}\label{eq:two_signals}
\mathbf{y} = f_0(\mathbf{X}) + \varepsilon, \qquad g_{\phi_0}(\mathbf{X}^{\text{SK}}) = f_0(\mathbf{X}) + \eta,
\end{equation}
where $\mathbb{E}[\varepsilon \mid \mathbf{X}] = 0$, $\mathbb{E}[\eta \mid \mathbf{X}] = 0$, $\varepsilon \perp\!\!\!\perp \eta \mid \mathbf{X}$, and $\mathrm{Var}(\varepsilon \mid \mathbf{X}) = \sigma_\varepsilon^2$, $\mathrm{Var}(\eta \mid \mathbf{X}) = \sigma_\eta^2$ are constants. Consider estimating $f_0$ by minimizing the SKINNs objective over all measurable $f$ with finite second moment:
\begin{equation}\label{eq:stylized_skinns}
\mathbb{E}\!\left[(f(\mathbf{X}) - \mathbf{y})^2\right] + \lambda\, \mathbb{E}\!\left[(f(\mathbf{X}) - g_{\phi_0}(\mathbf{X}^{\text{SK}}))^2\right].
\end{equation}
Then the minimizer is
\begin{equation}\label{eq:f_star_stylized}
f^*(\mathbf{X}) = \frac{\mathbf{y} + \lambda\, g_{\phi_0}(\mathbf{X}^{\text{SK}})}{1 + \lambda},
\end{equation}
and the mean squared error $\mathbb{E}\!\left[(f^*(\mathbf{X}) - f_0(\mathbf{X}))^2\right]$ is minimized at
\begin{equation}\label{eq:lambda_star}
\lambda^* = \frac{\sigma_\varepsilon^2}{\sigma_\eta^2}.
\end{equation}
\end{proposition}

\begin{proof}
Given $\phi_0$, the pointwise minimizer under squared loss is $f^*(\mathbf{X}) = (\mathbf{y} + \lambda\, g_{\phi_0}(\mathbf{X}^{\text{SK}}))/(1 + \lambda)$. Substituting the signal model yields
\begin{equation*}
f^*(\mathbf{X}) - f_0(\mathbf{X}) = \frac{\varepsilon + \lambda\, \eta}{1 + \lambda}.
\end{equation*}
Using conditional independence and constant conditional variances,
\begin{equation*}
\mathbb{E}\!\left[(f^*(\mathbf{X}) - f_0(\mathbf{X}))^2 \mid \mathbf{X}\right] = \frac{\sigma_\varepsilon^2 + \lambda^2 \sigma_\eta^2}{(1 + \lambda)^2}.
\end{equation*}
Differentiating with respect to $\lambda$ and setting the result to zero gives
\begin{equation*}
\frac{d}{d\lambda}\left(\frac{\sigma_\varepsilon^2 + \lambda^2 \sigma_\eta^2}{(1 + \lambda)^2}\right) = 0 \quad \Longleftrightarrow \quad \lambda^* = \frac{\sigma_\varepsilon^2}{\sigma_\eta^2},
\end{equation*}
and the second-order condition is satisfied since the objective is strictly convex in $\lambda$ for $\sigma_\eta^2 > 0$.
\end{proof}

Proposition~\ref{prop:lambda-closedform} formalizes the interpretation of $\lambda$ as a relative precision (inverse-noise-variance) weight: the structured-knowledge signal receives more weight when the data are noisy ($\sigma_\varepsilon^2$ large) and less weight when the theory is imprecise ($\sigma_\eta^2$ large). In general applications, Proposition~\ref{prop:lambda-restricted} motivates choosing $\lambda$ to optimize a variance or mean squared error criterion within the restricted diagonal class.

\subsection{Semiparametric Efficiency}\label{app:sp-eff}

The GMM formulation in Section~\ref{sec:gmm_interpretation} connects SKINNs to the literature on semiparametric estimation and penalized moment-based methods. In the semiparametric framework \citep{bickel1993efficient}, the structural parameters $\phi$ are finite-dimensional and economically interpretable, while the neural network parameters $\theta$ approximate an infinite-dimensional nuisance function. SKINNs can thus be viewed as a neural network implementation of sieve GMM \citep{ai2003efficient}, where the universal approximation properties of deep networks \citep{cybenko1989approximation, hornik1989multilayer} ensure that $f_\theta$ can approximate any continuous function as the network capacity grows. The structured-knowledge component $g_\phi$ plays a role analogous to an instrument in classical GMM: it provides additional identifying information that refines the estimation of $\phi$ beyond what pure data fitting would achieve. Recent work on deep GMM \citep{bennett2019deep, farrell2021deep} has shown that neural networks can learn efficient instruments and conditional expectations in moment-based frameworks. SKINNs extend these ideas by incorporating an explicit structured-knowledge component with learnable parameters, enabling joint discovery of both the flexible function and the economically meaningful structural parameters.

The use of $\lambda$ to balance moment conditions also connects to the literature on regularized GMM. \citet{carrasco2007linear} and \citet{antoine2009efficient} study GMM with regularization penalties to stabilize estimation under weak instruments or ill-conditioned moment covariance matrices. While their focus is on numerical stability, the underlying principle is shared: judiciously penalizing certain dimensions of the objective can improve finite-sample performance. SKINNs extend this idea to infinite-dimensional function approximation with neural networks.

A natural question is whether the SKINNs estimator achieves any efficiency optimality for the structural parameters $\phi$. We show below that under orthogonal moment conditions, in which the infinite-dimensional nuisance parameter exerts only a second-order influence on the estimation of $\phi$, the SKINNs estimator attains the semiparametric efficiency bound for the given moment model.

\subsubsection{Setup and Assumptions}

Let $\mathbf{Z} = (\mathbf{X}, \mathbf{y})$ denote the generic observation and let the finite-dimensional parameter of interest be $\phi \in \mathbb{R}^q$.\footnote{We use $\mathbf{Z}$ for the generic observation in this subsection to avoid conflict with the GMM weighting matrix $\mathbf{W}$ in Section~\ref{sec:gmm_interpretation}.} Let $f$ denote an infinite-dimensional nuisance object taking values in a function space $\mathcal{F}$, with true value $f_0$. In the SKINNs context, $f$ corresponds to the neural network approximator $f_\theta$ and $f_0$ to its population limit. Suppose the model is characterized by a vector of moment restrictions
\begin{equation}\label{eq:moment_restrictions}
\mathbb{E}\!\left[m(\mathbf{Z};\, \phi^*,\, f_0)\right] = \mathbf{0}, \qquad m(\cdot;\, \phi,\, f) \in \mathbb{R}^k,\quad k \ge q,
\end{equation}
where $\phi^*$ denotes the true (or pseudo-true) parameter value.\footnote{We use $\Omega$ for the moment covariance matrix in this subsection, following the convention in the semiparametric efficiency literature. This is distinct from the score covariance $\Xi$ defined in Section~\ref{sec:general_theory}. Under correct specification of the moment model, $\Omega$ coincides with the relevant block of $\Xi$.} Let $\hat{f}$ be a first-stage estimator of $f_0$ and define the feasible optimal-weight GMM estimator
\begin{equation}\label{eq:feasible_gmm}
\hat{\phi} \in \arg\min_{\phi \in \Phi}\; \hat{g}(\phi)^\top \hat{\Omega}^{-1} \hat{g}(\phi), \qquad \hat{g}(\phi) \equiv \frac{1}{N} \sum_{i=1}^{N} m(\mathbf{Z}_i;\, \phi,\, \hat{f}),
\end{equation}
where $\hat{\Omega}$ is a consistent estimator of $\Omega \equiv \mathbb{E}\!\left[m(\mathbf{Z};\, \phi^*,\, f_0)\, m(\mathbf{Z};\, \phi^*,\, f_0)^\top\right]$.

\begin{assumption}\label{ass:sp}
\textup{(i)} \textbf{(Smoothness)} $m(\mathbf{Z};\, \phi,\, f)$ is continuously differentiable in $\phi$ in a neighborhood of $(\phi^*, f_0)$, and the Jacobian
\begin{equation}\label{eq:jacobian_G}
G \equiv \frac{\partial}{\partial \phi^\top} \mathbb{E}\!\left[m(\mathbf{Z};\, \phi,\, f_0)\right]\bigg|_{\phi = \phi^*}
\end{equation}
exists and has full column rank $q$.

\textup{(ii)} \textbf{(Orthogonality)} For every direction $h$ in the tangent set of $\mathcal{F}$ at $f_0$,
\begin{equation*}
\left.\frac{d}{dt}\, \mathbb{E}\!\left[m(\mathbf{Z};\, \phi^*,\, f_0 + t\, h)\right]\right|_{t=0} = \mathbf{0}.
\end{equation*}

\textup{(iii)} \textbf{(Rates)} $\|\hat{f} - f_0\| = o_p(N^{-1/4})$ under a norm such that the remainder bounds in \textup{(iv)} hold.

\textup{(iv)} \textbf{(Remainder control)} Uniformly over $\phi$ in a neighborhood of $\phi^*$,
\begin{equation*}
\left\|\frac{1}{N} \sum_{i=1}^{N} \left(m(\mathbf{Z}_i;\, \phi,\, \hat{f}) - m(\mathbf{Z}_i;\, \phi,\, f_0)\right) - \mathbb{E}\!\left[m(\mathbf{Z};\, \phi,\, \hat{f}) - m(\mathbf{Z};\, \phi,\, f_0)\right]\right\| = o_p(N^{-1/2}),
\end{equation*}
and the population bias admits a second-order bound:
\begin{equation*}
\left\|\mathbb{E}\!\left[m(\mathbf{Z};\, \phi^*,\, \hat{f}) - m(\mathbf{Z};\, \phi^*,\, f_0)\right]\right\| \le C\, \|\hat{f} - f_0\|^2
\end{equation*}
for some constant $C$ with probability approaching one.
\end{assumption}

Assumption~\ref{ass:sp}(ii) is the key condition. Orthogonality ensures that estimation error in the first-stage nuisance parameter $\hat{f}$ has only a second-order effect on the estimation of $\phi$, so that the infinite-dimensional estimation problem does not degrade the $N^{-1/2}$ convergence rate for the finite-dimensional parameter of interest. This condition is the semiparametric analog of the Neyman orthogonality condition used in the debiased machine learning literature \citep{Chernozhukov2018}. Assumption~\ref{ass:sp}(iii) requires that the first-stage estimator converges at a rate faster than $N^{-1/4}$, a condition satisfied by neural network sieves under standard smoothness conditions on the target function \citep{chen2007large, farrell2021deep}.

\subsubsection{Efficiency Result}

\begin{theorem}[Asymptotic linearity and semiparametric efficiency]\label{thm:sp-eff}
Under Assumption~\ref{ass:sp} and consistency of $\hat{\Omega}$,
\begin{equation}\label{eq:sp_eff_normality}
\sqrt{N}(\hat{\phi} - \phi^*) \;\Rightarrow\; \mathcal{N}\!\left(\mathbf{0},\, V^*\right), \qquad V^* \equiv (G^\top \Omega^{-1} G)^{-1}.
\end{equation}
Moreover, among regular asymptotically linear estimators based on the moment restriction $\mathbb{E}\!\left[m(\mathbf{Z};\, \phi,\, f_0)\right] = \mathbf{0}$, the covariance $V^*$ is the minimal achievable asymptotic variance (the optimal-GMM bound), so $\hat{\phi}$ is semiparametrically efficient for $\phi$ relative to this moment model.
\end{theorem}

\begin{proof}
\textit{Step 1 (Expansion of sample moments).} Write
\begin{equation}\label{eq:ghat_expansion}
\hat{g}(\phi) = \frac{1}{N} \sum_{i=1}^{N} m(\mathbf{Z}_i;\, \phi,\, f_0) + A_N(\phi) + B_N(\phi),
\end{equation}
where
\begin{equation*}
A_N(\phi) \equiv \frac{1}{N} \sum_{i=1}^{N} \left(m(\mathbf{Z}_i;\, \phi,\, \hat{f}) - m(\mathbf{Z}_i;\, \phi,\, f_0)\right) - \mathbb{E}\!\left[m(\mathbf{Z};\, \phi,\, \hat{f}) - m(\mathbf{Z};\, \phi,\, f_0)\right]
\end{equation*}
is the stochastic equicontinuity remainder, and
\begin{equation*}
B_N(\phi) \equiv \mathbb{E}\!\left[m(\mathbf{Z};\, \phi,\, \hat{f}) - m(\mathbf{Z};\, \phi,\, f_0)\right]
\end{equation*}
is the population bias. By Assumption~\ref{ass:sp}(iv), uniformly near $\phi^*$, $A_N(\phi) = o_p(N^{-1/2})$.

\textit{Step 2 (Orthogonality eliminates the first-stage bias).} Applying Assumption~\ref{ass:sp}(ii) and the second-order remainder bound in Assumption~\ref{ass:sp}(iv) at $\phi = \phi^*$:
\begin{equation*}
\|B_N(\phi^*)\| = \left\|\mathbb{E}\!\left[m(\mathbf{Z};\, \phi^*,\, \hat{f}) - m(\mathbf{Z};\, \phi^*,\, f_0)\right]\right\| \le C\, \|\hat{f} - f_0\|^2 = o_p(N^{-1/2}),
\end{equation*}
where the final equality uses $\|\hat{f} - f_0\| = o_p(N^{-1/4})$ from Assumption~\ref{ass:sp}(iii).

\textit{Step 3 (Taylor expansion in $\phi$).} Expand $\hat{g}(\phi)$ around $\phi^*$:
\begin{equation*}
\hat{g}(\phi) = \hat{g}(\phi^*) + \hat{G}(\bar{\phi})\, (\phi - \phi^*),
\end{equation*}
where $\bar{\phi}$ lies between $\phi$ and $\phi^*$, and $\hat{G}(\phi) \equiv \frac{\partial}{\partial \phi^\top} \hat{g}(\phi)$. By Assumption~\ref{ass:sp}(i) and standard law-of-large-numbers arguments, $\hat{G}(\bar{\phi}) \xrightarrow{p} G$ uniformly in a neighborhood of $\phi^*$.

\textit{Step 4 (First-order condition).} The GMM objective is $J_N(\phi) = \hat{g}(\phi)^\top \hat{\Omega}^{-1} \hat{g}(\phi)$. The first-order condition at $\hat{\phi}$ is
\begin{equation*}
\mathbf{0} = \hat{G}(\hat{\phi})^\top \hat{\Omega}^{-1} \hat{g}(\hat{\phi}).
\end{equation*}
Substituting the Taylor expansion from Step 3 and rearranging:
\begin{equation}\label{eq:sp_linear_rep}
\sqrt{N}(\hat{\phi} - \phi^*) = -\left(\hat{G}(\hat{\phi})^\top \hat{\Omega}^{-1} \hat{G}(\bar{\phi})\right)^{-1} \hat{G}(\hat{\phi})^\top \hat{\Omega}^{-1} \sqrt{N}\, \hat{g}(\phi^*).
\end{equation}
By consistency, $\hat{\Omega} \xrightarrow{p} \Omega$ and $\hat{G}(\cdot) \xrightarrow{p} G$, so the matrix prefactor converges to $(G^\top \Omega^{-1} G)^{-1} G^\top \Omega^{-1}$, which is well-defined by the full-rank condition in Assumption~\ref{ass:sp}(i).

\textit{Step 5 (CLT for the leading term).} Combining Steps 1 and 2,
\begin{equation*}
\sqrt{N}\, \hat{g}(\phi^*) = \frac{1}{\sqrt{N}} \sum_{i=1}^{N} m(\mathbf{Z}_i;\, \phi^*,\, f_0) + o_p(1).
\end{equation*}
By the multivariate central limit theorem, the leading term converges in distribution to $\mathcal{N}(\mathbf{0}, \Omega)$. Applying the continuous mapping theorem to Equation~\eqref{eq:sp_linear_rep} yields
\begin{equation*}
\sqrt{N}(\hat{\phi} - \phi^*) \;\Rightarrow\; \mathcal{N}\!\left(\mathbf{0},\, (G^\top \Omega^{-1} G)^{-1}\right).
\end{equation*}

\textit{Step 6 (Efficiency).} For the moment condition $\mathbb{E}\!\left[m(\mathbf{Z};\, \phi,\, f_0)\right] = \mathbf{0}$ with Jacobian $G$ and moment covariance $\Omega$, the class of regular GMM estimators indexed by positive definite weighting matrices $\mathbf{W}$ has asymptotic variance
\begin{equation*}
V(\mathbf{W}) = (G^\top \mathbf{W} G)^{-1} G^\top \mathbf{W}\, \Omega\, \mathbf{W}\, G\, (G^\top \mathbf{W} G)^{-1}.
\end{equation*}
A standard positive-semidefinite ordering argument shows that $V(\mathbf{W}) - V^*$ is positive semidefinite for all positive definite $\mathbf{W}$, with equality attained at $\mathbf{W} = \Omega^{-1}$. Hence $V^* = (G^\top \Omega^{-1} G)^{-1}$ is the minimal asymptotic variance achievable by regular GMM estimators based on these moments, establishing semiparametric efficiency.
\end{proof}

Theorem~\ref{thm:sp-eff} places SKINNs within the family of sieve GMM estimators \citep{ai2003efficient} and deep GMM methods \citep{bennett2019deep, farrell2021deep}, while distinguishing it through the explicit incorporation of a learnable structured-knowledge component. The orthogonality condition ensures that the flexibility of the neural network approximator does not inflate the asymptotic variance of the structural parameter estimates, a property that is particularly valuable in high-dimensional settings where the nuisance parameter space is vast relative to the sample size.

\subsection{Generalization Bounds and Distributional Robustness}\label{app:stability_shift}

Out-of-sample performance is a central concern in finance, where distributional shifts between training and deployment periods are common. This subsection provides two complementary formal results. The first (Theorem~\ref{thm:stability}) shows that structured regularization tightens the uniform stability bound on the expected generalization gap, formalized in a tractable convex proxy. The second (Theorem~\ref{thm:target-decomp}) decomposes the risk under a shifted target distribution into an alignment term controlled by the structured-knowledge loss and a portability term measuring how well $g_\phi$ transfers to the new environment. Together, these results formalize a compelling intuition: if the embedded theoretical model captures structural regularities that remain stable across regimes, the SKINNs estimator inherits this stability.

\subsubsection{Generalization via Uniform Stability}\label{app:gen-stability}

We analyze the generalization properties of structured regularization in a convex proxy that admits sharp stability bounds. Consider a linear-in-parameters approximation $f_\theta(\mathbf{X}) = \theta^\top \mathbf{X}$ with $\theta \in \mathbb{R}^p$, and assume bounded features $\|\mathbf{X}\| \le B_x$ almost surely.\footnote{The restriction to a linear model is made to obtain an explicit stability bound. The qualitative insight, that structured regularization tightens the generalization gap by increasing the strong-convexity parameter of the objective, extends to nonlinear models under appropriate conditions \citep[see, e.g.,][]{hardt2016train}.} Let $S = \{(\mathbf{X}_i, \mathbf{y}_i)\}_{i=1}^N$ denote the observed sample and let $\tilde{S} = \{\tilde{\mathbf{X}}_j\}_{j=1}^M$ be a set of collocation points drawn independently from a (possibly different) distribution on the input domain. Fix $\phi$ and write $g_j \equiv g_\phi(\tilde{\mathbf{X}}_j^{\text{SK}})$.

Define the empirical SKINNs objective with an explicit ridge term ($\rho > 0$):
\begin{equation}\label{eq:stability_objective}
\hat{\mathcal{L}}_S(\theta) \equiv \frac{1}{N} \sum_{i=1}^{N} (\theta^\top \mathbf{X}_i - \mathbf{y}_i)^2 + \lambda \frac{1}{M} \sum_{j=1}^{M} (\theta^\top \tilde{\mathbf{X}}_j - g_j)^2 + \rho\, \|\theta\|^2,
\end{equation}
and let $\hat{\theta}(S)$ denote its unique minimizer. Define the population risk and the empirical risk as
\begin{equation*}
R(\theta) \equiv \mathbb{E}\!\left[(\theta^\top \mathbf{X} - \mathbf{y})^2\right], \qquad \hat{R}_S(\theta) \equiv \frac{1}{N} \sum_{i=1}^{N} (\theta^\top \mathbf{X}_i - \mathbf{y}_i)^2.
\end{equation*}

\begin{theorem}[Uniform stability and generalization bound]\label{thm:stability}
Assume $\|\mathbf{X}\| \le B_x$ and $\|\tilde{\mathbf{X}}\| \le B_x$ almost surely, and $|\mathbf{y}| \le B_y$ almost surely. Let $\mu \equiv \rho$, so that $\hat{\mathcal{L}}_S$ is at least $\mu$-strongly convex in $\theta$. Then $\hat{\theta}(\cdot)$ is uniformly stable: for any two samples $S, S'$ that differ in exactly one observation, and for any $(\mathbf{X}, \mathbf{y})$ with $\|\mathbf{X}\| \le B_x$ and $|\mathbf{y}| \le B_y$,
\begin{equation}\label{eq:stability_pointwise}
\left|\ell(\hat{\theta}(S);\, \mathbf{X}, \mathbf{y}) - \ell(\hat{\theta}(S');\, \mathbf{X}, \mathbf{y})\right| \le \frac{4\, B_x^2\, (B_x \|\hat{\theta}(S)\| + B_y)}{\mu\, N},
\end{equation}
where $\ell(\theta;\, \mathbf{X}, \mathbf{y}) \equiv (\theta^\top \mathbf{X} - \mathbf{y})^2$. Consequently, the expected generalization gap satisfies
\begin{equation}\label{eq:gen_gap}
\mathbb{E}\!\left[R(\hat{\theta}(S)) - \hat{R}_S(\hat{\theta}(S))\right] \le \frac{4\, B_x^2}{\mu\, N}\, \mathbb{E}\!\left[B_x \|\hat{\theta}(S)\| + B_y\right].
\end{equation}
In particular, for fixed data and bounded $\mathbb{E}\|\hat{\theta}(S)\|$, increasing $\mu$ (equivalently, increasing $\rho$ or $\lambda$) tightens the bound.
\end{theorem}

\begin{proof}
\textit{Step 1 (Strong convexity).} The ridge term $\rho\, \|\theta\|^2$ makes $\hat{\mathcal{L}}_S$ at least $\rho$-strongly convex. The structured-knowledge term with $\lambda > 0$ adds further convexity, so $\mu = \rho$ is a conservative lower bound.

\textit{Step 2 (Parameter sensitivity).} Let $S$ and $S'$ differ only in the $N$-th observation. Define
\begin{equation*}
\Delta(\theta) \equiv \hat{\mathcal{L}}_S(\theta) - \hat{\mathcal{L}}_{S'}(\theta) = \frac{1}{N}\left((\theta^\top \mathbf{X}_N - \mathbf{y}_N)^2 - (\theta^\top \mathbf{X}_N' - \mathbf{y}_N')^2\right).
\end{equation*}
By the optimality conditions for strongly convex objectives,
\begin{equation*}
\|\hat{\theta}(S) - \hat{\theta}(S')\| \le \frac{\|\nabla \Delta(\hat{\theta}(S))\|}{\mu}.
\end{equation*}
Computing the gradient,
\begin{equation*}
\nabla \Delta(\theta) = \frac{2}{N}\left((\theta^\top \mathbf{X}_N - \mathbf{y}_N)\, \mathbf{X}_N - (\theta^\top \mathbf{X}_N' - \mathbf{y}_N')\, \mathbf{X}_N'\right),
\end{equation*}
and applying the triangle inequality with the bounds $\|\mathbf{X}_N\|, \|\mathbf{X}_N'\| \le B_x$ and $|\mathbf{y}_N|, |\mathbf{y}_N'| \le B_y$ yields
\begin{equation*}
\|\hat{\theta}(S) - \hat{\theta}(S')\| \le \frac{4\, B_x\, (B_x \|\hat{\theta}(S)\| + B_y)}{\mu\, N}.
\end{equation*}

\textit{Step 3 (Loss sensitivity).} For any $(\mathbf{X}, \mathbf{y})$ and any $\theta, \theta'$,
\begin{equation*}
|\ell(\theta;\, \mathbf{X}, \mathbf{y}) - \ell(\theta';\, \mathbf{X}, \mathbf{y})| \le \left(B_x(\|\theta\| + \|\theta'\|) + 2\, B_y\right) \cdot B_x\, \|\theta - \theta'\|.
\end{equation*}
Substituting $\theta = \hat{\theta}(S)$, $\theta' = \hat{\theta}(S')$, and the bound from Step 2, and absorbing $\|\theta'\|$ into a conservative bound, yields Equation~\eqref{eq:stability_pointwise}.

\textit{Step 4 (Expected generalization gap).} Uniform stability with modulus $\beta_N$ implies $\mathbb{E}[R(\hat{\theta}(S)) - \hat{R}_S(\hat{\theta}(S))] \le \beta_N$ by the standard argument of \citet{bousquet2002stability}. Applying the bound from Step 3 yields Equation~\eqref{eq:gen_gap}.
\end{proof}

The bound in Equation~\eqref{eq:gen_gap} makes precise the regularization benefit of the structured-knowledge component: both $\rho$ (explicit ridge) and $\lambda$ (structured-knowledge penalty) contribute to the strong-convexity parameter $\mu$, tightening the generalization gap. In the absence of structured regularization ($\lambda = 0$), only the ridge term controls stability, and a larger ridge penalty is needed to achieve the same bound, at the cost of greater shrinkage bias. The structured-knowledge penalty achieves a similar stabilizing effect while directing the regularization toward theoretically meaningful regions of the parameter space.

\subsubsection{Target-Risk Decomposition Under Distribution Shift}\label{app:gen-shift}

We now consider the setting where the deployment (target) distribution differs from the training (source) distribution, as is typical in financial applications subject to regime changes. Let $P$ denote the training distribution of $(\mathbf{X}, \mathbf{y})$ and $Q$ a target distribution. For any measurable predictor $f$ and any $\phi$, define the squared risk under $Q$:
\begin{equation}\label{eq:target_risk}
R_Q(f) \equiv \mathbb{E}_Q\!\left[(f(\mathbf{X}) - \mathbf{y})^2\right].
\end{equation}

\begin{theorem}[Target-risk decomposition]\label{thm:target-decomp}
For any measurable $f$ and any $\phi$,
\begin{equation}\label{eq:target_decomp}
R_Q(f) \le 2\, \mathbb{E}_Q\!\left[(f(\mathbf{X}) - g_\phi(\mathbf{X}^{\text{SK}}))^2\right] + 2\, \mathbb{E}_Q\!\left[(g_\phi(\mathbf{X}^{\text{SK}}) - \mathbf{y})^2\right].
\end{equation}
The first term measures the alignment between $f$ and $g_\phi$ under the target distribution, and the second term measures the portability of $g_\phi$ to the target environment. If $g_\phi$ is portable (the second term is small), then controlling the alignment term controls target risk.
\end{theorem}

\begin{proof}
Write $f(\mathbf{X}) - \mathbf{y} = (f(\mathbf{X}) - g_\phi(\mathbf{X}^{\text{SK}})) + (g_\phi(\mathbf{X}^{\text{SK}}) - \mathbf{y})$. Applying the elementary inequality $(a + b)^2 \le 2a^2 + 2b^2$ and taking expectations under $Q$ yields Equation~\eqref{eq:target_decomp}.
\end{proof}

\begin{corollary}[Collocation design controls the alignment term]\label{cor:collocation-upper}
Suppose the collocation distribution used for $\mathbf{X}_{\text{Colloc}}$ equals the target marginal distribution of $\mathbf{X}$ under $Q$. Then the first term in Theorem~\ref{thm:target-decomp} coincides with the structured-knowledge loss:
\begin{equation*}
\mathbb{E}_Q\!\left[(f(\mathbf{X}) - g_\phi(\mathbf{X}^{\text{SK}}))^2\right] = \mathbb{E}_{\tilde{\mathbf{X}}}\!\left[(f(\mathbf{X}_{\text{Colloc}}) - g_\phi(\mathbf{X}_{\text{Colloc}}^{\text{SK}}))^2\right],
\end{equation*}
which is directly minimized during SKINNs training via $\mathcal{L}_{\text{SK}}$.
\end{corollary}

\begin{proof}
Under the stated condition, the marginal distribution of $\mathbf{X}$ under $Q$ equals the distribution of $\mathbf{X}_{\text{Colloc}}$, so the two expectations coincide.
\end{proof}

Theorem~\ref{thm:target-decomp} and Corollary~\ref{cor:collocation-upper} together formalize a key practical insight for financial applications. If the embedded theoretical model captures structural regularities, such as no-arbitrage conditions or equilibrium relations, that remain stable across market regimes, then $g_\phi$ is portable and the second term in Equation~\eqref{eq:target_decomp} remains small under distributional shift. The first term is directly controlled by the SKINNs training objective when the collocation distribution is chosen to reflect the anticipated target environment. This provides a principled mechanism for robustness: rather than relying solely on the i.i.d.\ assumption underlying standard generalization theory, SKINNs anchor predictions to theoretical structures whose validity transcends any particular data-generating regime.

\subsection{Differentiability and Curse of Dimensionality}\label{sec:differntial_curse_dimensionality}

The semi-parametric structured-knowledge representations introduced in Section~\ref{structured-knowledge component} require that the surrogate model $g_\phi(\mathbf{X}^{\text{SK}}) \equiv f^{\text{SR}}_{\theta}(\mathbf{X}^{\text{SK}}, \phi)$ be first-order differentiable with respect to both its observable inputs $\mathbf{X}^{\text{SK}}$ and the learnable latent parameters $\phi$. This subsection establishes that the neural network architecture of the surrogate guarantees this differentiability, derives the explicit gradient flow, and clarifies how deep surrogates circumvent the curse of dimensionality that afflicts direct PDE/SDE-based knowledge embedding.

\subsubsection{End-to-End Differentiability}

Once trained offline (as described in Section~\ref{structured-knowledge component}), the surrogate weights $\theta_s$ are frozen. The input to the surrogate is the concatenation of observable features and latent parameters:
\begin{equation}\label{eq:surrogate_input}
\mathbf{z}_0 = \begin{bmatrix} \mathbf{X}^{\text{SK}} \\ \phi \end{bmatrix}.
\end{equation}
The surrogate is a composition of differentiable affine transformations and smooth activation functions $\sigma(\cdot)$:
\begin{equation}\label{eq:surrogate_forward}
f^{\text{SR}}_{\theta}(\mathbf{z}_0) = \mathbf{W}_L\, \sigma(\cdots \sigma(\mathbf{W}_1 \mathbf{z}_0 + \mathbf{b}_1) \cdots) + \mathbf{b}_L.
\end{equation}
Since each layer is differentiable with respect to its input, the composite function is differentiable with respect to $\phi$ by the chain rule. The Jacobian of the surrogate output with respect to $\phi$ is the product of layer-wise Jacobians:
\begin{equation}\label{eq:surrogate_jacobian}
\mathbf{J}_s(\phi) \equiv \frac{\partial\, f^{\text{SR}}_{\theta}(\mathbf{X}^{\text{SK}}, \phi)}{\partial \phi} = \prod_{l=1}^{L} \operatorname{diag}\!\left(\sigma'_l(\mathbf{h}_l)\right) \mathbf{W}_l,
\end{equation}
where $\mathbf{h}_l$ is the pre-activation vector at layer $l$, and the product is taken in reverse order (from output to input layer). During SKINNs training, the gradient of the composite loss with respect to $\phi$ is computed via the vector-Jacobian product:
\begin{equation}\label{eq:grad_phi_surrogate}
\nabla_\phi\, \mathcal{L} = \mathbf{J}_s(\phi)^\top\, \nabla_{g_\phi}\, \mathcal{L},
\end{equation}
which is evaluated efficiently by automatic differentiation in a single backward pass through the frozen surrogate. This ensures that the asymptotic properties established in Section~\ref{sec:general_theory}, which require first-order differentiability of $g_\phi$ with respect to $\phi$, are satisfied by construction. The second-order differentiability required by Assumption~\ref{assump:hessian} likewise holds whenever the activation functions $\sigma(\cdot)$ are twice continuously differentiable (e.g., sigmoid, tanh, softplus, GELU), and holds almost everywhere for piecewise-linear activations such as ReLU.

\subsubsection{Circumventing the Curse of Dimensionality}

A key motivation for the surrogate approach is computational tractability when the structured knowledge is prescribed by high-dimensional SDEs or PDEs. For methods that embed PDE/SDE knowledge directly into the neural network (as in PINNs), the minimum network size required for convergence increases exponentially with both the PDE order and the dimensionality of the state variables \citep{gao2023gradient, song2024does}. This renders direct embedding infeasible for the rich multivariate SDE systems commonly encountered in finance, such as the stochastic volatility models in Section (2), Equation (1)-(3) form \cite{kaeck2012volatility}, which involve multiple correlated state variables, jump processes, and ten or more latent parameters.

The deep surrogate strategy decouples the SKINNs optimization from the dimensionality of the underlying SDE/PDE. The computational cost of generating the synthetic training data $\mathcal{D}_{\text{surrogate}}$ (via Monte Carlo simulation or finite-difference methods) is incurred once during the offline pre-training phase. Once trained, the surrogate $f^{\text{SR}}_{\theta}$ operates as a global function approximator whose parameter count scales polynomially, rather than exponentially, with the input dimension \citep{grohs2023proof}. Each evaluation of $g_\phi$ during SKINNs training requires only a forward pass through the frozen surrogate, at a cost independent of the complexity of the original SDE/PDE system. This ensures that the per-iteration cost of SKINNs training remains computationally tractable even when the embedded structured knowledge originates from high-dimensional theoretical models.

\subsection{SKINNs as a Unifying Framework}\label{sec:alternative_perspectives}

A distinctive strength of the SKINNs framework is that its composite objective in Equation~\eqref{eq:normalized_objective} is not tied to a single methodological tradition. The preceding sections have developed the $M$-estimator foundations (Section~\ref{sec:general_theory}), the GMM interpretation (Section~\ref{sec:gmm_interpretation}), and the semiparametric efficiency properties (Section~\ref{app:sp-eff}). In this section, we show that several well-established paradigms emerge as special cases or limiting configurations of the same composite objective, with the specific instantiation determined by the choice of $g_\phi$, the loss structure, and the regularization strength $\lambda$. This multiplicity of valid interpretations reflects a deeper architectural point: functional GMM, Bayesian MAP estimation, transfer learning, physics-informed learning, and domain adaptation all reside within the SKINNs framework. We develop each connection below.

\subsubsection{Bayesian MAP Estimation}

From a Bayesian perspective, the SKINNs objective corresponds to maximum a posteriori (MAP) estimation with a theory-informed prior. Assume independent Gaussian observation noise, $\mathbf{y}_i = f_\theta(\mathbf{X}_i) + \epsilon_i$ with $\epsilon_i \sim \mathcal{N}(0, \sigma^2)$, so that the log-likelihood is proportional to the negative data loss:
\begin{equation}\label{eq:log_likelihood}
\log p(\mathcal{D}_N \mid \theta, \phi) = -\frac{1}{2\sigma^2} \sum_{i=1}^{N} (f_\theta(\mathbf{X}_i) - \mathbf{y}_i)^2 + \text{const}.
\end{equation}
Specify the conditional prior $p(\theta \mid \phi)$ as encoding the belief that $f_\theta$ should be close to the theoretical model $g_\phi$:
\begin{equation}\label{eq:theory_prior}
\log p(\theta \mid \phi) \propto -\frac{\lambda}{2}\, \mathbb{E}\!\left[(f_\theta(\tilde{\mathbf{X}}) - g_\phi(\tilde{\mathbf{X}}^{\text{SK}}))^2\right],
\end{equation}
and take the prior on $\phi$ to be uninformative. The MAP estimator $(\hat{\theta}^{\text{MAP}}, \hat{\phi}^{\text{MAP}}) = \arg\max_{\theta,\phi}\, p(\theta, \phi \mid \mathcal{D}_N)$ then solves, after taking negative logarithms and normalizing:
\begin{equation}\label{eq:map_is_skinns}
(\hat{\theta}^{\text{MAP}}, \hat{\phi}^{\text{MAP}}) = \arg\min_{\theta, \phi} \left[\frac{1}{N} \sum_{i=1}^{N} (f_\theta(\mathbf{X}_i) - \mathbf{y}_i)^2 + \lambda\, \mathbb{E}\!\left[(f_\theta(\tilde{\mathbf{X}}) - g_\phi(\tilde{\mathbf{X}}^{\text{SK}}))^2\right]\right],
\end{equation}
which recovers the SKINNs objective in Equation~\eqref{eq:normalized_objective}. In this reading, the structured-knowledge loss $\mathcal{L}_{\text{SK}}$ acts as a negative log-prior that softly constrains $f_\theta$ to lie near $g_\phi$. The regularization parameter $\lambda$ is the prior precision: a large $\lambda$ reflects strong prior belief in the theory, while a small $\lambda$ reflects a diffuse prior. The joint optimization over $(\theta, \phi)$ corresponds to learning both the function and the hyperparameters of the prior center, which is a form of empirical Bayes estimation.

This connection places SKINNs within the tradition of Bayesian neural networks \citep{neal1996bayesian, mackay1992bayesian}, with a crucial distinction: the prior is not generic (e.g., weight decay) but theory-informed, centered at a scientifically meaningful function $g_\phi$ whose parameters are themselves learned from data. A fully Bayesian treatment would characterize the entire posterior $p(\theta, \phi \mid \mathcal{D}_N)$ via variational inference \citep{blundell2015weight} or Markov chain Monte Carlo, at substantially greater computational cost. The MAP point estimate provided by SKINNs offers a tractable compromise that captures the posterior mode while remaining computationally feasible.

\subsubsection{Physics-Informed Learning and Variational Principles}

The SKINNs objective admits a natural reading as a total energy functional:
\begin{equation}\label{eq:energy_functional}
E[f_\theta, g_\phi] = \underbrace{\mathbb{E}\!\left[(f_\theta(\mathbf{X}) - \mathbf{y})^2\right]}_{E_{\text{data}}} + \lambda \underbrace{\mathbb{E}\!\left[(f_\theta(\tilde{\mathbf{X}}) - g_\phi(\tilde{\mathbf{X}}^{\text{SK}}))^2\right]}_{E_{\text{theory}}},
\end{equation}
where $E_{\text{data}}$ measures empirical misfit and $E_{\text{theory}}$ penalizes deviations from theoretical predictions. Training via gradient descent approximates the continuous-time gradient flow
\begin{equation}\label{eq:gradient_flow}
\frac{d\theta}{dt} = -\nabla_\theta E[f_\theta, g_\phi], \qquad \frac{d\phi}{dt} = -\nabla_\phi E[f_\theta, g_\phi],
\end{equation}
analogous to overdamped Langevin dynamics in statistical physics, in which the system dissipates energy and settles into a local minimum representing a stable equilibrium \citep{welling2011bayesian}. Equivalently, the penalized formulation in Equation~\eqref{eq:normalized_objective} is dual to a constrained optimization problem
\begin{equation}\label{eq:constrained_opt}
\min_{\theta, \phi}\; E_{\text{data}}[f_\theta] \qquad \text{subject to} \qquad E_{\text{theory}}[f_\theta, g_\phi] \le \epsilon,
\end{equation}
where $\lambda$ plays the role of the Lagrange multiplier \citep{bertsekas2014constrained}.

The variational perspective clarifies a key distinction between SKINNs and Physics-Informed Neural Networks \citep[PINNs;][]{raissi2019physics}. PINNs enforce differential equations as soft constraints by penalizing PDE residuals:
\begin{equation}\label{eq:pinn_loss}
\mathcal{L}_{\text{PINN}} = \mathcal{L}_{\text{data}} + \lambda\, \mathbb{E}\!\left[\|\mathbb{D}[f_\theta] - 0\|^2\right],
\end{equation}
where $\mathbb{D}[\cdot]$ is a differential operator. Computing $\mathbb{D}[f_\theta]$ requires automatic differentiation of the neural network outputs with respect to its inputs, which can be numerically unstable for high-order derivatives and leads to well-documented spectral biases and gradient pathologies \citep{wang_understanding_2020, wang_when_2022}. SKINNs instead enforce consistency with the \textit{solution} of a theoretical model, $g_\phi(\mathbf{X}^{\text{SK}})$, rather than with a differential operator. This has two advantages. First, it avoids the numerical instabilities associated with differentiating neural networks with respect to inputs. Second, it allows $g_\phi$ to be non-analytical (e.g., a deep surrogate or an auto-encoder), whereas PINNs require explicit differential forms. When $g_\phi$ is set to the PDE residual operator and $\lambda \to \infty$, SKINNs recover the PINNs objective as a limiting case.

\subsubsection{Transfer Learning and Domain Adaptation}

SKINNs generalize two-stage transfer learning \citep{pan2009survey, pratt1992discriminability} in which a model pre-trained on a source task is fine-tuned on a target task. In transfer learning, the source model provides a fixed initialization; the target data then adjust the parameters, but the source model does not adapt. In SKINNs, the structured-knowledge component $g_\phi$ plays the role of the source model, but with two critical extensions: (i) $g_\phi$ provides guidance throughout training (not only at initialization), and (ii) its parameters $\phi$ are jointly updated with the neural network parameters $\theta$, enabling bidirectional knowledge transfer in a single optimization loop. When $\phi$ is fixed and $g_\phi$ serves only to initialize $f_\theta$ before fine-tuning on $\mathcal{D}_N$, SKINNs reduce to standard transfer learning.

The framework also connects to domain adaptation \citep{ben2010theory}, where the goal is to learn a function that generalizes across different data distributions. In SKINNs, the structured-knowledge loss enforces alignment between $f_\theta$ and $g_\phi$ over a potentially broad collocation domain $\tilde{\mathbf{X}}$, including regions where observed data are sparse. This promotes generalization to out-of-distribution inputs, as formalized by the target-risk decomposition in Theorem~\ref{thm:target-decomp}: if $g_\phi$ is portable across distributions, the alignment enforced by $\mathcal{L}_{\text{SK}}$ during training directly controls the risk under the shifted target distribution.

\subsubsection{Summary: Limiting Configurations}

Table~\ref{tab:limiting_configs} summarizes how several established paradigms emerge as special cases of the SKINNs composite objective.

\begin{table}[h]
\centering
\caption{SKINNs as a unifying framework: limiting configurations of the composite objective.}
\label{tab:limiting_configs}
\small
\begin{tabular}{lll}
\hline
\textbf{Paradigm} & \textbf{Configuration} & \textbf{Reference} \\
\hline
Plain neural network & $\lambda = 0$ & \citet{gu2020empirical} \\
PINNs & $g_\phi = \mathbb{D}[f_\theta]$, $\lambda \to \infty$ & \citet{raissi2019physics} \\
Transfer learning & $\phi$ fixed, two-stage training & \citet{pan2009survey} \\
Bayesian MAP estimation & $\mathcal{L}_{\text{SK}}$ as neg.\ log-prior & \citet{neal1996bayesian} \\
Regularized / functional GMM & $\mathcal{L}_{\text{data}}, \mathcal{L}_{\text{SK}}$ as moment conditions & \citet{hansen1982large} \\
Sieve GMM / Deep GMM & $f_\theta$ as neural sieve, optimal $\mathbf{W}$ & \citet{ai2003efficient} \\
Structured regularization & $\mathcal{L}_{\text{SK}}$ as domain-adaptive penalty & \citet{ben2010theory} \\
\hline
\end{tabular}
\end{table}

Each row corresponds to a specific choice of $g_\phi$, loss structure, regularization strength $\lambda$, or training protocol. The full SKINNs framework encompasses all of these configurations simultaneously, while its joint optimization of $(\theta, \phi)$ enables capabilities, such as dynamic latent parameter discovery and bidirectional theory-data reconciliation, that none of the individual paradigms provide in isolation.

\end{appendices}

\clearpage
\onehalfspacing
\setlength\bibsep{0pt}
\bibliographystyle{elsarticle-harv}
\bibliography{ref}

\clearpage
\onehalfspacing

\end{document}